\documentclass{article}
\usepackage[utf8]{inputenc}

\usepackage{setspace}
\usepackage{amsmath}
\usepackage{amssymb}
\usepackage{fullpage}
\usepackage{graphicx}
\usepackage{imakeidx}
\usepackage{blindtext}
\usepackage{amsthm}
\usepackage{caption}
\usepackage{subcaption}
\usepackage{parskip}
\usepackage{float}
\usepackage{url}
\usepackage{geometry}
\usepackage{authblk}
\usepackage{xcolor}
\usepackage{tikz-cd}
\usepackage{enumitem}
\usepackage[round]{natbib}
\usepackage[colorlinks=true]{hyperref}

\definecolor{imperialBlue}{RGB}{0, 62, 116}
\definecolor{imperialBrick}{RGB}{165,25,0}
\definecolor{imperialProcess}{RGB}{0,133,202}
\definecolor{imperialGreen}{RGB}{2,137,59}
\definecolor{imperialRed}{RGB}{221,37,1}
\definecolor{imperialOrange}{RGB}{210,64,0}
\definecolor{imperialBlue2}{RGB}{0,110,175}
\definecolor{imperialTangerine}{RGB}{236,115,0}
\definecolor{imperialPurple}{RGB}{101,48,152}
\definecolor{imperialLime}{RGB}{196,214,0}
\definecolor{imperialKermit}{RGB}{102,164,10}
\hypersetup{
    colorlinks,
    citecolor=imperialGreen,
    filecolor=imperialBlue,
    linkcolor=imperialBlue,
    urlcolor=imperialProcess
}

\usepackage{booktabs}
\usepackage{multirow}

\newtheorem{theorem}{Theorem}[section]
\newtheorem{proposition}{Proposition}[section]

\newtheorem{remark}{Remark}[section]
\newtheorem{lemma}{Lemma}
\newtheorem{assumption}{Assumption}[section]

\newcommand{\R}{\mathbb{R}}

\newcommand{\E}{\mathbb{E}}
\newcommand{\Prob}{\mathbb{P}}
\newcommand{\QQ}{\mathbb{Q}}
\newcommand{\F}{\mathcal{F}}
\newcommand{\diag}{\operatorname{diag}}

\newcommand{\dd}{\,\mathrm{d}}
\newcommand{\dt}{\Delta t}
\newcommand{\dW}{\Delta W}
\DeclareMathOperator{\ESS}{ESS}
\DeclareMathOperator{\KL}{KL}

\title{Structured Neural SDEs for Functional Calibration}

\author[1,\dag]{Francesco Piatti}
\author[1]{Andrea Iannucci}
\author[1]{Thomas Cass}

\affil[1]{Department of Mathematics, Imperial College London}
\affil[$\dag$]{Corresponding author, email: francesco.piatti19@imperial.ac.uk}

\date{}

\begin{document}
\maketitle
\setstretch{1.25}

\begin{abstract}
Neural Stochastic Differential Equations (Neural SDEs) provide flexible
continuous-time generative models, but generic neural drift and diffusion
networks are costly to simulate on long horizons and can give unstable
gradients when the training signal is a path functional rather than a
pointwise observation. We introduce \emph{SLiSDE}, a family of Neural SDE
models built from structured linear stochastic layers. Parallel-in-time
simulation is obtained at the layer level, while expressivity is recovered
by \emph{gated in-flow} stacking: previous-layer paths modulate the next
layer's latent flow through learned gates. For
functional calibration tasks in which rare paths dominate the loss, we add
an optional Girsanov tilt
that acts as a learned importance
sampler with an exact likelihood-ratio correction. We prove
well-posedness, a discretisation error bound, validity of the change of
measure, and a universality result: the terminal laws of the gated stack
are dense in the space of square-integrable laws. Experiments on
functional calibration benchmarks show that the structured model
outperforms fully neural SDE baselines while retaining parallel-time
simulation and stable importance weights.
\end{abstract}

\section{Introduction}
\label{sec:intro}
Many quantities of practical interest are expectations of functionals of a stochastic path whose trajectories are never observed. In derivatives markets the available data are option quotes -- expectations of terminal payoffs -- and exceedance probabilities, and calibrating a model of the underlying dynamics to such quotes is the entry point of pricing, hedging and risk management \citep{Buehler2019DeepHedging}; the same situation, dynamics observed only through a finite family of statistics, arises in the physical sciences. We refer to this task as \emph{functional calibration}: fitting a generative model of continuous-time stochastic dynamics to prescribed expectations of path functionals. Training such a model requires simulating it afresh at every gradient step, so the computational cost of simulation is the binding constraint, and it is the constraint that shapes the model proposed here.

Neural differential equations provide a flexible framework for learning continuous-time dynamics by parameterizing their vector fields with neural networks. Neural ODEs \citep{Chen2018} model deterministic latent
evolution, neural CDEs \citep{Kidger2020NeuralCDE, slice2025} condition the dynamics on an observed driving path, and neural SDEs \citep{Tzen2019} provide the stochastic analogue, modeling distributions on continuous-time paths via stochastic differential equations, which may be interpreted in either the It\^o or Stratonovich sense.

In this work, we focus on the It\^o setting, where the latent state evolves as
\begin{equation}
    \dd Z_t = \mu_\theta(t,Z_t)\,\dd t
    + \sigma_\theta(t,Z_t)\,\dd W_t,
    \qquad Z_0\sim p_0,
    \label{eq:generic-nsde}
\end{equation}
where $W_t$ is a Brownian motion and $\mu_\theta$, $\sigma_\theta$ unconstrained neural
networks. This is a rich model class \citep{Tzen2019,Liu2019,Li2020,Kidger2021NSDEsAsInfiniteDimensional}, and in practice it is often used as a \emph{Markovian lift}: the latent state is lifted to a higher-dimensional space in which a non-Markovian observed process can be represented as a Markov diffusion. The viewpoint is powerful, but it makes simulation expensive — every time step requires fresh nonlinear evaluations of the lifted vector field, and the recurrence is inherently sequential unless additional structure is imposed.

The central design principle of this paper is to move the non-linear part
of the model from the non-linear drift and diffusion evaluated at every time step to non-linear coupling between structured linear SDE layers. After discretisation, each layer is an affine recurrence whose
composition is associative, so it can be evaluated by the parallel associative scan of Section~\ref{ssec:linear-sde-affine}, which cuts the sequential simulation cost over a horizon of $T$ time steps from $O(T)$ to $O(\log T)$.
To recover expressivity, we stack such layers through \emph{gated in-flow}
coupling: the previous-layer path modulates the next layer's transition
and offset through learned gates. The current layer remains
affine in its own state, so the scan structure is preserved, while the gates make the stack a genuinely nonlinear model, universal at the level of terminal laws (Theorem~\ref{thm:wasserstein-density-terminal-laws}).

For high-variance path-dependent objectives, we add an optional
last-layer-only Girsanov tilt. A learned adapted controller changes only
the Brownian increments driving the final layer and returns the
corresponding Radon--Nikodym log-weight in closed form.

\subsection*{Related work.}
Within the neural-SDE literature two training paradigms coexist. Generative training matches the law of the model to an empirical distribution of paths through a statistical divergence -- the KL divergence of variational inference \citep{Tzen2019,Li2020}, the Wasserstein-1 distance of the GAN formulation \citep{Kidger2021NSDEsAsInfiniteDimensional}, or signature-kernel scores \citep{Issa2023SignatureKernelScores} -- and requires sample paths. Calibration training, our setting, instead matches prescribed expectations of functionals, typically derivative prices, when no trajectory is observed \citep{Gierjatowicz2020RobustPricing,Cuchiero2020GANCalibration,Cohen2021ArbitrageFree}; before neural SDEs this task was handled by parametric families that hard-wire the dynamics, such as stochastic-volatility and local-volatility models in finance or mechanistic diffusion models in the physical sciences.

Structured state-space models have shown that carefully parameterised linear recurrences can achieve strong sequence-modelling performance while remaining hardware efficient \citep{Gu2022S4,Smith2023S5,Hasani2022LiquidS4,Gu2023Mamba}. Recent work also connects this line of models to controlled differential equations, viewing sequence transformations as dynamics driven by an input path \citep{muca2024theoretical}. SLiCE \citep{slice2025} develops this perspective for structured linear controlled differential equations, showing how scan-compatible linear dynamics can be embedded in a CDE-style framework.

The Girsanov component is connected to importance sampling for diffusion
processes and to the stochastic-control view of variance reduction
\citep{Hartmann2017,Zhang2022PathIntegralSampler,Hartmann2024NeuralSamplers}.
In those approaches, one learns or designs a drift change that makes rare
paths less rare and then corrects the resulting bias with a likelihood ratio.
Our tilt is narrower: it acts only on the Brownian motion of the last SLiSDE layer. This keeps the prefix computation unchanged,
and makes the proposal lightweight
enough to train jointly with the backbone.

\subsection*{Contributions.}
\begin{itemize}[leftmargin=1.2em,itemsep=0.2em,topsep=0.2em]
    \item \textbf{Structured linear neural SDE for functional calibration.} We introduce a structured linear Neural SDE backbone whose discretized layers are affine recurrences. This gives $O(\log T)$ parallel depth in sequence length by the parallel associative scan, compared with the $O(T)$ complexity of standard neural SDE solvers.
    \item \textbf{Gated in-flow stacking.} We propose a gated in-flow stacking rule in which previous-layer paths modulate the next layer's affine transition and offset. Nonlinearity enters between layers only, so each layer stays affine in its own state and scan-compatible.
    \item \textbf{Girsanov overlay.} We add an optional last-layer-only
    Girsanov tilt with a closed-form likelihood-ratio correction: a learned
    adapted controller acts as a self-normalised importance sampler that
    improves performance on rare-event functional calibration with little
    added compute.
    \item \textbf{Theory.} We prove that the terminal laws generated by the gated stack are dense in $\mathcal P_2(\R^p)$ for the $2$-Wasserstein distance (Theorem~\ref{thm:wasserstein-density-terminal-laws}), so the structural restrictions cost nothing at the level of terminal laws.

    \item \textbf{Experiments.} We evaluate the model on nonlinear path-functional calibration tasks and real option-surface benchmarks, comparing against neural SDE and structured sequence-model baselines.
    SLiSDE attains the lowest or a statistically tied calibration loss on every benchmark, and trains faster than SLiCE in every setting tested and faster than the Neural SDE in all but two of them; a controlled study shows the tilt as a learned importance sampler for rare-event functionals.
\end{itemize}

\section{Mathematical Background}
\label{sec:background}

Let $(\Omega,\F,\{\F_t\}_{t\in[0,T]},\Prob)$ be a filtered probability
space supporting an $m$-dimensional Brownian motion $W_t=(W_t^1,\dots,W_t^m)$ on $[0,T]$. We consider a time grid $0=t_0<t_1<\cdots<t_N=T$, with Brownian increments $\Delta W_k = W_{t_{k+1}}-W_{t_k}\sim \mathcal N(0,\Delta t_k I_m)$. Throughout, complexities are stated in terms of the horizon $T$, which for a fixed step size is proportional to the number of steps $N$. The Girsanov change of measure on which the tilt of Section~\ref{sec:girsanov} rests is recalled in Appendix~\ref{app:girsanov-background}.

\subsection{Linear SDEs and affine flows}
\label{ssec:linear-sde-affine}

A general affine linear SDE takes the form
\begin{equation}
    \dd Z_t
    =
    (A Z_t + b)\,\dd t
    +
    \sum_{j=1}^m C^j Z_t\,\dd W_t^j
    +
    D\,\dd W_t, \quad Z_0 = z_0
    \label{eq:linear-sde}
\end{equation}
with $z_0 \in L^2(\Omega,\mathcal F_0,\mathbb P;\mathbb R^d)$, $A,C^j\in\R^{d\times d}$, $b\in\R^d$, and
$D\in\R^{d\times m}$. On each grid interval $[t_k,t_{k+1}]$ its discretisation, by the matrix-exponential transition or by the Euler--Maruyama transition (Appendix~\ref{app:discretisation}; the latter is used throughout our experiments), is an affine recurrence
\begin{equation}
    Z_{k+1}=F_kZ_k+g_k,
    \label{eq:affine-recurrence}
\end{equation}
where the transition matrix $F_k\in\R^{d\times d}$ and the offset $g_k\in\R^d$ are built from the coefficients and the Brownian increment $\Delta W_k$.

\paragraph{Parallel associative scan.} The solution of the affine recurrence Eq.~\ref{eq:affine-recurrence} can be written $Z_k=\widehat F_kZ_0+\widehat g_k$ with $(\widehat F_k,\widehat g_k)=(F_{k-1},g_{k-1})\circ\cdots\circ(F_0,g_0)$, where
\begin{equation}
    (F',g')\circ(F,g):=(F'F,\;F'g+g')
    \label{eq:model-affine-combine}
\end{equation}
is composition of affine maps. This operation is associative, so all prefixes $(\widehat F_k,\widehat g_k)_{k\le N}$ can be computed with any bracketing, in particular along a balanced binary tree of depth $O(\log T)$, and the result equals the sequential recursion exactly (Proposition~\ref{prop:affine-scan-closure}). We call this evaluation the \emph{parallel associative scan} and say that a recurrence is \emph{scan-compatible} when its one-step maps are affine in the state; this is the only name used for the scan in the sequel.
\begin{remark}
    The scan does not reduce the total arithmetic work across all time steps, but it reduces the sequential dependence from \(O(T)\) recurrent updates to \(O(\log T)\) parallel depth, exposing the time dimension to parallel hardware.
\end{remark}

\subsection{Functional calibration objective}
\label{sec:objective}
Let $Y^\theta$ denote the output path of a model with parameters $\theta$. The data of a functional calibration problem are a finite family $\Phi$ of path functionals $\varphi$, each with a prescribed value $c_\varphi$ of its expectation; no trajectory is observed. Given $n$ independent model paths $Y^{\theta,1},\dots,Y^{\theta,n}$ and the Monte-Carlo average $\widehat\E_{\theta,n}[\varphi]=n^{-1}\sum_{i=1}^n\varphi(Y^{\theta,i})$, the objective is
\begin{equation}
    \mathcal L_n(\theta)
    =
    \sum_{\varphi\in\Phi} w_\varphi
    \Bigl(\widehat\E_{\theta,n}[\varphi]-c_\varphi\Bigr)^2,
    \qquad w_\varphi>0,
    \label{eq:calib-loss}
\end{equation}
minimised by stochastic gradient descent with a fresh sample of $n$ paths at every step. In this context no model path is paired with a data path: Eq.~\ref{eq:calib-loss} is a functional of the model law rather than a samplewise error. Since $\E[(\widehat\E_{\theta,n}[\varphi]-c_\varphi)^2]=(\E_\theta[\varphi]-c_\varphi)^2+n^{-1}\operatorname{Var}_\theta(\varphi)$, it converges almost surely to the population objective $\mathcal L_\infty(\theta)=\sum_\varphi w_\varphi(\E_\theta[\varphi]-c_\varphi)^2$, the Monte-Carlo variance entering at order $1/n$ only, so there is no incentive to collapse the model law onto a mean. Finitely many expectation constraints do not determine the path law, and $\mathcal L_\infty$ has no unique minimiser; this is a property of the task rather than of the objective, and the selected model is validated on functionals it was never fitted to. Where a reference dynamics is available, a relative-entropy penalty turns Eq.~\ref{eq:calib-loss} into a selection rule that drops into the training loop unchanged. Appendix~\ref{app:objective} develops these points and relates the objective to divergence-based generative training.

\section{The SLiSDE Model}
\label{sec:model}

A SLiSDE is built by composing structured linear stochastic layers through a stacking mechanism, followed by a learned linear readout. The optional Girsanov component is described separately in Section~\ref{sec:girsanov}.

\subsection{Structured linear base layer}
\label{ssec:base-layer}

For latent dimension $d$ and Brownian dimension $m$, each structured layer is an It\^o SDE of the form
\begin{equation}
    \dd Z_t
    =
    \big(A_t Z_t+b_t\big)\,\dd t
    +
    \sum_{j=1}^m
    \big(C_t^j Z_t+d_t^j\big)\,\dd W_t^j,\quad Z_0 = z_0,
    \label{eq:model-time-dependent-sde}
\end{equation}
where $A_t,C_t^j\in\R^{d\times d}$ and
$b_t,d_t^j\in\R^d$ and we assume that $z_0 \in L^2(\Omega,\mathcal F_0,\mathbb P;\mathbb R^d)$. Under coefficient bounds that hold by construction (Appendix~\ref{app:assumptions}), the layer has a unique strong solution with moment bounds independent of the time grid (Theorem~\ref{thm:wellposed}). We write
$(A_t,b_t,C_t^1,\ldots,C_t^m,
d_t^1,\ldots,d_t^m)$ for the time-indexed coefficient family induced by the model parameters.

These parameters contain static structured coefficients
$(A,b,C^1,\ldots,C^m,d^1,\ldots,d^m)\in\Theta$, where $\Theta$ denotes the set of trainable weights, together with optional
time-feature decoder weights. The matrices $A$ and $C^j$ are chosen from
a fixed structured family: \textbf{diagonal}, \textbf{block-diagonal}, or \textbf{dense}. The additive diffusion vectors $d^j$ may be (jointly) represented densely or through a
low-rank parameterisation.

\begin{remark}[Time-dependent coefficients]
The map $A\mapsto A_t$, and similarly for $b,C^j,d^j$, is optional and makes the coefficients time dependent through a small deterministic time-feature decoder (described in Appendix~\ref{app:time-dependent-coefficients}).
In both cases the layer remains linear in the current state $Z_t$.
\end{remark}
\begin{remark}[Noise type]
    The noise type may be controlled by switching off parts of the diffusion. If $d_t^j=0$ for all $j$, the layer has only multiplicative noise; if $C_t^j=0$ for all $j$, the layer has only additive noise; and if both terms are retained, the layer has general affine diffusion.
\end{remark}

After discretisation, the layer produces affine transition pairs $(F_k,g_k)$ of the form of Eq.~\ref{eq:affine-recurrence}, with the static coefficients of the transitions of Appendix~\ref{app:discretisation} replaced by their time-dependent values at $t_k$.

The choice of matrix structure is important computationally:
diagonal, block-diagonal, and dense matrices are closed under matrix
multiplication, so products of transition matrices remain in the same
family. This closure is what allows the affine maps to be composed by the parallel associative scan with  $O(\log T)$ parallel depth without leaving the chosen representation. By contrast, other
cheap matrix parameterization families such as fixed-rank low-rank or DPLR matrices do not satisfy this property, thus requiring sequential recursion and $O(T)$ complexity.

\subsection{Residual stacking (sequence-model baseline)}
\label{ssec:residual-stacking}
Sequence models such as Mamba~\citep{Gu2023Mamba} and the linear CDE SLiCE~\citep{slice2025} build depth by \emph{residual stacking}: a sequence block $\mathcal S^{(\ell)}$ is applied to a normalised hidden path and added back, $H^{(\ell)} = H^{(\ell-1)} + \mathcal S^{(\ell)}(\operatorname{Norm}(H^{(\ell-1)}))$, so a single hidden path is refined layer by layer. This is the natural depth mechanism to compare against. It is, however, ill-suited to our setting: each SLiSDE layer is itself a Brownian-driven SDE, so reusing the previous layer's already-Brownian path as the driver of the next would turn that layer into a \emph{deterministic} transform of an already-stochastic input, adding compositional depth but injecting no fresh randomness per layer. We therefore do not stack SLiSDE layers residually; instead we couple them through the flow itself, via gated in-flow stacking (Section~\ref{ssec:gated-stacking}).

\subsection{Gated in-flow stacking}
\label{ssec:gated-stacking}
Our main architectural contribution is \emph{gated in-flow stacking}: a selective, path-dependent modulation of the next layer's affine flow, in the spirit of selective state-space models but adapted to stochastic latent dynamics. Layer $\ell$ builds its affine transition pairs $(F_k^{(\ell)},g_k^{(\ell)})$, $k=0,\ldots,N-1$, from its own structured coefficients and Brownian increments. Before they are scanned, causal features $x_k^{(\ell-1)}$ of the previous layer's path (its normalised state and the time features) produce, through two gated linear units, a diagonal scale $\alpha_k^{(\ell)}\in(\mathbf 1-\varepsilon,\mathbf 1+\varepsilon)^d$ and an offset $o_k^{(\ell)}\in\R^d$; the gate architecture, the construction of the gate input and alternative parameterisations are given in Appendix~\ref{ssec:gated_in_flow_app}.

The affine pair is then modified by scaling only the diagonal entries of the transition matrix and by adding the offset correction,
\begin{equation}\label{eq:gated-pair}
        \bar F_k^{(\ell)} = \mathcal D(\alpha^{(\ell)}_k, F^{(\ell)}_k)
    \qquad\textrm{and }\qquad \bar g_k^{(\ell)}
    =
    g_k^{(\ell)}
    +
    o_k^{(\ell)},
\end{equation}
where \(\mathcal D\) rescales the diagonal of \(F^{(\ell)}_k\) entrywise by \(\alpha^{(\ell)}_k\) (Appendix~\ref{ssec:gated_in_flow_app}), and layer \(\ell\) is evaluated by scanning the modified affine recurrence
\begin{equation}\label{eq:gated-recurrence}
    Z_{k+1}^{(\ell)}
    =
    \bar F_k^{(\ell)} Z_k^{(\ell)}
    +
    \bar g_k^{(\ell)}.
\end{equation}
The key point is that gated in-flow introduces nonlinearity through the dependence of \((\bar F_k^{(\ell)},\bar g_k^{(\ell)})\) on the previous path \(Z^{(\ell-1)}\), while preserving affine dependence on the current state \(Z_k^{(\ell)}\); the recurrence is therefore evaluated exactly by the parallel associative scan (Proposition~\ref{prop:affine-scan-closure}(b), Appendix~\ref{app:proofs}).

\subsection{Linear decoder} After the final stacking layer, the model maps the last latent path to the observation space through a learned linear readout:
\begin{equation}
Y_t = \Pi_o Z_t^{(L)},
\qquad     \Pi_o\in\R^{p\times d},     \label{eq:output-readout}
\end{equation}
where \(Z^{(L)}\) denotes the path produced by the final layer and \(\Pi_o\in\Theta\).
When the initial observation is prescribed, we apply an optional shift so that the generated path satisfies the required value at \(t=0\).

\subsection{Discretisation and scan}
\label{ssec:discretisation-scan}
At each time step, every layer builds an affine pair $(F_k,g_k)$ by freezing the time-dependent coefficients of Eq.~\ref{eq:model-time-dependent-sde} at $t_k$, using either the matrix-exponential or the Euler--Maruyama transition of Appendix~\ref{app:discretisation}. The previous-layer path then modifies the pairs to \((\bar F_k^{(\ell)},\bar g_k^{(\ell)})\) as in Eq.~\ref{eq:gated-pair}. Once the previous layers are fixed, each layer is therefore an affine recurrence in its own state, and the modified recurrence is evaluated, layer after layer, by the parallel associative scan of Section~\ref{ssec:linear-sde-affine} (Proposition~\ref{prop:affine-scan-closure}); the strong error of the resulting layerwise Euler--Maruyama scheme is of order $1/2$ for every fixed depth (Theorem~\ref{thm:disc-error}, Appendix~\ref{app:proofs}).

\section{Last-Layer Girsanov Tilt}
\label{sec:girsanov}

Calibration losses on rare events receive little signal from plain Monte-Carlo paths under $\Prob$: an event of probability $p$ estimated from $n$ paths has relative standard error $\sqrt{(1-p)/(pn)}$, so more paths help only as $n^{-1/2}$. We therefore simulate under a tilted measure $\QQ$ that visits the rare region more often and reweight by $\dd\Prob/\dd\QQ$; for diffusion paths Girsanov's theorem (Appendix~\ref{app:girsanov-background}) gives the ratio in closed form for an adapted drift shift, and choosing the shift as a learned controller yields an importance sampler that acts on the constant $1/\sqrt p$ itself.

The shift is confined to the final layer, so that the prefix dynamics calibrated by the vanilla loss are not perturbed. The final layer is driven by a partially independent innovation $\epsilon$, and only the innovation is tilted: under $\QQ$ it is shifted by an adapted control $u_{\theta,k}\in\R^m$, a function of the innovation history and of stop-gradient features of the prefix latent path,
\begin{align}
    \dd W^{(L)}_k &= \rho\,\dd W^{\mathrm{prefix}}_k + \sqrt{1-\rho^2}\,\dd\epsilon_k, \qquad \rho\in[0,1), \label{eq:correlated-driver}\\
    \dd\epsilon^{\QQ}_k &= \dd\epsilon^{\Prob}_k - u_{\theta,k}\,\dd t . \label{eq:eps-tilt}
\end{align}
In a gated stack the random initial latent $z_0\sim\mathcal N(0,I_d)$ carries a large share of the terminal variance, so the tilt also shifts the initial state, $z_0=\mu_\theta+\xi$ with $\xi\sim\mathcal N(0,I_d)$ and a learned $\mu_\theta\in\R^d$. The discrete log-likelihood ratio is then
\begin{equation}\label{eq:rn-discrete}
    \log \frac{\dd\QQ}{\dd\Prob}
    =
    \mu_\theta^{\top}z_0-\tfrac12\|\mu_\theta\|^2+
    \sum_{k=0}^{N-1}
    \left(
        u_{\theta,k}^{\top}\dd\epsilon_k^{\Prob}
        -
        \tfrac12\|u_{\theta,k}\|^2\,\dd t
    \right),
\end{equation}
which is exact, normalised ($\E^{\QQ}[\dd\Prob/\dd\QQ]=1$) and depends only on the control, the shift and the shifted innovation (Theorem~\ref{thm:girsanov-validity}). The tilt is a training-time device: the backbone minimises the calibration loss with the far-tail terms estimated on the tilted batch by self-normalised importance sampling, the controller is trained by the cross-entropy method under a wall on $\KL(\QQ\|\Prob)$, and samples from the trained model are drawn under $\Prob$ without it. The rationale for the initial-state shift, the controller, the two objectives and the gradient convention they require are given in Appendix~\ref{app:girsanov-details}.

\section{Theoretical Results}
\label{sec:theory}

This section states the guarantees behind the architecture; the assumptions in full, the auxiliary results and all proofs are collected in Appendix~\ref{app:proofs}. Three questions are answered here:
\begin{enumerate}[label=(\roman*),leftmargin=1.8em,itemsep=0.1em,topsep=0.2em]
    \item whether the stacked model is a well-posed stochastic process with moment bounds that do not deteriorate with depth (Theorem~\ref{thm:wellposed});
    \item whether the last-layer change of measure yields exact importance weights (Theorem~\ref{thm:girsanov-validity});
    \item what the class can represent (Theorem~\ref{thm:wasserstein-density-terminal-laws}).
\end{enumerate}
Two further facts are used throughout and are stated and proved in the appendix:
\begin{enumerate}[label=(\roman*),leftmargin=1.8em,itemsep=0.1em,topsep=0.2em,resume]
    \item the parallel associative scan returns exactly the sequential iterates of every layer, gated or not, so that the depth-$L$ stack is evaluated exactly in $O(L\log T)$ parallel depth (Proposition~\ref{prop:affine-scan-closure});
    \item the layerwise Euler--Maruyama scheme, with the gates evaluated on the discrete previous layer, has strong order $1/2$ for every fixed depth (Theorem~\ref{thm:disc-error}).
\end{enumerate}

\paragraph{Standing assumptions.} The statements concern the model of Section~\ref{sec:model} under three assumptions, stated in full in Appendix~\ref{app:assumptions}. All of them hold by construction of the architecture; we list them only to fix what the theorems cover.
\begin{itemize}[leftmargin=1.2em,itemsep=0.1em,topsep=0.2em]
    \item \emph{Model class} (Assumption~\ref{ass:model-class}): structured affine layers with the time-dependent coefficients of Appendix~\ref{app:time-dependent-coefficients}, gated in-flow stacking with gates that read the normalised previous-layer state and the time features, arbitrary depth and widths, a linear readout, and a square-integrable initial state whose law may be degenerate.
    \item \emph{Coefficients} (Assumption~\ref{ass:coeff-bounds}): the coefficients of every layer are bounded and regular in time. For the ungated base layer this is automatic, because the time features are smooth and the decoders are fixed networks; the gated layers inherit it from the regularity of the previous layer.
    \item \emph{Gates} (Assumption~\ref{ass:gate-bounds}): the gates are causal and Lipschitz in their input path, the flow gate stays within a fixed distance of one, and the offset gate is bounded for every fixed choice of its weights because its input is normalised.
\end{itemize}
The consequence that drives every constant is that the gated coefficients of a layer obey bounds that depend neither on the layer index nor on the depth of the stack.

\begin{theorem}[Well-posedness]
\label{thm:wellposed}
Let $q\ge2$ and $\E\|z_0\|^q<\infty$.
\begin{enumerate}[label=(\alph*),leftmargin=1.6em,itemsep=0.1em,topsep=0.2em]
\item \emph{Single layer.} Under Assumptions~\ref{ass:model-class} and~\ref{ass:coeff-bounds}, the layer SDE Eq.~\ref{eq:model-time-dependent-sde} has a unique strong solution on $[0,T]$, and
\[
    \E\sup_{0\le t\le T}\|Z_t\|^q
    \le
    K_q\bigl(1+\E\|z_0\|^q\bigr),
\]
where $K_q$ depends only on $(a_*,b_*,c_*,d_*,m,T,q)$.
\item \emph{Gated stack.} Under Assumptions~\ref{ass:model-class}--\ref{ass:gate-bounds}, the depth-$L$ stack of Section~\ref{ssec:gated-stacking} has a unique strong solution, and the bound in (a) holds for every layer $Z^{(\ell)}$, $\ell=1,\dots,L$, with one constant $K_q$, which depends on $(a_*,b_*,c_*,d_*,m,T,q,\varepsilon,\Lambda_o)$ but neither on $\ell$ nor on $L$, for any correlation structure of the layer drivers.
\end{enumerate}
\end{theorem}

\begin{theorem}[Validity of the last-layer Girsanov tilt]
\label{thm:girsanov-validity}
Assume the last-layer controller $u_{\theta,t}$ is progressively measurable with
respect to the final-layer filtration and satisfies Novikov's condition, and let $\mu_\theta\in\R^d$ be a fixed initial-state shift, with $\mu_\theta=0$ unless $z_0\sim\mathcal N(0,I_d)$. Then $\dd\QQ/\dd\Prob=\exp\bigl(\mu_\theta^\top z_0-\tfrac12\|\mu_\theta\|^2\bigr)\,\mathcal E_T(u_\theta)$, with $\mathcal E_T$ the stochastic exponential of Eq.~\ref{eq:girsanov-density},
defines an equivalent measure $\QQ\sim\Prob$ under which $z_0\sim\mathcal N(\mu_\theta,I_d)$, the prefix driver keeps its law, and
$\epsilon_t^\QQ=\epsilon_t-\int_0^t u_{\theta,s}\dd s$ is Brownian, the three being independent. For any integrable
functional $G$ of the full SLiSDE output,
\[
    \E^{\Prob}[G]
    =
    \E^{\QQ} \left[G\,\frac{\dd\Prob}{\dd\QQ}\right].
\]
Moreover, the likelihood ratio depends
only on the final-layer control, the initial-state shift and on the Brownian innovation being shifted.
\end{theorem}

\begin{theorem}[Wasserstein density of terminal laws]
\label{thm:wasserstein-density-terminal-laws}
Fix \(T>0, m \geq 1\) and an output dimension \(p\in\mathbb N\), and let the gated layers use the RMS normalisation with a strictly positive regulariser and a gate input containing the time coordinate, as in Assumption~\ref{ass:model-class}.
For every prescribed initial condition
\(
    z_0\in
    L^2(\Omega,\mathcal F_0,\mathbb P;\mathbb R^{d_0}),
\)
possibly having a degenerate law, the collection of terminal laws
generated by the model is dense in
\(\mathcal P_2(\mathbb R^p)\) with respect to the \(2\)-Wasserstein
distance. More precisely, if
\[
    \mathfrak L_T
    :=
    \left\{
        \mathcal L(Y_T^\theta):
        \theta \text{ is an admissible model configuration}
    \right\},
\]
where $\theta$ ranges over the admissible configurations of all latent widths $d\ge d_0$, then
\[
    \overline{\mathfrak L_T}^{\,W_2}
    =
    \mathcal P_2(\mathbb R^p).
\]
One ungated layer followed by one gated layer suffices.
\end{theorem}

Theorem~\ref{thm:wasserstein-density-terminal-laws} is what makes the structural restrictions harmless: every target law with finite second moment is approximated arbitrarily well by the terminal law of an admissible configuration, so the selection among the diffusions consistent with the calibration constraints (Section~\ref{sec:objective}) is never forced by a limitation of the class.

\section{Experiments}
\label{sec:experiments}

We evaluate SLiSDE on three main benchmarks: a path-dependent toy dynamical system
(\textsc{toy}) and two option-surface datasets (\textsc{dax} and \textsc{spx}), where the
\textsc{toy} benchmark tests whether a model can match nonlinear
long-horizon path statistics generated by a highly nonlinear stochastic
dynamical system.
Dataset construction,
training protocols, and search grids are described in
Appendix~\ref{app:expset}. Additional results -- ablations over the structural choices, the timing study, distributional recovery on \textsc{toy} and the Girsanov tilt studies -- are collected in Appendix~\ref{app:additional-results}.

\paragraph{Models and baselines.} We compare SLiSDE against two baselines: (i) a fully connected \textbf{Neural SDE} with neural drift and diffusion, and (ii) \textbf{SLiCE} \citep{slice2025}. SLiCE is a natural structured-sequence baseline because it is a scan-compatible linear CDE and it is more expressive than Mamba \citep{Gu2023Mamba}; in our experiments
we drive it with Brownian motion to match the stochastic input used by
SLiSDE. For all model families, we report results for both two-layer and three-layer configurations.

\paragraph{Complexities.}
Table~\ref{tab:matrix-complexities} reports the recurrent and scan
costs for the SLiSDE transition structures. The scan reduces sequential
depth from \(O(T)\) to \(O(\log T)\) when affine composition is available.
For a generic Neural SDE, no analogous parallel associative scan is available,
so simulation remains sequential in time.
\begin{table}[!ht]
\centering
\small
\captionsetup{font=footnotesize}
\renewcommand{\arraystretch}{1.15}
\begin{tabular}{lccc|c}
\toprule
 & Diagonal & Block-diagonal & Dense & NeuralSDE \\
\midrule
Recurrent cost & $O(d T)$ & $O(n_b b_s^2T)$ & $O(d^2 T)$ & $O(d^2 T)$ \\
Scan cost & $O(d\log T)$ & $O(n_b b_s^3\log T)$ & $O(d^3\log T)$ & --- \\
\bottomrule
\end{tabular}
\vspace{0.1cm}
\caption{Computational cost of composing affine transition maps over $T$ time steps. Here $d$ is the latent dimension, $b_s$ is the block size, and $n_b=d/b_s$ is the number of blocks. Typically $n_b\ll d$. }
\label{tab:matrix-complexities}
\end{table}
Dense transitions remain
computationally expensive because each scan composition requires dense
matrix multiplication; this can offset the practical gains from
parallelisation, and we therefore exclude dense SLiSDE variants from the
main experimental analysis.

\paragraph{Implementation and Compute.}
All models and experiments are implemented in JAX. The code is available at \url{https://anonymous.4open.science/r/SLiSDE-3B86/}.
All experiments were run on a single NVIDIA RTX 6000 Ada Generation GPU.

\subsection{Path-dependent toy benchmark (\textsc{TOY})}

\textsc{toy} is a controlled path-functional benchmark generated from a
highly nonlinear stochastic dynamical system on \([0,1]\), discretised with
\(N=2048\) time steps (see Appendix~\ref{app:expset}). The task is to calibrate the model law to the functional targets of Appendix~\ref{app:expset} -- thresholded positive parts, running maximum, squared path average and threshold-crossing probabilities at five evaluation times -- through the objective of Eq.~\ref{eq:calib-loss}. In this first experiment, all models are trained
without the Girsanov tilt, so the comparison isolates the effect of the
architecture and stacking mechanism.

\begin{table*}[!ht]
\centering
\footnotesize
\renewcommand{\arraystretch}{1.1}
\captionsetup{font=footnotesize}
\setlength{\tabcolsep}{4pt}
\begin{tabular*}{\textwidth}{@{\extracolsep{\fill}} l c c c c @{}}
\toprule
Model & Params & Loss ($10^{-4}$) & Barrier MAE ($10^{-2}$) & ms / epoch \\
\midrule
Neural SDE -- $L=2$ & 124\textsc{K} & $1.557\pm0.206$ & $1.058\pm0.110$ & 91.7 \\
SLiCE -- $L=2$ & 142\textsc{K} & $4.690\pm0.423$ & $2.168\pm0.102$ & 171.8 \\
\textbf{SLiSDE gated in-flow -- $L=2$} & 103\textsc{K} & $\mathbf{0.852\pm0.168}$ & $\mathbf{0.658\pm0.054}$ & \textbf{72.9} \\
\midrule
Neural SDE -- $L=3$ & 157\textsc{K} & $3.024\pm0.341$ & $1.637\pm0.120$ & 110.4 \\
SLiCE -- $L=3$ & 213\textsc{K} & $5.006\pm0.914$ & $2.250\pm0.252$ & 248.2 \\
\textbf{SLiSDE gated in-flow -- $L=3$} & 74\textsc{K} & $\mathbf{0.659\pm0.059}$ & $\mathbf{0.672\pm0.057}$ & \textbf{68.9} \\
\bottomrule
\end{tabular*}
\par\vspace{2pt}
\begin{minipage}{\textwidth}\scriptsize
Mean $\pm$ standard error over seven seeds per cell; best result per column in bold; the loss is $\mathcal L_n(\theta)$ of Eq.~\ref{eq:calib-loss} on a fresh Monte-Carlo sample.
\end{minipage}
\caption{\textsc{toy} results without Girsanov tilt. $L$ is the number of stacked layers; each row is the tuned cell of the search grid for that family and depth, retrained under the common protocol of Appendix~\ref{app:expset}. We report parameter count, calibration loss, threshold-crossing error, and time per epoch.}
\label{tab:toy3-vanilla}
\end{table*}

Table~\ref{tab:toy3-vanilla} shows that SLiSDE substantially improves both accuracy and efficiency on the toy path-functional benchmark. The gated in-flow variant achieves the lowest total loss for both \(L=2\) and \(L=3\), while using fewer parameters than the Neural SDE and SLiCE baselines. In terms of runtime, SLiSDE remains faster than the Neural SDE baseline, reflecting the benefit of scan-compatible structured dynamics.
SLiCE, which shares the parallel associative scan, is not designed for random inputs: it treats the Brownian path as one more input sequence, and its gap shows that deterministic sequence models do not transfer to a task that requires genuinely nonlinear stochastic dynamics.
The recovery of the full path law beyond the calibrated targets (Kolmogorov--Smirnov and Wasserstein-1 distances to the ground truth for the terminal value and the running maximum) is reported in Appendix~\ref{app:additional-results}.

\subsection{Option-surface benchmark (\textsc{DAX})}

\textsc{dax} is a real-data functional calibration benchmark built from
historical European option quotes on the DAX index across multiple
threshold levels and maturities. The task is to learn a stochastic path
generator whose terminal functionals match the observed option surface. For
evaluation, we hold out \(25\%\) of the threshold levels in each target
group using a structured non-contiguous mask, so that the model is tested
on interpolation across the surface rather than extrapolation outside the
observed range.

The same protocol is applied to the \textsc{spx} surface (Appendix~\ref{app:expset}): there SLiSDE matches the held-out loss of SLiCE with two layers and attains the lowest held-out and far-tail loss with three layers, with $6$--$8\times$ fewer parameters and $3.5$--$4\times$ faster epochs, and both structured models are far ahead of the Neural SDE (Table~\ref{tab:spx-full}, Appendix~\ref{app:additional-results}).

Out-of-the-money (OTM) options correspond to thresholds that are far from the current level of the underlying index, so their value is determined by paths that end in the tails of the distribution.

Table~\ref{tab:dax3-main} shows the real-data benchmark. SLiSDE gated in-flow attains the lowest held-out loss at both depths (with two layers within one standard error of SLiCE) and, with three layers, also the lowest far-tail loss of all models, at a parameter count below SLiCE's and in about half of its time per epoch; with two layers SLiCE fits the far puts slightly better, within one standard error. The Neural SDE is the weakest at both depths.

\begin{table*}[!ht]
\centering
\footnotesize
\captionsetup{font=footnotesize}
\renewcommand{\arraystretch}{1.1}
\setlength{\tabcolsep}{4pt}
\begin{tabular*}{\textwidth}{@{\extracolsep{\fill}} l c c c c @{}}
\toprule
Model & Params & Loss ($10^{-4}$) & Far put ($10^{-6}$) & ms / epoch \\
\midrule
Neural SDE -- $L=2$ & 79\textsc{K} & $1.972\pm0.117$ & $9.87\pm2.47$ & 91.4 \\
SLiCE -- $L=2$ & 142\textsc{K} & $1.605\pm0.149$ & $\mathbf{8.14\pm1.69}$ & 169.9 \\
\textbf{SLiSDE gated in-flow -- $L=2$} & 90\textsc{K} & $\mathbf{1.552\pm0.070}$ & $9.40\pm2.30$ & \textbf{79.7} \\
\midrule
Neural SDE -- $L=3$ & 157\textsc{K} & $1.908\pm0.180$ & $8.20\pm1.42$ & \textbf{114.4} \\
SLiCE -- $L=3$ & 213\textsc{K} & $1.560\pm0.086$ & $6.41\pm1.33$ & 246.9 \\
\textbf{SLiSDE gated in-flow -- $L=3$} & 135\textsc{K} & $\mathbf{1.401\pm0.244}$ & $\mathbf{4.07\pm0.62}$ & 129.1 \\
\bottomrule
\end{tabular*}
\par\vspace{2pt}
\begin{minipage}{\textwidth}\scriptsize
Mean $\pm$ standard error over seven seeds per cell; best result per column in bold; losses are $\mathcal L_n(\theta)$ of Eq.~\ref{eq:calib-loss} on the held-out strikes, from a fresh Monte-Carlo sample.
\end{minipage}
\caption{\textsc{dax} results without Girsanov tilt. $L$ is the number of stacked layers; each row is the tuned cell of the search grid for that family and depth, retrained under the common protocol of Appendix~\ref{app:expset}. We report parameter count, held-out loss, held-out far-put (OTM) loss, and time per epoch.}
\label{tab:dax3-main}
\end{table*}

\paragraph{Girsanov tilt.} With the tilt of Section~\ref{sec:girsanov} the far-tail calibration of the \textsc{dax} model improves at a fixed path budget: over seven seeds the mean far-call held-out error falls by about a third relative to the vanilla model ($0.64\times$; lower on six seeds of seven, paired $t$-test on the log-errors $p=0.03$) and the mean total held-out loss drops by $15$--$24\%$, while the budget-matched untilted model is never better on the far tails (Table~\ref{tab:dax3-girsanov}). In a controlled rare-event study on \textsc{toy}, the tilt with the vanilla importance-sampling estimator reduces the relative error of plain Monte Carlo at equal budget by a factor of $2$ to $4$, growing with the rarity of the event (Table~\ref{tab:girsanov-estimator}); both studies are reported in Appendix~\ref{app:additional-results}.

\paragraph{Timing.} Wall-clock times of full training steps for batch sizes up to $4096$ and horizons up to $T=8192$ are reported in Appendix~\ref{app:additional-results}. At equal parameter count, SLiSDE evaluated sequentially is $6$--$7\times$ faster per step than the Neural SDE at $B=64$ and $1.3$--$1.4\times$ faster at $B=1024$; only at $B=4096$, where the GPU is saturated by batch parallelism, is the Neural SDE step faster ($179$ against $242$\,ms).
\section{Conclusion}
\label{sec:conclusion}
We introduced \emph{SLiSDE}, a family of Neural SDEs that moves nonlinearity from per-step drift and diffusion evaluations to \emph{gated in-flow} coupling between structured linear stochastic layers. Each layer remains affine in its own state and is therefore parallel-scannable, while cross-layer modulation, time-dependent coefficients, and a learned linear decoder provide expressivity. For tail-sensitive path-functional objectives, we also use a last-layer Girsanov tilt, together with a shift of the initial latent state, with a closed-form Radon--Nikodym correction, yielding a lightweight importance-sampling overlay. We prove well-posedness, scan compatibility, discretisation error bounds, and validity of the change of measure, and that the terminal laws of the gated stack are dense in $\mathcal P_2$: the structural restrictions cost nothing in expressive power at the level of terminal laws.

Empirically, SLiSDE attains a lower calibration loss than the fully neural SDE baseline on every benchmark and at both depths, and trains faster in all but two of the settings tested, the three-layer \textsc{dax} configuration and the largest batch size of the timing study. Against the structured-CDE baseline it is clearly better on \textsc{toy}; on the option surfaces it attains the lowest mean loss on \textsc{dax} at both depths and on the three-layer \textsc{spx} surface and matches it on the two-layer \textsc{spx} surface, differences that are within one standard error, and with three layers it attains the lowest far-tail error on both surfaces. Throughout, it uses $1.4$--$8\times$ fewer parameters than the structured-CDE baseline and trains $2$--$4\times$ faster per epoch. The tilt is a variance-reduction device: on \textsc{toy} it reduces the error of rare-event estimates by factors of two to four at equal budget, growing with the rarity of the event, and on the far tails of \textsc{dax} it improves the far-tail calibration of the vanilla model, while budget-matched plain Monte Carlo remains a strong competitor where the targets are not rare. The results suggest that structured linear stochastic layers, combined with selective in-flow coupling and optional importance sampling, are a practical alternative to fully neural drift--diffusion networks for long-horizon functional calibration.

\paragraph{Future work.} From the stochastic-analysis point of view, a
natural direction is to better understand gated in-flow stacking as a
path-dependent transformation of linear diffusions, including its possible
connections to filtering. A second direction is to study structured Neural
SDEs beyond functional calibration, in particular as continuous-time
generative models such as Neural-SDE GANs.

\newpage
\bibliographystyle{plainnat}
\bibliography{biblio}

\newpage
\appendix

\section{Mathematical Background Continued}
\label{app:background-continued}

\subsection{Girsanov change of measure}
\label{app:girsanov-background}
This subsection states the change of measure on which the tilt of Section~\ref{sec:girsanov} rests.

Let $u_t$ be a progressively measurable $\R^m$-valued process satisfying
Novikov's condition
\begin{equation}\label{eq:novikov_cond}
    \E^{\Prob} \left[
        \exp \left(
            \frac12\int_0^T \|u_t\|^2\,\dd t
        \right)
    \right]<\infty .
\end{equation}
Define
\begin{equation}
    \mathcal E_T(u)
    =
    \exp \left(
        \int_0^T u_t^\top \dd W_t
        -
        \tfrac12\int_0^T \|u_t\|^2\,\dd t
    \right).
    \label{eq:girsanov-density}
\end{equation}
Then $\mathcal E_T(u)$ is a martingale density. If
$\dd\QQ/\dd\Prob=\mathcal E_T(u)$, then
\[
    W_t^\QQ = W_t - \int_0^t u_s\,\dd s
\]
is a Brownian motion under $\QQ$. Moreover, for any integrable
$\F_T$-measurable path functional $G$,
\begin{equation}
    \E^{\Prob}[G]
    =
    \E^{\QQ} \left[
        G\,\frac{\dd\Prob}{\dd\QQ}
    \right].
    \label{eq:rn-expectation-app}
\end{equation}

The inverse density $\dd\Prob/\dd\QQ$ is the importance weight that turns
expectations under the tilted measure back into reference-measure
expectations.

\subsection{Calibration objective: population limit, non-uniqueness and model selection}
\label{app:objective}
\paragraph{Population limit.} With the notation of Section~\ref{sec:objective}, assume $\varphi(Y)\in L^2(\Prob_\theta)$ for every $\varphi\in\Phi$. Since $\widehat\E_{\theta,n}[\varphi(Y)]$ is an average of $n$ independent copies,
\[
    \E\Bigl[\bigl(\widehat\E_{\theta,n}[\varphi(Y)]-c_\varphi\bigr)^2\Bigr]
    =
    \bigl(\E_{\Prob_\theta}[\varphi(Y)]-c_\varphi\bigr)^2
    +\frac1n\operatorname{Var}_{\Prob_\theta}\bigl(\varphi(Y)\bigr),
\]
and by the strong law of large numbers, almost surely as $n\to\infty$,
\[
    \mathcal L_n(\theta)\longrightarrow\mathcal L_\infty(\theta)
    =
    \sum_{\varphi\in\Phi}w_\varphi\bigl(\E_{\Prob_\theta}[\varphi(Y)]-c_\varphi\bigr)^2 .
\]
The objective is therefore a functional of the model \emph{law}. The decomposition
\[
    \E\|X'_t-X_t\|^2
    =
    \operatorname{Var}(X'_t)+\operatorname{Var}(X_t)+\|\E X'_t-\E X_t\|^2 ,
\]
which shows that a samplewise squared error between model and data paths rewards collapse of the model onto the data mean, does not apply: no model path is paired with a data path, and the only variance term in $\mathcal L_n$ carries the coefficient $1/n$ and vanishes in the population limit. In all experiments $n\ge1024$.

\paragraph{No collapse.} Collapse is moreover incompatible with the population objective whenever the target marginal is non-degenerate. If $Y_t=m_t$ almost surely then
\[
    \E_{\Prob_\theta}[(Y_t-K)_+]=(m_t-K)_+ ,
\]
a piecewise affine function of $K$ with a single change of slope at $K=m_t$, which cannot coincide on an interval with a strictly convex target map $K\mapsto c_{t,K}$ such as an arbitrage-free call surface. More generally
\[
    \E_{\Prob_\theta}[(Y_t-K)_+]=\int_K^\infty\Prob_\theta(Y_t>x)\,\dd x ,
\]
so knowledge of the positive-part expectations for every $K$ determines the marginal law of $Y_t$; the finite strike grid of the implementation provides a finite collection of constraints of this type, and the running-maximum and threshold-crossing functionals constrain path-dependent features in addition. We do not claim that a finite collection of such constraints determines the entire path law.

\paragraph{Non-uniqueness and model risk.} Finitely many expectation constraints are satisfied by infinitely many diffusions, so $\mathcal L_\infty$ has no unique minimiser. \emph{This is a property of functional calibration rather than of the objective: the data are finitely many expectations, so no objective built on them can identify the path law.} In derivatives markets one typically has nothing beyond quoted prices, and it is standard practice to work with -- and deliberately explore -- the multiplicity of models consistent with the same quotes in order to quantify model risk; \citet{Gierjatowicz2020RobustPricing}, for example, use randomised calibrated neural SDEs to obtain price \emph{bounds} over the calibrated set. Enforcing uniqueness through an auxiliary term would not remove the ambiguity but shift it to the choice of the reference measure. Theorem~\ref{thm:wasserstein-density-terminal-laws} guarantees that the model class itself never restricts what is reachable, so the selection among constraint-consistent processes is driven by the data, the initialisation and the training procedure, never by a hidden limitation of the class. Operationally we validate the selection where a selection can be validated, on functionals the model was never fitted to: the held-out-strike transfer on \textsc{dax} and \textsc{spx}, and the distributional-recovery study on \textsc{toy} (Kolmogorov--Smirnov and Wasserstein-1 distances to the ground-truth simulator, Appendix~\ref{app:additional-results}).

\paragraph{Model selection within the calibrated set.} There are settings, physics in particular, where prior information \emph{is} available: a reference dynamics, an equilibrium law, known symmetries. There, an additive penalty turns Eq.~\ref{eq:calib-loss} into a principled selection rule -- a Kullback--Leibler divergence to the available reference \citep{Smith2026DistributionalConstraints}, a minimal-relative-entropy criterion in the spirit of \citet{Avellaneda1997RelativeEntropy} for volatility surfaces, or a context-specific penalty -- and such terms drop into the training loop unchanged. The framework already computes and uses exactly this quantity where a reference measure genuinely exists: in the Girsanov layer, the controller is trained with a $\KL(\QQ\|\Prob)$ anchor (Appendix~\ref{app:kl-regulariser}), the pathwise log Radon--Nikodym exponent that our machinery evaluates, which penalises tilts that move too far from the untilted reference model.

\paragraph{Relation to generative training.} Training against a fixed family of statistics can be viewed as a maximum-mean-discrepancy objective with a non-characteristic kernel, a point discussed by \citet{Kidger2021NSDEsAsInfiniteDimensional}, who contrast it with their adversarial formulation using a learned discriminator; the variational formulation of \citet{Tzen2019} and the signature-kernel scores of \citet{Issa2023SignatureKernelScores} likewise address a different, generative learning objective that requires sample paths. The method belongs instead to the calibration strand \citep{Gierjatowicz2020RobustPricing,Cuchiero2020GANCalibration,Cohen2021ArbitrageFree}, in which prescribed functional expectations, such as derivative prices, are matched.

\newpage
\section{Theoretical Results: Assumptions, Auxiliary Results and Proofs}
\label{app:proofs}

This appendix states the standing assumptions in full, establishes the two auxiliary results used throughout the paper (exactness of the parallel associative scan and the strong discretisation error of the stacked scheme) and proves every result of Section~\ref{sec:theory}. It is organised as follows.
\begin{itemize}[leftmargin=1.2em,itemsep=0.1em,topsep=0.2em]
    \item Appendix~\ref{app:assumptions} states the assumptions and proves Lemma~\ref{lem:uniform-bounds}: a gated layer inherits the coefficient bounds of its base layer with constants that do not depend on the depth. Every later constant rests on this lemma.
    \item Appendix~\ref{app:proof-wellposed} proves Theorem~\ref{thm:wellposed}, first for one layer with random coefficients (Lemma~\ref{lem:one-layer}) and then by induction over the layers.
    \item Appendix~\ref{app:scan} proves Proposition~\ref{prop:affine-scan-closure}, the exactness of the scan.
    \item Appendix~\ref{app:proof-disc} proves Theorem~\ref{thm:disc-error} in four steps: one-step moments (Lemma~\ref{lem:onestep}), discrete stability (Lemma~\ref{lem:stab}), the Euler--Maruyama error of one layer (Lemma~\ref{lem:euler}) and an induction over the layers that carries a doubled moment order at every step down the stack.
    \item Appendices~\ref{app:girsanov-validity} and~\ref{app:proof-density} prove Theorems~\ref{thm:girsanov-validity} and~\ref{thm:wasserstein-density-terminal-laws}.
\end{itemize}

\paragraph{Notation.}
$\|\cdot\|$ denotes the Euclidean norm of a vector and the operator norm of a matrix, $\|\cdot\|_{\mathrm F}$ the Frobenius norm and $\|x\|_\infty$ the largest modulus of the entries of $x$. For a coefficient tuple $\vartheta=(A,b,C^1,\dots,C^m,d^1,\dots,d^m)$ we write
\begin{equation}
    \|\vartheta\|
    :=
    \|A\|+\|b\|+\sum_{j=1}^m\bigl(\|C^j\|+\|d^j\|\bigr).
    \label{eq:app-coeff-norm}
\end{equation}
For a coefficient process $\vartheta_t=(A_t,b_t,C^1_t,\dots,C^m_t,d^1_t,\dots,d^m_t)$ we set
\begin{equation}
    \mu(t,z):=A_tz+b_t,
    \qquad
    \sigma^j(t,z):=C^j_tz+d^j_t,
    \qquad
    \sigma(t,z):=\bigl[\sigma^1(t,z),\dots,\sigma^m(t,z)\bigr]\in\R^{d\times m},
    \label{eq:app-mu-sigma}
\end{equation}
so that the layer SDE Eq.~\ref{eq:model-time-dependent-sde} reads
\[
    \dd Z_t=\mu(t,Z_t)\,\dd t+\sigma(t,Z_t)\,\dd W_t .
\]
The letter $C$ denotes a constant that may change from line to line and depends only on the bounds in the assumptions, on $T$, on $m$ and on the moment order under consideration; where a dependence matters it is made explicit. We use the Burkholder--Davis--Gundy inequality \citep[Ch.~IV]{ProtterSDE} in the form
\begin{equation}
    \E\sup_{s\le t}\Bigl\|\int_0^s H_u\,\dd W_u\Bigr\|^q
    \le
    c_q\,\E\Bigl(\int_0^t\|H_u\|_{\mathrm F}^2\,\dd u\Bigr)^{q/2},
    \qquad q\ge2,
    \label{eq:app-bdg}
\end{equation}
valid for every progressively measurable $\R^{d\times m}$-valued process $H$ with finite right-hand side. Finally,
\[
    m_r:=\E|\xi|^r,
    \qquad
    \xi\sim\mathcal N(0,1),
\]
denotes the absolute moments of the standard normal law.

\subsection{Standing assumptions}
\label{app:assumptions}

\begin{assumption}[Admissible model class]
\label{ass:model-class}
Fix $T>0$, a Brownian dimension $m\ge1$ and an output dimension $p\in\mathbb N$. Let $W=(W^1,\dots,W^m)$ be an $m$-dimensional Brownian motion on the filtered probability space of Section~\ref{sec:background} and fix $z_0\in L^2(\Omega,\F_0,\Prob;\R^{d_0})$ for some $d_0\in\mathbb N$, without any non-degeneracy assumption on its law. For every $d\ge d_0$ set
\[
    \iota_d(z_0):=(z_0,0,\dots,0)\in\R^d .
\]
An admissible model configuration consists of a finite number $L$ of stacked layers of a common latent width $d\ge d_0$, each driven by a Brownian motion $W^{(\ell)}$ adapted to the common filtration (shared, independent or correlated across layers; the implementation shares one driver across the prefix layers and uses the correlated driver of Eq.~\ref{eq:correlated-driver} for the last one), such that:
\begin{enumerate}[leftmargin=1.6em,itemsep=0.1em,topsep=0.2em]
\item each base layer is the structured affine It\^o SDE Eq.~\ref{eq:model-time-dependent-sde} with $Z_0=\iota_d(z_0)$, whose coefficients $A_t,C^j_t\in\R^{d\times d}$ and $b_t,d^j_t\in\R^d$ are the static parameters $A,C^j,b,d^j$ perturbed by the time-feature decoder of Appendix~\ref{app:time-dependent-coefficients}, which adds deterministic, smooth-in-time corrections to the free coefficient entries and respects the chosen diagonal, block-diagonal or dense structure;
\item successive layers are connected by gated in-flow stacking (Section~\ref{ssec:gated-stacking}) through the gates of Eqs.~\ref{eq:alpha-gate} and~\ref{eq:offset-gate}, with fixed $\varepsilon\in(0,1)$, and the gate input is
\begin{equation}
    x^{(\ell-1)}_t
    =
    \bigl[\Pi(Z^{(\ell-1)}_t),\,\tau(t)\bigr],
    \qquad
    \Pi(z):=\gamma\odot\frac{z}{\sqrt{\varsigma+d^{-1}\|z\|^2}},
    \label{eq:app-gate-input}
\end{equation}
where $\Pi$ is the RMS normalisation with gain $\gamma\in\R^d$ and regulariser $\varsigma>0$, and $\tau$ is the vector of deterministic time features, which contains the time coordinate $t$ and is Lipschitz on $[0,T]$; all gate weights and biases are trainable, and the optional causal convolution is disabled;
\item the number of layers, the latent width and the hidden widths of the time decoders and gates may be increased arbitrarily, and their affine weights and biases range freely over the corresponding Euclidean spaces;
\item the output is a linear readout of the final layer,
\[
    Y^\theta_t=\Pi_oZ^{(L)}_t,
    \qquad
    \Pi_o\in\R^{p\times d},
\]
shifted to match a prescribed initial observation $y_0$ when one is given.
\end{enumerate}
\end{assumption}

\begin{remark}[Discrete and continuous gating]
\label{rem:gate-conventions}
The implementation gates the discrete affine pair,
\begin{equation}
    \bar F_k=\mathcal D(\alpha_k,F_k),
    \qquad
    \bar g_k=g_k+o_k,
    \label{eq:app-discrete-gating}
\end{equation}
with no explicit $\dt$ on the offset, where $\mathcal D(\alpha,F)$ rescales the $i$-th diagonal entry of $F$ by $\alpha_i$. The continuous-time model analysed below carries the gated coefficients
\begin{equation}
    \alpha^{(\ell)}_t\odot A^{(\ell)}_t
    \quad\text{in place of}\quad A^{(\ell)}_t,
    \qquad
    b^{(\ell)}_t+o^{(\ell)}_t
    \quad\text{in place of}\quad b^{(\ell)}_t,
    \label{eq:app-continuous-gating}
\end{equation}
where $\alpha\odot A$ rescales the $i$-th diagonal entry of $A$ by $\alpha_i$, while the diffusion coefficients are not gated. The two conventions are related by the dictionary
\begin{equation}
    o_k=o_{t_k}\,\dt,
    \qquad
    \alpha_k=\exp\bigl(\delta_{t_k}\dt\bigr)\ \text{entrywise},
    \label{eq:app-dictionary}
\end{equation}
where $\delta$ denotes the diagonal drift modulation. The first identity is exact because the linear branch of the offset gate is homogeneous in its weights, so that rescaling $W_{io},c_{io}$ by $\dt$ stays within the family. The second is realisable on every fixed grid because $\alpha_k$ ranges over a neighbourhood of $\mathbf 1$.
\end{remark}

\begin{assumption}[Coefficients]
\label{ass:coeff-bounds}
For every layer, the coefficient process $\vartheta_t$ is $(\F_t)$-progressively measurable and uniformly bounded,
\begin{equation}
    \|A_t\|\le a_*,
    \qquad
    \|b_t\|\le b_*,
    \qquad
    \|C^j_t\|\le c_*,
    \qquad
    \|d^j_t\|\le d_*,
    \label{eq:app-coeff-bounds}
\end{equation}
almost surely, for all $t\in[0,T]$ and $j\le m$, and $\tfrac12$-H\"older in time in every $L^q$: for each $q\ge2$ there is $L_q<\infty$ with
\begin{equation}
    \E\|\vartheta_t-\vartheta_s\|^q\le L_q^q\,|t-s|^{q/2},
    \qquad s,t\in[0,T].
    \label{eq:treg}
\end{equation}
\end{assumption}

For the ungated base layer the coefficients are deterministic and Lipschitz in $t$ with a constant $L_\tau$, because the time features are smooth and the decoders are fixed networks. Eq.~\ref{eq:treg} then holds with
\[
    L_q=L_\tau T^{1/2},
\]
so Assumption~\ref{ass:coeff-bounds} holds by construction. For the gated layers it is established in Lemma~\ref{lem:uniform-bounds} below. The bounds Eq.~\ref{eq:app-coeff-bounds} imply, for the maps of Eq.~\ref{eq:app-mu-sigma}, the Lipschitz and linear-growth bounds
\begin{equation}
\begin{gathered}
    \|\mu(t,z)-\mu(t,z')\|+\sum_{j=1}^m\|\sigma^j(t,z)-\sigma^j(t,z')\|
    \le
    \Lambda'\|z-z'\|,\\
    \|\mu(t,z)\|+\|\sigma(t,z)\|_{\mathrm F}
    \le
    \Lambda'\bigl(1+\|z\|\bigr),
\end{gathered}
    \label{eq:linbounds}
\end{equation}
almost surely for all $t$, $z$ and $z'$, with
\begin{equation}
    \Lambda'
    :=
    a_*+b_*+m\,(c_*+d_*).
    \label{eq:app-lambda-prime}
\end{equation}

\begin{assumption}[Gates]
\label{ass:gate-bounds}
For each layer $\ell\ge2$ the gate map sends a continuous path $z=(z_t)_{t\in[0,T]}$ of the previous layer to the gate processes,
\[
    z\longmapsto\bigl(\alpha^{(\ell)}_t(z),\,o^{(\ell)}_t(z)\bigr)_{t\in[0,T]} .
\]
It is causal, in that $(\alpha^{(\ell)}_t,o^{(\ell)}_t)$ depends on $z$ only through $z_{[0,t]}$ (for the gates of Eqs.~\ref{eq:alpha-gate} and~\ref{eq:offset-gate} with the input Eq.~\ref{eq:app-gate-input}, only through $z_t$ and $t$), and jointly measurable, so that applied to a continuous adapted process it yields progressively measurable processes. Moreover, with $\Lambda_\Pi$ and $\Lambda_x$ defined in (a):
\begin{enumerate}[leftmargin=1.6em,label=(\alph*),itemsep=0.1em,topsep=0.2em]
\item \emph{(Normalised input.)} For all $z,z'\in\R^d$,
\begin{equation}
    \|\Pi(z)\|\le\sqrt d\,\|\gamma\|_\infty=:\Lambda_\Pi,
    \qquad
    \|\Pi(z)-\Pi(z')\|\le L_\Pi\|z-z'\|,
    \qquad
    L_\Pi:=\frac{\|\gamma\|_\infty}{\sqrt\varsigma},
    \label{eq:app-pi-bounds}
\end{equation}
so that the gate input satisfies
\[
    \|x^{(\ell-1)}_t\|\le\Lambda_\Pi+\sup_{t\le T}\|\tau(t)\|=:\Lambda_x .
\]
\item \emph{(Bounded flow gate.)} Surely, for every input path and every choice of the weights,
\[
    \|\alpha^{(\ell)}_t-\mathbf1\|_\infty\le\varepsilon .
\]
\item \emph{(Offset bounded for fixed weights.)} Surely,
\begin{equation}
    \|o^{(\ell)}_t\|
    \le
    \|W^{(\ell)}_{io}\|\,\Lambda_x+\|c^{(\ell)}_{io}\|
    =:\Lambda_o<\infty;
    \label{eq:app-offset-bound}
\end{equation}
$\Lambda_o$ is finite for every fixed choice of the weights and free in magnitude.
\item \emph{(Lipschitz gate.)} There is $L_g<\infty$ such that, for all continuous paths $z,z'$ and all $s\le t\le T$,
\begin{align}
    \|\alpha^{(\ell)}_t(z)-\alpha^{(\ell)}_t(z')\|_\infty+\|o^{(\ell)}_t(z)-o^{(\ell)}_t(z')\|
    &\le
    L_g\sup_{r\le t}\|z_r-z'_r\|,
    \label{eq:app-gate-lip-path}\\
    \|\alpha^{(\ell)}_t(z)-\alpha^{(\ell)}_s(z)\|_\infty+\|o^{(\ell)}_t(z)-o^{(\ell)}_s(z)\|
    &\le
    L_g\bigl(\|z_t-z_s\|+|t-s|\bigr).
    \label{eq:app-gate-lip-time}
\end{align}
\end{enumerate}
\end{assumption}

All four properties hold for the implemented gates.

\medskip
\noindent
\emph{Property (a).}
Write
\[
    s(z):=\sqrt{\varsigma+d^{-1}\|z\|^2},
    \qquad
    f(z):=\frac{z}{s(z)},
\]
so that $\Pi(z)=\gamma\odot f(z)$ and $\|f(z)\|\le\sqrt d$. The Jacobian of $f$ is the symmetric matrix
\[
    Df(z)=s(z)^{-1}I-s(z)^{-3}d^{-1}zz^\top .
\]
Its eigenvalue on $z^\perp$ is $s(z)^{-1}$, and its eigenvalue in the direction of $z$ is
\[
    s(z)^{-1}-s(z)^{-3}d^{-1}\|z\|^2
    =
    s(z)^{-3}\varsigma ;
\]
all eigenvalues therefore lie in $(0,\varsigma^{-1/2}]$. Hence $\|\Pi(z)\|\le\|\gamma\|_\infty\sqrt d$, and $\Pi$ is Lipschitz with constant $\|\gamma\|_\infty\varsigma^{-1/2}$.

\medskip
\noindent
\emph{Properties (b) and (c).}
They hold because $|\sigma|\le1$ and $|\tanh|\le1$ in Eq.~\ref{eq:alpha-gate}, and because $|\sigma|\le1$ in Eq.~\ref{eq:offset-gate} while the gate input lies in the ball of radius $\Lambda_x$.

\medskip
\noindent
\emph{Property (d).}
The gates are compositions of $\Pi$, affine maps and the Lipschitz activations $\sigma$ and $\tanh$, all restricted to the bounded input set, together with the Lipschitz time features. The product structure of Eq.~\ref{eq:offset-gate} is Lipschitz on that bounded set because both factors are bounded and Lipschitz there.

\begin{lemma}[Gated layers satisfy Assumption~\ref{ass:coeff-bounds} with depth-uniform constants]
\label{lem:uniform-bounds}
Let Assumptions~\ref{ass:model-class} and~\ref{ass:gate-bounds} hold, let $\ell\ge2$, let the previous layer $Z^{(\ell-1)}$ be a continuous adapted process, and consider the gated coefficient process of layer $\ell$,
\begin{equation}
    \vartheta^{(\ell)}_t
    :=
    \bigl(\alpha^{(\ell)}_t\odot A^{(\ell)}_t,\;
    b^{(\ell)}_t+o^{(\ell)}_t,\;
    C^{(\ell),1}_t,\dots,C^{(\ell),m}_t,\;
    d^{(\ell),1}_t,\dots,d^{(\ell),m}_t\bigr),
    \label{eq:app-gated-coeff}
\end{equation}
with $(\alpha^{(\ell)},o^{(\ell)})$ evaluated on $Z^{(\ell-1)}$.
\begin{enumerate}[leftmargin=1.6em,label=(\alph*),itemsep=0.1em,topsep=0.2em]
\item \emph{(Bounds.)} $\vartheta^{(\ell)}$ is progressively measurable and satisfies
\begin{equation}
    \|\alpha^{(\ell)}_t\odot A^{(\ell)}_t\|\le(1+\varepsilon)a_*,
    \qquad
    \|b^{(\ell)}_t+o^{(\ell)}_t\|\le b_*+\Lambda_o,
    \qquad
    \|C^{(\ell),j}_t\|\le c_*,
    \qquad
    \|d^{(\ell),j}_t\|\le d_*,
    \label{eq:app-uniform-bounds}
\end{equation}
Consequently the maps of Eq.~\ref{eq:app-mu-sigma} built from $\vartheta^{(\ell)}$ satisfy Eq.~\ref{eq:linbounds} with the constant
\begin{equation}
    \Lambda'_\star
    :=
    (1+\varepsilon)a_*+b_*+\Lambda_o+m\,(c_*+d_*),
    \label{eq:app-lambda-star}
\end{equation}
which depends neither on $\ell$, nor on $L$, nor on the previous layer. This part needs no moment bound on $Z^{(\ell-1)}$.
\item \emph{(Time regularity.)} If moreover, for some $q\ge2$,
\begin{equation}
    \E\|Z^{(\ell-1)}_t-Z^{(\ell-1)}_s\|^q\le K'_q\,|t-s|^{q/2},
    \qquad s,t\in[0,T],
    \label{eq:app-prev-layer}
\end{equation}
then $\vartheta^{(\ell)}$ satisfies Eq.~\ref{eq:treg} at the exponent $q$, with
\begin{equation}
    L_q^q
    =
    3^{q-1}\Bigl(\bigl((1+a_*)L_g\bigr)^q\bigl(K'_q+T^{q/2}\bigr)+\bigl((2+\varepsilon)L_\tau\bigr)^qT^{q/2}\Bigr).
    \label{eq:app-gated-holder}
\end{equation}
\end{enumerate}
In particular, if Eq.~\ref{eq:app-prev-layer} holds for every $q\ge2$, then $\vartheta^{(\ell)}$ satisfies Assumption~\ref{ass:coeff-bounds}.
\end{lemma}

\begin{proof}
\medskip
\noindent
\emph{Part (a), step 1: measurability.}
By Assumption~\ref{ass:gate-bounds} the gate map is causal and jointly measurable, so its evaluation on the continuous adapted process $Z^{(\ell-1)}$ is progressively measurable. The base coefficients are deterministic and continuous in $t$. Hence $\vartheta^{(\ell)}$ is progressively measurable.

\medskip
\noindent
\emph{Part (a), step 2: bounds.}
Rescaling the $i$-th diagonal entry of $A^{(\ell)}_t$ by $\alpha^{(\ell)}_{t,i}$ adds a diagonal matrix to $A^{(\ell)}_t$,
\[
    \alpha^{(\ell)}_t\odot A^{(\ell)}_t
    =
    A^{(\ell)}_t+\operatorname{diag}\bigl((\alpha^{(\ell)}_{t,i}-1)A^{(\ell)}_{t,ii}\bigr)_{i\le d},
    \qquad
    |A^{(\ell)}_{t,ii}|\le\|A^{(\ell)}_t\| .
\]
Assumption~\ref{ass:gate-bounds}(b) therefore gives
\begin{equation}
    \|\alpha^{(\ell)}_t\odot A^{(\ell)}_t\|
    \le
    \|A^{(\ell)}_t\|+\max_i|\alpha^{(\ell)}_{t,i}-1|\,|A^{(\ell)}_{t,ii}|
    \le
    (1+\varepsilon)\,a_* .
    \label{eq:app-alpha-A}
\end{equation}
By Assumption~\ref{ass:gate-bounds}(c),
\[
    \|b^{(\ell)}_t+o^{(\ell)}_t\|\le b_*+\Lambda_o ,
\]
and the diffusion coefficients are not gated. This proves Eq.~\ref{eq:app-uniform-bounds}. Eq.~\ref{eq:linbounds} with the constant $\Lambda'_\star$ follows exactly as Eq.~\ref{eq:app-lambda-prime} follows from Eq.~\ref{eq:app-coeff-bounds}.

\medskip
\noindent
\emph{Part (b): time regularity.}
Fix $s\le t$ and abbreviate $\alpha_t=\alpha^{(\ell)}_t$, $o_t=o^{(\ell)}_t$ and $Z=Z^{(\ell-1)}$. The argument of Eq.~\ref{eq:app-alpha-A} gives
\[
    \|\alpha_s\odot B\|\le(1+\varepsilon)\|B\|
    \qquad\text{for every matrix }B,
\]
and the base coefficients are Lipschitz in $t$ with constant $L_\tau$. Therefore
\begin{equation}
\begin{split}
    \|\vartheta^{(\ell)}_t-\vartheta^{(\ell)}_s\|
    &\le
    \|(\alpha_t-\alpha_s)\odot A^{(\ell)}_t\|
    +\|\alpha_s\odot(A^{(\ell)}_t-A^{(\ell)}_s)\|
    +\|o_t-o_s\|\\
    &\qquad
    +\|b^{(\ell)}_t-b^{(\ell)}_s\|
    +\sum_{j=1}^m\bigl(\|C^{(\ell),j}_t-C^{(\ell),j}_s\|+\|d^{(\ell),j}_t-d^{(\ell),j}_s\|\bigr)\\
    &\le
    a_*\|\alpha_t-\alpha_s\|_\infty+\|o_t-o_s\|+(2+\varepsilon)L_\tau|t-s| .
\end{split}
\label{eq:app-gated-increment}
\end{equation}
By Eq.~\ref{eq:app-gate-lip-time},
\[
    a_*\|\alpha_t-\alpha_s\|_\infty+\|o_t-o_s\|
    \le
    (1+a_*)L_g\bigl(\|Z_t-Z_s\|+|t-s|\bigr).
\]
Raising to the power $q$ and using $(x+y+z)^q\le3^{q-1}(x^q+y^q+z^q)$,
\[
    \|\vartheta^{(\ell)}_t-\vartheta^{(\ell)}_s\|^q
    \le
    3^{q-1}\Bigl(\bigl((1+a_*)L_g\bigr)^q\|Z_t-Z_s\|^q
    +\bigl((1+a_*)L_g\bigr)^q|t-s|^q
    +\bigl((2+\varepsilon)L_\tau\bigr)^q|t-s|^q\Bigr).
\]
Taking expectations and using Eq.~\ref{eq:app-prev-layer} together with $|t-s|^q\le T^{q/2}|t-s|^{q/2}$,
\[
    \E\|\vartheta^{(\ell)}_t-\vartheta^{(\ell)}_s\|^q
    \le
    L_q^q\,|t-s|^{q/2}
\]
with the constant $L_q$ of Eq.~\ref{eq:app-gated-holder}, which is Eq.~\ref{eq:treg}.
\end{proof}

\subsection{Well-posedness: proof of Theorem~\ref{thm:wellposed}}
\label{app:proof-wellposed}

We first treat one layer with a general random coefficient process; the time regularity Eq.~\ref{eq:treg} is not needed here.

\begin{lemma}[One layer with random coefficients]
\label{lem:one-layer}
Let $\vartheta$ be a progressively measurable coefficient process satisfying the bounds Eq.~\ref{eq:app-coeff-bounds}, let $q\ge2$, and let $Z_0$ be $\F_0$-measurable with $\E\|Z_0\|^q<\infty$. Then the SDE
\[
    \dd Z_t=\mu(t,Z_t)\,\dd t+\sigma(t,Z_t)\,\dd W_t,
    \qquad Z_0 \text{ given},
\]
has a unique strong solution on $[0,T]$, and
\begin{align}
    \E\sup_{t\le T}\|Z_t\|^q
    &\le
    K_q\bigl(1+\E\|Z_0\|^q\bigr),
    \label{eq:app-moment}\\
    \E\sup_{s\le r\le t}\|Z_r-Z_s\|^q
    &\le
    K'_q\bigl(1+\E\|Z_0\|^q\bigr)\,|t-s|^{q/2},
    \qquad 0\le s\le t\le T,
    \label{eq:app-increment}
\end{align}
with constants $K_q,K'_q$ depending only on $(\Lambda',T,m,q)$.
\end{lemma}

\begin{proof}
\medskip
\noindent
\emph{Step 1: existence and uniqueness.}
By Eq.~\ref{eq:linbounds} the coefficients are progressively measurable, globally Lipschitz in $z$ and of linear growth, uniformly in $(t,\omega)$. Existence and uniqueness of a strong solution for such random Lipschitz coefficients is classical \citep[Ch.~V, Thm.~7]{ProtterSDE}.

\medskip
\noindent
\emph{Step 2: moment bound.}
Fix $n\in\mathbb N$ and the localising time
\[
    \tau_n:=\inf\{t\ge0:\|Z_t\|\ge n\}\wedge T .
\]
For $t\le T$, the integral form of the equation gives
\[
    Z_{t\wedge\tau_n}
    =
    Z_0
    +\int_0^{t\wedge\tau_n}\mu(s,Z_s)\,\dd s
    +\int_0^{t}\mathbf 1_{\{s<\tau_n\}}\sigma(s,Z_s)\,\dd W_s .
\]
Using $\|x+y+z\|^q\le3^{q-1}(\|x\|^q+\|y\|^q+\|z\|^q)$,
\begin{equation}
    \sup_{r\le t}\|Z_{r\wedge\tau_n}\|^q
    \le
    3^{q-1}\Bigl[
        \|Z_0\|^q
        +\Bigl(\int_0^{t\wedge\tau_n}\|\mu(s,Z_s)\|\,\dd s\Bigr)^q
        +\sup_{r\le t}\Bigl\|\int_0^{r}\mathbf 1_{\{s<\tau_n\}}\sigma(s,Z_s)\,\dd W_s\Bigr\|^q
    \Bigr].
    \label{eq:app-wp-split}
\end{equation}
On $\{s<\tau_n\}$ we have $Z_s=Z_{s\wedge\tau_n}$. By Jensen's inequality and Eq.~\ref{eq:linbounds},
\[
    \Bigl(\int_0^{t\wedge\tau_n}\|\mu(s,Z_s)\|\,\dd s\Bigr)^q
    \le
    T^{q-1}\int_0^{t}\|\mu(s,Z_{s\wedge\tau_n})\|^q\,\dd s
    \le
    T^{q-1}\,2^{q-1}\Lambda'^{\,q}\int_0^t\bigl(1+\|Z_{s\wedge\tau_n}\|^q\bigr)\dd s .
\]
By Eq.~\ref{eq:app-bdg}, Jensen's inequality (for $q/2\ge1$) and Eq.~\ref{eq:linbounds},
\[
\begin{split}
    \E\sup_{r\le t}\Bigl\|\int_0^{r}\mathbf 1_{\{s<\tau_n\}}\sigma(s,Z_s)\,\dd W_s\Bigr\|^q
    &\le
    c_q\,T^{q/2-1}\,\E\int_0^t\mathbf 1_{\{s<\tau_n\}}\|\sigma(s,Z_s)\|_{\mathrm F}^q\,\dd s\\
    &\le
    c_q\,T^{q/2-1}\,2^{q-1}\Lambda'^{\,q}\int_0^t\bigl(1+\E\|Z_{s\wedge\tau_n}\|^q\bigr)\dd s .
\end{split}
\]
Set
\[
    \phi_n(t):=\E\sup_{r\le t}\|Z_{r\wedge\tau_n}\|^q,
\]
which is finite because $\|Z_{r\wedge\tau_n}\|\le\max(n,\|Z_0\|)$. Taking expectations in Eq.~\ref{eq:app-wp-split},
\[
    \phi_n(t)
    \le
    3^{q-1}\,\E\|Z_0\|^q
    +C_1\int_0^t\bigl(1+\phi_n(s)\bigr)\dd s,
    \qquad
    C_1:=3^{q-1}2^{q-1}\Lambda'^{\,q}\bigl(T^{q-1}+c_qT^{q/2-1}\bigr),
\]
and Gr\"onwall's lemma gives
\[
    \phi_n(T)\le\bigl(3^{q-1}\E\|Z_0\|^q+C_1T\bigr)e^{C_1T},
\]
uniformly in $n$. Since the paths are continuous, $\tau_n\uparrow T$ almost surely and
\[
    \sup_{r\le T}\|Z_{r\wedge\tau_n}\|\uparrow\sup_{r\le T}\|Z_r\| ,
\]
so monotone convergence yields Eq.~\ref{eq:app-moment} with
\[
    K_q:=\bigl(3^{q-1}+C_1T\bigr)e^{C_1T} .
\]

\medskip
\noindent
\emph{Step 3: increments.}
Fix $s\le t$. For $r\in[s,t]$,
\[
    Z_r-Z_s=\int_s^r\mu(u,Z_u)\,\dd u+\int_s^r\sigma(u,Z_u)\,\dd W_u .
\]
The same two inequalities, now without localisation because the moments are finite by Step~2, give
\[
\begin{split}
    \E\sup_{s\le r\le t}\|Z_r-Z_s\|^q
    &\le
    2^{q-1}\Bigl[(t-s)^{q-1}\,\E\int_s^t\|\mu(u,Z_u)\|^q\,\dd u\\
    &\qquad\qquad
    +c_q\,(t-s)^{q/2-1}\,\E\int_s^t\|\sigma(u,Z_u)\|_{\mathrm F}^q\,\dd u\Bigr]\\
    &\le
    2^{q-1}\,2^{q-1}\Lambda'^{\,q}\bigl(1+\E\sup_{u\le T}\|Z_u\|^q\bigr)\bigl[(t-s)^q+c_q(t-s)^{q/2}\bigr] .
\end{split}
\]
Since $(t-s)^q\le T^{q/2}(t-s)^{q/2}$, Eq.~\ref{eq:app-moment} gives
\[
    \E\sup_{s\le r\le t}\|Z_r-Z_s\|^q
    \le
    4^{q-1}\Lambda'^{\,q}\bigl(1+K_q\bigr)\bigl(T^{q/2}+c_q\bigr)\bigl(1+\E\|Z_0\|^q\bigr)\,(t-s)^{q/2},
\]
which is Eq.~\ref{eq:app-increment}.
\end{proof}

\begin{proof}[Proof of Theorem~\ref{thm:wellposed}]
\medskip
\noindent
\emph{(a) Single layer.}
The base layer satisfies Assumption~\ref{ass:coeff-bounds} by construction. Lemma~\ref{lem:one-layer} with $Z_0=\iota_d(z_0)$ gives existence, uniqueness and Eq.~\ref{eq:app-moment} with
\[
    K_q=K_q(\Lambda',T,m,q),
\]
where $\Lambda'$ is the function Eq.~\ref{eq:app-lambda-prime} of $(a_*,b_*,c_*,d_*,m)$.

\medskip
\noindent
\emph{(b) Gated stack.}
Let $\Lambda'_\star\ge\Lambda'$ be the constant Eq.~\ref{eq:app-lambda-star}, and write
\[
    K^\star_q:=K_q(\Lambda'_\star,T,m,q),
    \qquad
    K'^{\star}_q:=K'_q(\Lambda'_\star,T,m,q),
\]
for the constants of Lemma~\ref{lem:one-layer} at this value; they depend on $(a_*,b_*,c_*,d_*,m,T,q,\varepsilon,\Lambda_o)$ only. We prove by induction on $\ell$ that layer $\ell$ has a unique strong solution and that, for every $q\ge2$ with $\E\|z_0\|^q<\infty$,
\begin{equation}
\begin{gathered}
    \E\sup_{t\le T}\|Z^{(\ell)}_t\|^q\le K^\star_q\bigl(1+\E\|z_0\|^q\bigr),\\
    \E\|Z^{(\ell)}_t-Z^{(\ell)}_s\|^q\le K'^{\star}_q\bigl(1+\E\|z_0\|^q\bigr)\,|t-s|^{q/2},
    \qquad s,t\in[0,T].
\end{gathered}
    \label{eq:app-induction-wp}
\end{equation}
The base layer is (a), since $\Lambda'\le\Lambda'_\star$.

Let $\ell\ge2$ and assume the claim for layer $\ell-1$; in particular $Z^{(\ell-1)}$ is a continuous adapted process. By Lemma~\ref{lem:uniform-bounds}(a), which needs no moment bound on the previous layer, the gated coefficient process $\vartheta^{(\ell)}$ of Eq.~\ref{eq:app-gated-coeff} is progressively measurable and satisfies the bounds Eq.~\ref{eq:app-uniform-bounds}, hence Eq.~\ref{eq:linbounds} with the constant $\Lambda'_\star$. Lemma~\ref{lem:one-layer} applied to layer $\ell$, with $Z_0=\iota_d(z_0)$ and $\Lambda'_\star$ in place of $\Lambda'$, gives existence, uniqueness and Eq.~\ref{eq:app-induction-wp} for layer $\ell$. The constants do not depend on $\ell$, on $L$ or on the previous layer, which closes the induction.

Only adaptedness of the drivers to the common filtration was used, so the conclusion holds for any correlation structure of the layer drivers.
\end{proof}

\begin{remark}[The normalisation is what decouples the depth cascade]
\label{rem:wp-norm}
The single point where the architecture enters is that the offset, though of unbounded form, is a bounded forcing \emph{because its input is normalised}. If the gate read the raw state, the offset would satisfy only
\[
    \|o^{(\ell)}_t\|\lesssim\|W_{io}\|\,\|Z^{(\ell-1)}_t\|,
\]
a linear growth in the previous layer. The moment bounds $M_\ell:=\E\sup_t\|Z^{(\ell)}_t\|^q$ would then obey a recursion of the form
\[
    M_\ell\le K\bigl(1+\|W_{io}\|^qM_{\ell-1}\bigr),
\]
and the bound would grow geometrically in $L$ unless $K\|W_{io}\|^q\le1$. Uniformity in $L$ is therefore a consequence of the normalisation, not of the algebraic form of the gate.
\end{remark}

\subsection{Exactness of the parallel associative scan}
\label{app:scan}

\begin{proposition}[Scan exactness]
\label{prop:affine-scan-closure}
Let $\mathcal M\subseteq\R^{d\times d}$ be closed under multiplication with $I\in\mathcal M$ (the diagonal, block-diagonal or dense family), and let $\circ$ be the composition rule of Eq.~\ref{eq:model-affine-combine}.
\begin{enumerate}[leftmargin=1.6em,label=(\alph*),itemsep=0.1em,topsep=0.2em]
\item $\circ$ is associative; the prefix products
\begin{equation}
    (\widehat F_k,\widehat g_k):=(F_{k-1},g_{k-1})\circ\cdots\circ(F_0,g_0),
    \qquad
    (\widehat F_0,\widehat g_0):=(I,0),
    \label{eq:app-prefix}
\end{equation}
satisfy
\[
    Z_k=\widehat F_kZ_0+\widehat g_k
    \qquad\text{for the recursion}\qquad
    Z_{k+1}=F_kZ_k+g_k ;
\]
hence any associative-scan evaluation order returns exactly the sequential iterates, in $O(\log T)$ parallel depth, and $\widehat F_k\in\mathcal M$ whenever all $F_k\in\mathcal M$.
\item Fix a realisation of the previous-layer path. The gated pairs
\begin{equation}
    (\bar F_k^{(\ell)},\bar g_k^{(\ell)})
    =
    \bigl(\mathcal D(\alpha^{(\ell)}_k,F^{(\ell)}_k),\;g^{(\ell)}_k+o^{(\ell)}_k\bigr)
\end{equation}
of Eq.~\ref{eq:gated-pair} are then fixed affine data with $\bar F^{(\ell)}_k\in\mathcal M$, so (a) applies within each layer, and the depth-$L$ stack is evaluated by $L$ sequential scans in $O(L\log T)$ depth.
\end{enumerate}
\end{proposition}

\begin{proof}
\medskip
\noindent
\emph{(a)}
To an affine pair associate the augmented matrix
\begin{equation}
    \iota(F,g):=\begin{pmatrix}F&g\\0&1\end{pmatrix}\in\R^{(d+1)\times(d+1)} .
    \label{eq:app-augmented}
\end{equation}
The map $\iota$ is injective, and a direct multiplication gives
\begin{equation}
    \iota(F',g')\,\iota(F,g)
    =
    \begin{pmatrix}F'F&F'g+g'\\0&1\end{pmatrix}
    =
    \iota\bigl((F',g')\circ(F,g)\bigr).
    \label{eq:app-homomorphism}
\end{equation}
Hence, for any three pairs,
\[
\begin{split}
    \iota\Bigl(\bigl((F_c,g_c)\circ(F_b,g_b)\bigr)\circ(F_a,g_a)\Bigr)
    &=
    \iota(F_c,g_c)\,\iota(F_b,g_b)\,\iota(F_a,g_a)\\
    &=
    \iota\Bigl((F_c,g_c)\circ\bigl((F_b,g_b)\circ(F_a,g_a)\bigr)\Bigr),
\end{split}
\]
and injectivity of $\iota$ gives associativity of $\circ$.

Next, write $\xi_k:=(Z_k^\top,1)^\top$. The recursion reads
\[
    \xi_{k+1}=\iota(F_k,g_k)\,\xi_k ,
\]
so by induction and Eq.~\ref{eq:app-homomorphism}
\[
    \xi_k
    =
    \iota(F_{k-1},g_{k-1})\cdots\iota(F_0,g_0)\,\xi_0
    =
    \iota(\widehat F_k,\widehat g_k)\,\xi_0,
    \qquad\text{that is,}\qquad
    Z_k=\widehat F_kZ_0+\widehat g_k .
\]
Since every bracketing of an associative product agrees, any order of evaluation of the prefixes Eq.~\ref{eq:app-prefix} returns the same pairs, and therefore the same iterates $Z_k$. The work-efficient parallel prefix algorithm \citep{Blelloch1990PrefixSums} evaluates all of them along a balanced binary tree of depth logarithmic in the number of time steps, that is, $O(\log T)$ with the convention of Section~\ref{sec:background}. Finally
\[
    \widehat F_k=F_{k-1}\cdots F_0\in\mathcal M
\]
because $\mathcal M$ is closed under multiplication. The diagonal and block-diagonal implementations are structured representations of the same products: diagonal multiplication is elementwise and block-diagonal multiplication is independent multiplication inside each block.

\medskip
\noindent
\emph{(b)}
The gates $\alpha^{(\ell)}_k$ and $o^{(\ell)}_k$ depend only on the previous layer up to time $t_k$ and on deterministic time features, never on the current state $Z^{(\ell)}_k$; conditioning on the previous path fixes them as deterministic data. Rescaling the diagonal entries of a matrix in the diagonal, block-diagonal or dense family leaves its sparsity pattern intact, so $\bar F^{(\ell)}_k\in\mathcal M$, and the offset only shifts $g^{(\ell)}_k$. Thus (a) applies inside layer $\ell$. The layers run sequentially because the data of layer $\ell$ depend on $Z^{(\ell-1)}$, which gives $L$ scans of depth $O(\log T)$ each.
\end{proof}

\begin{remark}
The gate makes the coefficients depend on the \emph{other} layer, never on the current state, so affinity within a layer, and with it exactness of the scan, is untouched. This is the precise sense in which gated in-flow stacking preserves the parallel associative scan.
\end{remark}

\subsection{Strong discretisation error}
\label{app:proof-disc}

Throughout this subsection the coefficient process $\vartheta$ of the layer under consideration is progressively measurable and satisfies the bounds Eq.~\ref{eq:app-coeff-bounds}; the time regularity Eq.~\ref{eq:treg} is invoked where it is needed. On the uniform grid
\[
    t_k=k\dt,
    \qquad k=0,\dots,N,
    \qquad
    \dt=T/N\le1,
\]
write
\[
    \kappa(t):=\max\{t_k:t_k\le t\},
    \qquad
    \dW_k:=W_{t_{k+1}}-W_{t_k} .
\]
The Euler--Maruyama step of Eq.~\ref{eq:linear-approx-flow} is the affine map
\begin{equation}
\begin{gathered}
    \bar Z_{k+1}
    =
    \bar Z_k+\mu(t_k,\bar Z_k)\,\dt+\sigma(t_k,\bar Z_k)\,\dW_k
    =
    F_k\bar Z_k+g_k,\\
    F_k=I+M_k,
    \qquad
    M_k:=A_{t_k}\dt+\sum_{j=1}^mC^j_{t_k}\dW^j_k,
    \qquad
    g_k:=b_{t_k}\dt+\sum_{j=1}^md^j_{t_k}\dW^j_k,
\end{gathered}
    \label{eq:app-euler}
\end{equation}
started at $\bar Z_0=Z_0$. Its continuous interpolation is
\begin{equation}
    \hat Z_t
    :=
    \bar Z_k+\mu(t_k,\bar Z_k)\,(t-t_k)+\sigma(t_k,\bar Z_k)\,(W_t-W_{t_k}),
    \qquad t\in[t_k,t_{k+1}],
    \label{eq:app-interp}
\end{equation}
which satisfies $\hat Z_{t_k}=\bar Z_k$ and
\begin{equation}
    \dd\hat Z_t
    =
    \mu\bigl(\kappa(t),\hat Z_{\kappa(t)}\bigr)\dd t
    +\sigma\bigl(\kappa(t),\hat Z_{\kappa(t)}\bigr)\dd W_t .
    \label{eq:app-interp-sde}
\end{equation}
All estimates hold for any correlation structure of the layer drivers, since no conditioning on other layers is used.

\begin{lemma}[One-step conditional moments]
\label{lem:onestep}
For every $r\ge1$ there is a constant $C_r$, depending only on $(a_*,\allowbreak b_*,\allowbreak c_*,\allowbreak d_*,\allowbreak m,\allowbreak r)$, such that for $\dt\le1$ and every $k$,
\begin{equation}
\begin{gathered}
    \bigl\|\E[M_k\mid\F_{t_k}]\bigr\|\le a_*\dt,
    \qquad
    \bigl\|\E[g_k\mid\F_{t_k}]\bigr\|\le b_*\dt,\\
    \E\bigl[\|M_k\|^r\bigm|\F_{t_k}\bigr]\le C_r\dt^{r/2},
    \qquad
    \E\bigl[\|g_k\|^r\bigm|\F_{t_k}\bigr]\le C_r\dt^{r/2}.
\end{gathered}
    \label{eq:app-onestep}
\end{equation}
\end{lemma}

\begin{proof}
$\vartheta_{t_k}$ is $\F_{t_k}$-measurable and bounded by Eq.~\ref{eq:app-coeff-bounds}, while $\dW_k$ is independent of $\F_{t_k}$ with
\[
    \E[\dW_k]=0,
    \qquad
    \E|\dW^j_k|^r=m_r\,\dt^{r/2} .
\]
The first two bounds follow. For the third,
\[
    \|M_k\|\le a_*\dt+c_*\sum_{j=1}^m|\dW^j_k| ,
\]
so, by
\[
    (x_0+\dots+x_m)^r\le(m+1)^{r-1}\bigl(x_0^r+\dots+x_m^r\bigr)
    \qquad\text{and}\qquad
    \dt\le1,
\]
\[
    \E\bigl[\|M_k\|^r\bigm|\F_{t_k}\bigr]
    \le
    (m+1)^{r-1}\bigl(a_*^r\dt^r+m\,c_*^r\,m_r\,\dt^{r/2}\bigr)
    \le
    (m+1)^{r-1}\bigl(a_*^r+m\,c_*^r\,m_r\bigr)\dt^{r/2} .
\]
The bound for $g_k$ is identical with $(b_*,d_*)$ in place of $(a_*,c_*)$.
\end{proof}

\begin{lemma}[Discrete stability]
\label{lem:stab}
Let $q\ge2$ be even and $\E\|Z_0\|^q<\infty$. There is $C_q=C_q(\Lambda',T,m,q)$ such that, for $\dt\le1$,
\begin{align}
    \sup_{0\le k\le N}\E\|\bar Z_k\|^q
    &\le
    C_q\bigl(1+\E\|Z_0\|^q\bigr),
    \label{eq:app-stab}\\
    \sup_{t\le T}\E\|\hat Z_t\|^q
    \le
    C_q\bigl(1+\E\|Z_0\|^q\bigr),
    \qquad
    \E\|\hat Z_t-\hat Z_{\kappa(t)}\|^q
    &\le
    C_q\bigl(1+\E\|Z_0\|^q\bigr)\,\dt^{q/2},
    \qquad t\le T,
    \label{eq:app-interp-moments}
\end{align}
uniformly in $N$. Moreover $\E\sup_{t\le T}\|\hat Z_t\|^q<\infty$.
\end{lemma}

\begin{proof}
\medskip
\noindent
\emph{Step 1: proof of Eq.~\ref{eq:app-stab}.}
Write $q=2p$ with $p\in\mathbb N$ and
\[
    \xi_k:=M_k\bar Z_k+g_k,
    \qquad
    \eta_k:=2\langle\bar Z_k,\xi_k\rangle+\|\xi_k\|^2,
\]
so that $\bar Z_{k+1}=\bar Z_k+\xi_k$ and
\[
    \|\bar Z_{k+1}\|^2=\|\bar Z_k\|^2+\eta_k .
\]
By the binomial theorem,
\begin{equation}
    \|\bar Z_{k+1}\|^{2p}
    =
    \|\bar Z_k\|^{2p}
    +p\,\|\bar Z_k\|^{2p-2}\eta_k
    +\sum_{j=2}^p\binom pj\|\bar Z_k\|^{2(p-j)}\eta_k^{\,j}.
    \label{eq:app-binomial}
\end{equation}
Since $\bar Z_k$ is $\F_{t_k}$-measurable, Lemma~\ref{lem:onestep} gives, for every $r\ge1$,
\begin{equation}
    \E\bigl[\|\xi_k\|^r\bigm|\F_{t_k}\bigr]
    \le
    2^{r-1}\Bigl(\|\bar Z_k\|^r\,\E\bigl[\|M_k\|^r\bigm|\F_{t_k}\bigr]+\E\bigl[\|g_k\|^r\bigm|\F_{t_k}\bigr]\Bigr)
    \le
    2^{r-1}C_r\dt^{r/2}\bigl(1+\|\bar Z_k\|\bigr)^r .
    \label{eq:app-xi-moments}
\end{equation}
For the linear term of Eq.~\ref{eq:app-binomial}, Lemma~\ref{lem:onestep} and Eq.~\ref{eq:app-xi-moments} with $r=2$ give
\[
\begin{split}
    \E[\eta_k\mid\F_{t_k}]
    &=
    2\bigl\langle\bar Z_k,\E[M_k\mid\F_{t_k}]\bar Z_k+\E[g_k\mid\F_{t_k}]\bigr\rangle
    +\E\bigl[\|\xi_k\|^2\bigm|\F_{t_k}\bigr]\\
    &\le
    2\dt\bigl(a_*\|\bar Z_k\|^2+b_*\|\bar Z_k\|\bigr)
    +2C_2\dt\bigl(1+\|\bar Z_k\|\bigr)^2
    \le
    C\dt\bigl(1+\|\bar Z_k\|\bigr)^2 .
\end{split}
\]
For the higher terms,
\[
    |\eta_k|\le2\|\bar Z_k\|\,\|\xi_k\|+\|\xi_k\|^2 ,
\]
so for $j\ge2$, by Eq.~\ref{eq:app-xi-moments} with $r=j$ and $r=2j$,
\[
\begin{split}
    \E\bigl[|\eta_k|^j\bigm|\F_{t_k}\bigr]
    &\le
    2^{j-1}\Bigl(2^j\|\bar Z_k\|^j\,\E\bigl[\|\xi_k\|^j\bigm|\F_{t_k}\bigr]
    +\E\bigl[\|\xi_k\|^{2j}\bigm|\F_{t_k}\bigr]\Bigr)\\
    &\le
    C\bigl(\dt^{j/2}+\dt^{j}\bigr)\bigl(1+\|\bar Z_k\|\bigr)^{2j}
    \le
    C\dt\bigl(1+\|\bar Z_k\|\bigr)^{2j},
\end{split}
\]
because $\dt\le1$ and $j\ge2$. Taking conditional expectations in Eq.~\ref{eq:app-binomial} and using
\[
    \|\bar Z_k\|^{2(p-j)}\bigl(1+\|\bar Z_k\|\bigr)^{2j}
    \le
    \bigl(1+\|\bar Z_k\|\bigr)^{2p}
    \le
    2^{2p-1}\bigl(1+\|\bar Z_k\|^{2p}\bigr),
\]
we obtain
\[
    \E\bigl[\|\bar Z_{k+1}\|^{2p}\bigm|\F_{t_k}\bigr]
    \le
    \|\bar Z_k\|^{2p}+C\dt\bigl(1+\|\bar Z_k\|^{2p}\bigr).
\]
With $u_k:=\E\|\bar Z_k\|^{2p}$ this reads
\[
    u_{k+1}\le(1+C\dt)\,u_k+C\dt ,
\]
and by induction
\[
    u_k
    \le
    (1+C\dt)^k\,u_0+C\dt\sum_{i=0}^{k-1}(1+C\dt)^i
    \le
    e^{Ck\dt}\bigl(u_0+Ck\dt\bigr)
    \le
    e^{CT}\bigl(u_0+CT\bigr),
    \qquad k\le N,
\]
which is Eq.~\ref{eq:app-stab}.

\medskip
\noindent
\emph{Step 2: proof of Eq.~\ref{eq:app-interp-moments}.}
For $t\in[t_k,t_{k+1}]$, Eq.~\ref{eq:app-interp} gives
\[
    \hat Z_t-\bar Z_k=\mu(t_k,\bar Z_k)\,(t-t_k)+\sigma(t_k,\bar Z_k)\,(W_t-W_{t_k}) .
\]
Conditionally on $\F_{t_k}$ the second term is a centred Gaussian vector, and
\[
    \|\sigma(t_k,\bar Z_k)(W_t-W_{t_k})\|\le\|\sigma(t_k,\bar Z_k)\|_{\mathrm F}\,\|W_t-W_{t_k}\|,
    \qquad
    \E\|W_t-W_{t_k}\|^q\le m^{q/2}m_q\,\dt^{q/2} .
\]
Hence, by Eq.~\ref{eq:linbounds},
\[
    \E\bigl[\|\hat Z_t-\bar Z_k\|^q\bigm|\F_{t_k}\bigr]
    \le
    2^{q-1}\Lambda'^{\,q}\bigl(1+\|\bar Z_k\|\bigr)^q\bigl(\dt^q+m^{q/2}m_q\dt^{q/2}\bigr)
    \le
    C\bigl(1+\|\bar Z_k\|\bigr)^q\dt^{q/2} .
\]
Taking expectations and using Eq.~\ref{eq:app-stab} gives the second bound in Eq.~\ref{eq:app-interp-moments}; the first follows from
\[
    \|\hat Z_t\|^q\le2^{q-1}\bigl(\|\bar Z_k\|^q+\|\hat Z_t-\bar Z_k\|^q\bigr).
\]
Finally,
\[
    \sup_{t\le T}\|\hat Z_t\|
    \le
    \max_{k<N}\Bigl(\|\bar Z_k\|+\Lambda'\bigl(1+\|\bar Z_k\|\bigr)\bigl(\dt+\sup_{t\in[t_k,t_{k+1}]}\|W_t-W_{t_k}\|\bigr)\Bigr)
\]
is a maximum over finitely many random variables with finite $q$-th moments. Hence
\[
    \E\sup_{t\le T}\|\hat Z_t\|^q<\infty .
\]
\end{proof}

\begin{lemma}[Euler--Maruyama error for one layer]
\label{lem:euler}
Let $q\ge2$ be even, $\dt\le1$ and $\E\|Z_0\|^{2q}<\infty$. Let $Z$ solve
\[
    \dd Z_t=\mu(t,Z_t)\,\dd t+\sigma(t,Z_t)\,\dd W_t
\]
and let $\hat Z$ be the interpolated scheme Eq.~\ref{eq:app-interp}, with the same $Z_0$ and the same driver. Then
\begin{equation}
    \E\sup_{t\le T}\|Z_t-\hat Z_t\|^q
    \le
    C_q\Bigl(\dt^{q/2}+\sup_{s\le T}\bigl(\E\|\vartheta_s-\vartheta_{\kappa(s)}\|^{2q}\bigr)^{1/2}\Bigr),
    \label{eq:app-euler-general}
\end{equation}
where $C_q$ depends only on $(\Lambda',T,m,q)$ and on $\E\|Z_0\|^{2q}$. If in addition $\vartheta$ satisfies the time regularity Eq.~\ref{eq:treg}, then
\begin{equation}
\begin{gathered}
    \E\sup_{t\le T}\|Z_t-\hat Z_t\|^q
    \le
    C_q\bigl(1+L_{2q}^{q}\bigr)\dt^{q/2},\\
    \text{in particular}\qquad
    \max_{0\le k\le N}\E\|Z_{t_k}-\bar Z_k\|^q
    \le
    C_q\bigl(1+L_{2q}^{q}\bigr)\dt^{q/2}:
\end{gathered}
    \label{eq:app-euler-rate}
\end{equation}
the scheme has strong order $\tfrac12$ in every $L^q$ \citep[Thm.~10.2.2]{KloedenPlaten}.
\end{lemma}

\begin{proof}
Set $e_t:=Z_t-\hat Z_t$, so that $e_0=0$ and, by Eq.~\ref{eq:app-interp-sde},
\begin{equation}
    e_t
    =
    \int_0^t\bigl[\mu(s,Z_s)-\mu(\kappa(s),\hat Z_{\kappa(s)})\bigr]\dd s
    +\int_0^t\bigl[\sigma(s,Z_s)-\sigma(\kappa(s),\hat Z_{\kappa(s)})\bigr]\dd W_s .
    \label{eq:app-error-eq}
\end{equation}

\medskip
\noindent
\emph{Step 1: decomposition of the coefficient differences.}
For the drift,
\[
    \mu(s,Z_s)-\mu(\kappa(s),\hat Z_{\kappa(s)})
    =
    A_se_s
    +\bigl(A_s-A_{\kappa(s)}\bigr)\hat Z_s
    +\bigl(b_s-b_{\kappa(s)}\bigr)
    +A_{\kappa(s)}\bigl(\hat Z_s-\hat Z_{\kappa(s)}\bigr),
\]
and, for each $j$,
\[
    \sigma^j(s,Z_s)-\sigma^j(\kappa(s),\hat Z_{\kappa(s)})
    =
    C^j_se_s
    +\bigl(C^j_s-C^j_{\kappa(s)}\bigr)\hat Z_s
    +\bigl(d^j_s-d^j_{\kappa(s)}\bigr)
    +C^j_{\kappa(s)}\bigl(\hat Z_s-\hat Z_{\kappa(s)}\bigr).
\]
Define
\begin{equation}
    D_s:=\|\vartheta_s-\vartheta_{\kappa(s)}\|\bigl(1+\|\hat Z_s\|\bigr),
    \qquad
    I_s:=\|\hat Z_s-\hat Z_{\kappa(s)}\| .
    \label{eq:app-D-I}
\end{equation}
By Eq.~\ref{eq:app-coeff-bounds} and the definition Eq.~\ref{eq:app-coeff-norm} of $\|\vartheta\|$,
\begin{equation}
\begin{gathered}
    \|\mu(s,Z_s)-\mu(\kappa(s),\hat Z_{\kappa(s)})\|
    \le
    a_*\|e_s\|+D_s+a_*I_s,\\
    \|\sigma(s,Z_s)-\sigma(\kappa(s),\hat Z_{\kappa(s)})\|_{\mathrm F}
    \le
    m\,c_*\|e_s\|+D_s+m\,c_*I_s .
\end{gathered}
    \label{eq:app-coeff-diff-bounds}
\end{equation}

\medskip
\noindent
\emph{Step 2: a Gr\"onwall inequality.}
Let
\[
    \phi(t):=\E\sup_{r\le t}\|e_r\|^q ,
\]
which is finite by Lemma~\ref{lem:one-layer} and Lemma~\ref{lem:stab}. From Eq.~\ref{eq:app-error-eq}, $\|x+y\|^q\le2^{q-1}(\|x\|^q+\|y\|^q)$, Jensen's inequality for the time integral, Eq.~\ref{eq:app-bdg} with Jensen's inequality for the stochastic integral, and Eq.~\ref{eq:app-coeff-diff-bounds},
\begin{equation}
\begin{split}
    \phi(t)
    &\le
    2^{q-1}\Bigl[T^{q-1}\,\E\int_0^t\|\mu(s,Z_s)-\mu(\kappa(s),\hat Z_{\kappa(s)})\|^q\,\dd s\\
    &\qquad\qquad
    +c_qT^{q/2-1}\,\E\int_0^t\|\sigma(s,Z_s)-\sigma(\kappa(s),\hat Z_{\kappa(s)})\|_{\mathrm F}^q\,\dd s\Bigr]\\
    &\le
    C_2\int_0^t\bigl(\E\|e_s\|^q+\E D_s^q+\E I_s^q\bigr)\dd s
    \ \le\
    C_2\int_0^t\phi(s)\,\dd s+C_2\int_0^T\bigl(\E D_s^q+\E I_s^q\bigr)\dd s,
\end{split}
\label{eq:app-gronwall-pre}
\end{equation}
with $C_2=C_2(\Lambda',T,m,q)$. Gr\"onwall's lemma gives
\begin{equation}
    \phi(T)
    \le
    C_2e^{C_2T}\int_0^T\bigl(\E D_s^q+\E I_s^q\bigr)\dd s .
    \label{eq:app-gronwall}
\end{equation}

\medskip
\noindent
\emph{Step 3: the two forcing terms.}
By the Cauchy--Schwarz inequality and Eq.~\ref{eq:app-interp-moments} at the exponent $2q$,
\[
\begin{split}
    \E D_s^q
    &\le
    \bigl(\E\|\vartheta_s-\vartheta_{\kappa(s)}\|^{2q}\bigr)^{1/2}
    \bigl(\E(1+\|\hat Z_s\|)^{2q}\bigr)^{1/2}\\
    &\le
    C\bigl(1+\E\|Z_0\|^{2q}\bigr)^{1/2}
    \sup_{s\le T}\bigl(\E\|\vartheta_s-\vartheta_{\kappa(s)}\|^{2q}\bigr)^{1/2},
\end{split}
\]
and by Eq.~\ref{eq:app-interp-moments} at the exponent $q$,
\[
    \E I_s^q
    \le
    C\bigl(1+\E\|Z_0\|^q\bigr)\dt^{q/2} .
\]
Inserting both bounds into Eq.~\ref{eq:app-gronwall} proves Eq.~\ref{eq:app-euler-general}.

\medskip
\noindent
\emph{Step 4: the rate under time regularity.}
Under Eq.~\ref{eq:treg},
\[
    \E\|\vartheta_s-\vartheta_{\kappa(s)}\|^{2q}
    \le
    L_{2q}^{2q}\,|s-\kappa(s)|^{q}
    \le
    L_{2q}^{2q}\,\dt^{q},
\]
so the supremum in Eq.~\ref{eq:app-euler-general} is at most $L_{2q}^q\dt^{q/2}$, which gives Eq.~\ref{eq:app-euler-rate}. The grid bound follows because $\hat Z_{t_k}=\bar Z_k$.
\end{proof}

\begin{theorem}[Stacked discretisation error]
\label{thm:disc-error}
Let Assumptions~\ref{ass:model-class}--\ref{ass:gate-bounds} hold, assume $\E\|z_0\|^r<\infty$ for every $r\ge2$, and apply the Euler--Maruyama scheme of Eq.~\ref{eq:gated-recurrence} layerwise, the gates of the discrete layer $\ell$ being evaluated on the discrete previous layer at the grid points. Let $\hat Z^{(\ell)}$ denote the interpolation Eq.~\ref{eq:app-interp} of the discrete layer $\ell$, so that $\hat Z^{(\ell)}_{t_k}=\bar Z^{(\ell)}_k$. Then for each fixed depth $L$ and every even $q\ge2$ there is a constant $C_{L,q}$, depending on the coefficient and gate bounds, on $T$, $m$, $q$, $L$ and on the moments of $z_0$ but not on $\dt$, such that for $\dt\le1$
\begin{equation}
\begin{gathered}
    \max_{1\le\ell\le L}\ \E\sup_{t\le T}\bigl\|Z^{(\ell)}_t-\hat Z^{(\ell)}_t\bigr\|^{q}
    \le
    C_{L,q}\,\dt^{q/2},\\
    \text{in particular}\qquad
    \max_{1\le\ell\le L}\ \max_{0\le k\le N}\
    \E\bigl\|Z^{(\ell)}_{t_k}-\bar Z^{(\ell)}_k\bigr\|^{q}
    \le
    C_{L,q}\,\dt^{q/2}.
\end{gathered}
    \label{eq:app-stack-bound}
\end{equation}
The scheme therefore has strong order $\tfrac12$ for every fixed depth.
\end{theorem}

\begin{proof}
Fix $L$ and an even $q\ge2$, and set
\begin{equation}
    q_\ell:=q\,2^{\,L-\ell},
    \qquad
    \hat e_\ell:=\E\sup_{t\le T}\bigl\|Z^{(\ell)}_t-\hat Z^{(\ell)}_t\bigr\|^{q_\ell},
    \qquad \ell=1,\dots,L,
    \label{eq:app-q-ell}
\end{equation}
so that $q_L=q$ and $q_{\ell-1}=2q_\ell$: higher moments are carried at the shallow layers because each inter-layer step consumes one Cauchy--Schwarz inequality. We prove by induction on $\ell$ that
\begin{equation}
    \hat e_\ell\le C\,\dt^{q_\ell/2},
    \qquad \ell=1,\dots,L,
    \label{eq:app-induction}
\end{equation}
with constants that do not depend on $\dt$. Eq.~\ref{eq:app-stack-bound} then follows: at $\ell=L$ directly, and at $\ell<L$ by Jensen's inequality,
\[
    \E\sup_{t\le T}\bigl\|Z^{(\ell)}_t-\hat Z^{(\ell)}_t\bigr\|^{q}
    \le
    \bigl(\hat e_\ell\bigr)^{q/q_\ell}
    \le
    C^{q/q_\ell}\,\dt^{q/2} ;
\]
the grid bound follows from $\hat Z^{(\ell)}_{t_k}=\bar Z^{(\ell)}_k$.

\medskip
\noindent
\emph{Base case.}
Layer $1$ is ungated: its coefficients are deterministic and Lipschitz in $t$, so Assumption~\ref{ass:coeff-bounds} holds by construction, and Lemma~\ref{lem:euler} at the exponent $q_1$ (which requires $\E\|z_0\|^{2q_1}<\infty$) gives
\[
    \hat e_1\le C\dt^{q_1/2} .
\]

\medskip
\noindent
\emph{Induction step: the frozen coefficients.}
Let $\ell\ge2$ and assume Eq.~\ref{eq:app-induction} for $\ell-1$. The step compares three processes: the exact layer $Z^{(\ell)}$, an intermediate SDE $\check Z$ whose gates are frozen on the grid and evaluated on the discrete previous layer, and the discrete layer $\hat Z^{(\ell)}$. The distance from $\check Z$ to $\hat Z^{(\ell)}$ is a one-layer Euler--Maruyama error (Lemma~\ref{lem:euler}); the distance from $Z^{(\ell)}$ to $\check Z$ is a coefficient perturbation controlled by the induction hypothesis. Let $\vartheta$ denote the exact-gate coefficient process of layer $\ell$, Eq.~\ref{eq:app-gated-coeff}, with the gates evaluated on the exact previous layer $Z^{(\ell-1)}$, and let $\bar\vartheta$ denote the discrete-gate coefficients, with the gates evaluated on the discrete previous layer and frozen on each grid interval:
\begin{equation}
\begin{gathered}
    \bar\vartheta_{t_k}
    :=
    \bigl(\alpha^{(\ell)}_{t_k}(\hat Z^{(\ell-1)})\odot A^{(\ell)}_{t_k},\;
    b^{(\ell)}_{t_k}+o^{(\ell)}_{t_k}(\hat Z^{(\ell-1)}),\;
    C^{(\ell),j}_{t_k},\;d^{(\ell),j}_{t_k}\bigr),\\
    \bar\vartheta_t:=\bar\vartheta_{\kappa(t)} .
\end{gathered}
    \label{eq:app-frozen-coeff}
\end{equation}
Because the gates of Eqs.~\ref{eq:alpha-gate} and~\ref{eq:offset-gate} read only the current value of their input path, $\alpha^{(\ell)}_{t_k}(\hat Z^{(\ell-1)})$ and $o^{(\ell)}_{t_k}(\hat Z^{(\ell-1)})$ are the gates evaluated on
\[
    \hat Z^{(\ell-1)}_{t_k}=\bar Z^{(\ell-1)}_k ,
\]
that is, exactly the gates used by the discrete layer $\ell$. The process $\bar\vartheta$ is progressively measurable, and it satisfies the bounds Eq.~\ref{eq:app-uniform-bounds} because Assumption~\ref{ass:gate-bounds}(b)--(c) hold surely for every input path. Let $\check Z$ be the solution of the intermediate SDE with the frozen coefficients and the same driver and initial state,
\begin{equation}
    \dd\check Z_t
    =
    \bar\mu(t,\check Z_t)\,\dd t+\bar\sigma(t,\check Z_t)\,\dd W^{(\ell)}_t,
    \qquad
    \check Z_0=\iota_d(z_0),
    \label{eq:app-intermediate}
\end{equation}
where $\bar\mu,\bar\sigma$ are the maps Eq.~\ref{eq:app-mu-sigma} built from $\bar\vartheta$; it exists, is unique and has finite moments of all orders by Lemma~\ref{lem:one-layer}. The Euler--Maruyama scheme Eq.~\ref{eq:app-euler} for $\check Z$ uses the coefficients $\bar\vartheta_{t_k}$, which are exactly the coefficients of the discrete layer $\ell$; hence its iterates are $\bar Z^{(\ell)}_k$ and its interpolation is $\hat Z^{(\ell)}$. We split
\begin{equation}
    Z^{(\ell)}-\hat Z^{(\ell)}
    =
    \bigl(Z^{(\ell)}-\check Z\bigr)+\bigl(\check Z-\hat Z^{(\ell)}\bigr).
    \label{eq:app-split}
\end{equation}

\medskip
\noindent
\emph{Induction step (i): $\check Z$ versus $\hat Z^{(\ell)}$.}
Lemma~\ref{lem:euler} applied to $\check Z$ at the exponent $q_\ell$, whose hypotheses hold with $\Lambda'_\star$ in place of $\Lambda'$, gives Eq.~\ref{eq:app-euler-general}. Since $\bar\vartheta_s=\bar\vartheta_{\kappa(s)}$ for every $s$, the supremum in Eq.~\ref{eq:app-euler-general} vanishes, and
\begin{equation}
    \E\sup_{t\le T}\bigl\|\check Z_t-\hat Z^{(\ell)}_t\bigr\|^{q_\ell}
    \le
    C\dt^{q_\ell/2} .
    \label{eq:app-step-i}
\end{equation}

\medskip
\noindent
\emph{Induction step (ii): $Z^{(\ell)}$ versus $\check Z$.}
Both are continuous SDEs with the same driver and initial state. Set $e_t:=Z^{(\ell)}_t-\check Z_t$; then $e_0=0$ and
\[
    e_t
    =
    \int_0^t\bigl[\mu(s,Z^{(\ell)}_s)-\bar\mu(s,\check Z_s)\bigr]\dd s
    +\int_0^t\bigl[\sigma(s,Z^{(\ell)}_s)-\bar\sigma(s,\check Z_s)\bigr]\dd W^{(\ell)}_s ,
\]
where $\mu,\sigma$ are built from $\vartheta$. Writing $\vartheta_s=(A_s,b_s,\dots)$ and $\bar\vartheta_s=(\bar A_s,\bar b_s,\dots)$,
\[
    \mu(s,Z^{(\ell)}_s)-\bar\mu(s,\check Z_s)
    =
    A_se_s+(A_s-\bar A_s)\check Z_s+(b_s-\bar b_s),
\]
so that
\[
\begin{gathered}
    \|\mu(s,Z^{(\ell)}_s)-\bar\mu(s,\check Z_s)\|
    \le
    (1+\varepsilon)a_*\|e_s\|+\|\vartheta_s-\bar\vartheta_s\|\bigl(1+\|\check Z_s\|\bigr),\\
    \|\sigma(s,Z^{(\ell)}_s)-\bar\sigma(s,\check Z_s)\|_{\mathrm F}
    \le
    m\,c_*\|e_s\|+\|\vartheta_s-\bar\vartheta_s\|\bigl(1+\|\check Z_s\|\bigr).
\end{gathered}
\]
The argument of Step~2 of the proof of Lemma~\ref{lem:euler}, with the forcing $\|\vartheta_s-\bar\vartheta_s\|(1+\|\check Z_s\|)$ in place of $D_s+\Lambda'I_s$, gives
\begin{equation}
    \E\sup_{t\le T}\|e_t\|^{q_\ell}
    \le
    C\int_0^T\E\Bigl[\|\vartheta_s-\bar\vartheta_s\|^{q_\ell}\bigl(1+\|\check Z_s\|\bigr)^{q_\ell}\Bigr]\dd s
    \le
    C\sup_{s\le T}\bigl(\E\|\vartheta_s-\bar\vartheta_s\|^{2q_\ell}\bigr)^{1/2},
    \label{eq:app-step-ii-pre}
\end{equation}
where the last step is the Cauchy--Schwarz inequality together with the moment bound Eq.~\ref{eq:app-moment} for $\check Z$ at the exponent $2q_\ell$. The coefficient difference splits into a time-regularity part and a gate-perturbation part,
\begin{equation}
    \|\vartheta_s-\bar\vartheta_s\|
    \le
    \|\vartheta_s-\vartheta_{\kappa(s)}\|
    +\|\vartheta_{\kappa(s)}-\bar\vartheta_{\kappa(s)}\| .
    \label{eq:app-two-parts}
\end{equation}
For the first part, the exact previous layer satisfies Eq.~\ref{eq:app-prev-layer} at every exponent by Theorem~\ref{thm:wellposed}, since $z_0$ has moments of all orders, so Lemma~\ref{lem:uniform-bounds}(b) gives Eq.~\ref{eq:treg} for $\vartheta$ at the exponent $2q_\ell$, and
\begin{equation}
    \bigl(\E\|\vartheta_s-\vartheta_{\kappa(s)}\|^{2q_\ell}\bigr)^{1/2}
    \le
    L_{2q_\ell}^{q_\ell}\,|s-\kappa(s)|^{q_\ell/2}
    \le
    L_{2q_\ell}^{q_\ell}\,\dt^{q_\ell/2} .
    \label{eq:app-part-one}
\end{equation}
For the second part, at the grid time $t_k=\kappa(s)$ the two coefficient tuples differ only through the gates, evaluated on the paths $Z^{(\ell-1)}$ and $\hat Z^{(\ell-1)}$ up to time $t_k$. By Eq.~\ref{eq:app-gate-lip-path} and the argument of Eq.~\ref{eq:app-alpha-A},
\[
\begin{split}
    \|\vartheta_{t_k}-\bar\vartheta_{t_k}\|
    &\le
    a_*\bigl\|\alpha^{(\ell)}_{t_k}(Z^{(\ell-1)})-\alpha^{(\ell)}_{t_k}(\hat Z^{(\ell-1)})\bigr\|_\infty
    +\bigl\|o^{(\ell)}_{t_k}(Z^{(\ell-1)})-o^{(\ell)}_{t_k}(\hat Z^{(\ell-1)})\bigr\|\\
    &\le
    (1+a_*)L_g\sup_{r\le T}\bigl\|Z^{(\ell-1)}_r-\hat Z^{(\ell-1)}_r\bigr\| ,
\end{split}
\]
so that, by the induction hypothesis and $2q_\ell=q_{\ell-1}$,
\begin{equation}
\begin{split}
    \bigl(\E\|\vartheta_{\kappa(s)}-\bar\vartheta_{\kappa(s)}\|^{2q_\ell}\bigr)^{1/2}
    &\le
    \bigl((1+a_*)L_g\bigr)^{q_\ell}\,\bigl(\hat e_{\ell-1}\bigr)^{1/2}\\
    &\le
    \bigl((1+a_*)L_g\bigr)^{q_\ell}\,C^{1/2}\,\dt^{q_{\ell-1}/4}
    =
    C\,\dt^{q_\ell/2} .
\end{split}
    \label{eq:app-part-two}
\end{equation}
Inserting Eqs.~\ref{eq:app-part-one} and~\ref{eq:app-part-two} into Eq.~\ref{eq:app-step-ii-pre} through Eq.~\ref{eq:app-two-parts} and $(x+y)^{2q_\ell}\le2^{2q_\ell-1}(x^{2q_\ell}+y^{2q_\ell})$,
\begin{equation}
    \E\sup_{t\le T}\bigl\|Z^{(\ell)}_t-\check Z_t\bigr\|^{q_\ell}
    \le
    C\dt^{q_\ell/2} .
    \label{eq:app-step-ii}
\end{equation}

\medskip
\noindent
\emph{Conclusion.}
By Eq.~\ref{eq:app-split} and $\|x+y\|^{q_\ell}\le2^{q_\ell-1}(\|x\|^{q_\ell}+\|y\|^{q_\ell})$,
\[
    \hat e_\ell
    \le
    2^{q_\ell-1}\Bigl(\E\sup_{t\le T}\bigl\|Z^{(\ell)}_t-\check Z_t\bigr\|^{q_\ell}
    +\E\sup_{t\le T}\bigl\|\check Z_t-\hat Z^{(\ell)}_t\bigr\|^{q_\ell}\Bigr)
    \le
    C\dt^{q_\ell/2}
\]
by Eqs.~\ref{eq:app-step-i} and~\ref{eq:app-step-ii}, which is Eq.~\ref{eq:app-induction} for $\ell$ and closes the induction.
\end{proof}

\begin{remark}[Dependence on the depth]
\label{rem:disc-scope}
The constant $C_{L,q}$ depends on $L$ in two ways: the induction carries the moment order $q\,2^{L-\ell}$ at layer $\ell$, so the constants of Lemmas~\ref{lem:stab} and~\ref{lem:euler} enter at exponents that grow with the depth, and each inter-layer step multiplies the propagated error by a factor involving the gate Lipschitz constant $L_g$ through Eq.~\ref{eq:app-part-two}. The theorem is therefore a statement at fixed depth: the rate $\tfrac12$ holds for every $L$, but no uniformity of the constant in $L$ is claimed. For the depths used in practice ($L=2,3$) the exponents involved are $4q$ at most.
\end{remark}

\begin{remark}[Strong order $\tfrac12$, not $1$]
Order $1$ would require the iterated It\^o integrals
\[
    \int_{t_k}^{t_{k+1}}\!\!\int_{t_k}^{s}\dd W^i_u\,\dd W^j_s,
    \qquad i\ne j,
\]
and their L\'evy areas, which the scheme does not simulate. We therefore claim only order $\tfrac12$, without assuming that the $C^j_t$ commute; if they do, the exponential transition of Appendix~\ref{app:discretisation} attains order $1$ and, in the diagonal case, is an exact one-step solution operator.
\end{remark}

\subsection{Proof of Theorem \ref{thm:girsanov-validity}}
\label{app:girsanov-validity}

\begin{proof}
Throughout, $\epsilon$ is the $m$-dimensional Brownian innovation of Eq.~\ref{eq:correlated-driver}, independent of the prefix driver $W^{\mathrm{prefix}}$ and of $\F_0$, and $u_{\theta,t}$ is progressively measurable with respect to the final-layer filtration and satisfies Novikov's condition Eq.~\ref{eq:novikov_cond}. The stochastic exponential of Eq.~\ref{eq:girsanov-density} is taken with respect to $\epsilon$,
\begin{equation}
    \mathcal E_t(u_\theta)
    :=
    \exp\Bigl(\int_0^tu_{\theta,s}^\top\dd\epsilon_s-\tfrac12\int_0^t\|u_{\theta,s}\|^2\dd s\Bigr),
    \qquad t\in[0,T].
    \label{eq:app-stoch-exp}
\end{equation}

\medskip
\noindent
\emph{Step 0: the initial-state shift.}
Under $\Prob$, $z_0\sim\mathcal N(0,I_d)$ is independent of $(W^{\mathrm{prefix}},\epsilon)$. The factor
\[
    \exp\bigl(\mu_\theta^\top z_0-\tfrac12\|\mu_\theta\|^2\bigr)
    =
    \frac{\dd\mathcal N(\mu_\theta,I_d)}{\dd\mathcal N(0,I_d)}(z_0)
\]
is $\F_0$-measurable and has $\Prob$-expectation one. Being independent of $\mathcal E_T(u_\theta)$ under $\Prob$, its product with $\mathcal E_T(u_\theta)$ is again a probability density, under $\QQ$ the initial state has law $\mathcal N(\mu_\theta,I_d)$, and Steps~1--3 below, which argue conditionally on $\F_0$, apply verbatim. In the implementation
\[
    z_0=\mu_\theta+\xi,
    \qquad
    \xi\sim\mathcal N(0,I_d),
\]
and the term $\mu_\theta^\top z_0-\tfrac12\|\mu_\theta\|^2$ is added to the discrete log-ratio Eq.~\ref{eq:app-rn-discrete}; its expectation under $\QQ$,
\[
    \E^{\QQ}\bigl[\mu_\theta^\top z_0-\tfrac12\|\mu_\theta\|^2\bigr]
    =
    \tfrac12\|\mu_\theta\|^2 ,
\]
is the initial-state share of $\KL(\QQ\|\Prob)$. When $\mu_\theta=0$ nothing changes.

\medskip
\noindent
\emph{Step 1: $\QQ$ is a probability measure and $\epsilon^\QQ$ is a $\QQ$-Brownian motion.}
Under Novikov's condition, $(\mathcal E_t(u_\theta))_{t\le T}$ is a strictly positive martingale \citep[Ch.~III]{ProtterSDE}, so
\[
    \E^{\Prob}\bigl[\mathcal E_T(u_\theta)\bigr]=1 .
\]
Hence
\[
    \QQ(A):=\E^{\Prob}\bigl[\mathcal E_T(u_\theta)\mathbf 1_A\bigr],
    \qquad A\in\F_T,
\]
is a probability measure equivalent to $\Prob$, with $\dd\QQ/\dd\Prob=\mathcal E_T(u_\theta)$. By Girsanov's theorem applied to the $2m$-dimensional Brownian motion $(W^{\mathrm{prefix}},\epsilon)$, whose density process involves only the components of $\epsilon$, the process
\begin{equation}
    \Bigl(W^{\mathrm{prefix}}_t,\ \epsilon^\QQ_t\Bigr)_{t\le T},
    \qquad
    \epsilon^\QQ_t:=\epsilon_t-\int_0^tu_{\theta,s}\,\dd s,
    \label{eq:app-girsanov-bm}
\end{equation}
is a $2m$-dimensional Brownian motion under $\QQ$: the innovation is shifted by the control, while the prefix driver keeps its law and remains independent of the shifted innovation.

\medskip
\noindent
\emph{Step 2: the change-of-measure identity.}
For every $\F_T$-measurable functional $G$ with $\E^{\Prob}|G|<\infty$, the definition of $\QQ$ gives
\begin{equation}
    \E^{\Prob}[G]
    =
    \E^{\QQ}\Bigl[G\,\frac{\dd\Prob}{\dd\QQ}\Bigr],
    \qquad
    \frac{\dd\Prob}{\dd\QQ}
    =
    \mathcal E_T(u_\theta)^{-1} .
    \label{eq:app-rn-identity}
\end{equation}
Substituting $\dd\epsilon_s=\dd\epsilon^\QQ_s+u_{\theta,s}\dd s$ in Eq.~\ref{eq:app-stoch-exp},
\[
    \frac{\dd\Prob}{\dd\QQ}
    =
    \exp\Bigl(-\int_0^Tu_{\theta,s}^\top\dd\epsilon^\QQ_s-\tfrac12\int_0^T\|u_{\theta,s}\|^2\dd s\Bigr).
\]
In particular
\[
    \E^{\QQ}\Bigl[\frac{\dd\Prob}{\dd\QQ}\Bigr]=\Prob(\Omega)=1 :
\]
the importance weights are exactly normalised.

\medskip
\noindent
\emph{Step 3: the discrete implementation is exact.}
In the implementation the innovation is simulated on the grid. Under $\Prob$ the increments
\[
    \Delta\epsilon^{\Prob}_k,
    \qquad k=0,\dots,N-1,
\]
are independent $\mathcal N(0,\dt I_m)$ vectors, and the control $u_{\theta,k}$ is $\F_{t_k}$-measurable, being a function of $\{\Delta\epsilon^{\Prob}_j\}_{j<k}$ and of prefix features up to $t_k$. Simulation under $\QQ$ draws $\Delta\epsilon^\QQ_k\sim\mathcal N(0,\dt I_m)$ and sets
\[
    \Delta\epsilon^{\Prob}_k=\Delta\epsilon^\QQ_k+u_{\theta,k}\dt ,
\]
which is Eq.~\ref{eq:eps-tilt}. Conditionally on $\F_{t_k}$, the $\Prob$-law of $\Delta\epsilon^{\Prob}_k$ has the density
\[
    \varphi_{\dt}(x)\propto\exp\Bigl(-\frac{\|x\|^2}{2\dt}\Bigr),
\]
and its $\QQ$-law has the density $\varphi_{\dt}(x-u_{\theta,k}\dt)$. Their ratio is
\[
    \frac{\varphi_{\dt}(x-u_{\theta,k}\dt)}{\varphi_{\dt}(x)}
    =
    \exp\Bigl(u_{\theta,k}^\top x-\tfrac12\|u_{\theta,k}\|^2\dt\Bigr).
\]
Multiplying the conditional density ratios over $k$ gives the likelihood ratio of the two laws of the whole increment sequence,
\begin{equation}
\begin{gathered}
    \log\frac{\dd\QQ}{\dd\Prob}
    =
    \sum_{k=0}^{N-1}\Bigl(u_{\theta,k}^\top\Delta\epsilon^{\Prob}_k-\tfrac12\|u_{\theta,k}\|^2\dt\Bigr),\\
    \log\frac{\dd\Prob}{\dd\QQ}
    =
    -\sum_{k=0}^{N-1}u_{\theta,k}^\top\Delta\epsilon^{\QQ}_k-\tfrac12\sum_{k=0}^{N-1}\|u_{\theta,k}\|^2\dt ,
\end{gathered}
    \label{eq:app-rn-discrete}
\end{equation}
which is Eq.~\ref{eq:rn-discrete}: on the grid the weight is an exact density ratio, not an approximation of Eq.~\ref{eq:app-rn-identity}. It is also exactly normalised, because for a fixed $\F_{t_k}$-measurable $u_{\theta,k}$ and an independent Gaussian increment
\[
    \E^{\Prob}\Bigl[\exp\bigl(u_{\theta,k}^\top\Delta\epsilon^{\Prob}_k-\tfrac12\|u_{\theta,k}\|^2\dt\bigr)\Bigm|\F_{t_k}\Bigr]=1 ,
\]
and the tower property gives
\[
    \E^{\Prob}\Bigl[\frac{\dd\QQ}{\dd\Prob}\Bigr]=1
\]
without any integrability condition on the control beyond measurability.

\medskip
\noindent
\emph{Step 4: locality of the ratio.}
The prefix layers are functionals of $z_0$ and of the prefix driver. By Steps~0--1 the law of $W^{\mathrm{prefix}}$ is the same under $\Prob$ and $\QQ$, the law of $z_0$ changes only through the shift $\mu_\theta$, and both are independent of $\epsilon$: the prefix computation is the same map under both measures, evaluated at a shifted starting point. The final layer is a functional of the prefix path and of $\epsilon$, and Eq.~\ref{eq:app-rn-identity} together with Step~0 shows that the likelihood ratio is a functional of the control, of $z_0$ and of the shifted innovation only. Hence the ratio depends only on the final-layer control, the initial-state shift and on the Brownian innovation being shifted.
\end{proof}

\subsection{Proof of Theorem \ref{thm:wasserstein-density-terminal-laws}}
\label{app:proof-density}
\begin{proof}
Let \(\mu\in\mathcal P_2(\mathbb R^p)\) and let \(\eta>0\).

\medskip
\noindent
\emph{Step 1: construction of an atomless scalar feature.}
Choose a latent width \(d\), to be specified below, and configure the
first layer as
\[
    \dd Z_t^{(1)}
    =
    \beta e_1\,\dd W_t^1,
    \qquad
    \beta\neq 0,
\]
where \(e_1\) is the first coordinate vector. Thus
\[
    Z_t^{(1)}
    =
    \iota_d(z_0)+\beta e_1W_t^1.
\]
In particular,
\[
    S_T
    :=
    \bigl\langle e_1,Z_T^{(1)}\bigr\rangle
    =
    \bigl\langle e_1,\iota_d(z_0)\bigr\rangle+\beta W_T^1.
\]
Since \(W_T^1\) is independent of \(\mathcal F_0\), the law of \(S_T\)
is the convolution of the law of
\(\langle e_1,\iota_d(z_0)\rangle\) with a nondegenerate Gaussian law.
Consequently, \(S_T\) has a smooth density, irrespective of whether
the law of \(z_0\) is degenerate.

Recall the normalisation of Eq.~\ref{eq:app-gate-input},
\[
    \Pi(z)
    =
    \gamma\odot
    \frac{z}{
        \sqrt{\varsigma+d^{-1}\|z\|^2}
    },
    \qquad
    \varsigma>0,
\]
and choose the first gain \(\gamma_1\) to be 1.
Set
\[
    R_t
    :=
    \bigl[
        \Pi\bigl(Z_t^{(1)}\bigr)
    \bigr]_1.
\]
Conditionally on \(z_0\), \(R_T\) is a strictly monotone function of
\(W_T^1\). Indeed, since \(\iota_d(z_0)=(z_0,0,\dots,0)\), for fixed \(z_0\) the map
\[
    u
    \longmapsto
    \gamma_1
    \frac{u}{
        \sqrt{
            \varsigma+d^{-1}
            \left(
                u^2+\sum_{j=2}^{d_0} z_{0,j}^2
            \right)
        }
    }
\]
has derivative
\[
    \gamma_1
    \frac{
        \varsigma+d^{-1}\sum_{j=2}^{d_0} z_{0,j}^2
    }{
        \left[
            \varsigma+d^{-1}
            \left(
                u^2+\sum_{j=2}^{d_0} z_{0,j}^2
            \right)
        \right]^{3/2}
    },
\]
which never vanishes. Hence \(R_T\) is atomless. Moreover, \(R\) has
continuous paths and is bounded.

\medskip
\noindent
\emph{Step 2: reduction to a finitely supported target law.}
Finitely supported probability measures are dense in
\(\mathcal P_2(\mathbb R^p)\). We may therefore choose
\[
    \nu
    =
    \sum_{i=1}^N p_i\delta_{y_i},
    \qquad
    p_i>0,
    \qquad
    \sum_{i=1}^N p_i=1,
\]
such that
\[
    W_2(\mu,\nu)<\frac{\eta}{3}.
\]
Let \(F_R\) be the distribution function of \(R_T\), and define
\[
    P_i:=\sum_{j=1}^i p_j,
    \qquad
    i=1,\ldots,N-1.
\]
Since \(R_T\) is atomless, \(F_R\) is continuous. We may therefore
choose
\[
    q_1<\cdots<q_{N-1}
\]
such that
\[
    F_R(q_i)=P_i.
\]
With \(q_0=-\infty\) and \(q_N=+\infty\), define
\[
    H(r)
    :=
    y_i
    \quad\text{whenever}\quad
    q_{i-1}<r\leq q_i.
\]
It follows that
\[
    \mathcal L\bigl(H(R_T)\bigr)=\nu.
\]
Equivalently,
\[
    H(r)
    =
    y_1
    +
    \sum_{i=1}^{N-1}
    (y_{i+1}-y_i)
    \mathbf 1_{\{r>q_i\}}.
\]

For \(a>0\), define the sigmoid approximation
\[
    H_a(r)
    :=
    y_1
    +
    \sum_{i=1}^{N-1}
    (y_{i+1}-y_i)
    \sigma\bigl(a(r-q_i)\bigr).
\]
Since
\[
    \sigma\bigl(a(r-q_i)\bigr)
    \longrightarrow
    \mathbf 1_{\{r>q_i\}}
\]
for every \(r\neq q_i\), and
\(\mathbb P(R_T=q_i)=0\), we have
\[
    H_a(R_T)\longrightarrow H(R_T)
    \qquad\text{almost surely}.
\]
The family \(\{H_a(R_T):a>0\}\) is uniformly bounded. Hence, by
dominated convergence,
\[
    \left\|
        H_a(R_T)-H(R_T)
    \right\|_{L^2}
    \longrightarrow 0.
\]
Choose \(a>0\) such that
\[
    \left\|
        H_a(R_T)-H(R_T)
    \right\|_{L^2}
    <
    \frac{\eta}{3}.
\]

\medskip
\noindent
\emph{Step 3: realisation by the gated layer.}
Choose \(d\geq N\). Disable the scale gate by taking
\(W_{iF}^{(2)}=0\), together with a zero bias in that branch, so that
\[
    \alpha_t^{(2)}=\mathbf 1,
\]
and set all diffusion coefficients of the second layer equal to zero.
Choose its diagonal drift matrix to be
\[
    A^{(2)}=-\lambda I_d,
    \qquad
    \lambda>0.
\]

Use one coordinate \(V^{0,\lambda}\) to generate the constant feature
and \(N-1\) coordinates \(V^{i,a,\lambda}\) to generate the shifted
sigmoid features:
\begin{align*}
    \dd V_t^{0,\lambda}
    &=
    -\lambda V_t^{0,\lambda}\,\dd t
    +
    \lambda\frac{t}{T}\,\dd t,\\
    \dd V_t^{i,a,\lambda}
    &=
    -\lambda V_t^{i,a,\lambda}\,\dd t
    +
    \lambda\frac{t}{T}
    \sigma\left(
        aR_t-aq_i\frac{t}{T}
    \right)\dd t,
    \qquad
    i=1,\ldots,N-1.
\end{align*}
These dynamics are contained in the stated offset-gate
parameterisation. Indeed, because the gate input contains both \(R_t\)
and \(t\), suitable rows of the gate matrices give
\[
    W_{go,i}^{(2)}x_t^{(1)}
    =
    aR_t-aq_i\frac{t}{T},
    \qquad
    W_{io,i}^{(2)}x_t^{(1)}
    =
    \lambda\frac{t}{T}.
\]
For the constant coordinate, take
\[
    W_{go,0}^{(2)}x_t^{(1)}=0,
    \qquad
    W_{io,0}^{(2)}x_t^{(1)}
    =
    2\lambda\frac{t}{T},
\]
and use \(\sigma(0)=1/2\).

We claim that, as \(\lambda\to\infty\),
\begin{align*}
    V_T^{0,\lambda}
    &\longrightarrow 1,\\
    V_T^{i,a,\lambda}
    &\longrightarrow
    \sigma\bigl(a(R_T-q_i)\bigr),
    \qquad
    i=1,\ldots,N-1,
\end{align*}
in \(L^2\). To see this, let \(\phi\) be any bounded process which is
continuous in \(L^2\) at \(T\). The solution of
\[
    \dd U_t^\lambda
    =
    -\lambda U_t^\lambda\,\dd t
    +
    \lambda\phi_t\,\dd t
\]
satisfies
\[
    U_T^\lambda
    =
    e^{-\lambda T}U_0^\lambda
    +
    \int_0^T
    \lambda e^{-\lambda(T-s)}\phi_s\,\dd s.
\]
Consequently,
\[
\begin{split}
    \left\|U_T^\lambda-\phi_T\right\|_{L^2}
    &\leq
    e^{-\lambda T}
    \left(
        \left\|U_0^\lambda\right\|_{L^2}
        +
        \left\|\phi_T\right\|_{L^2}
    \right)\\
    &\quad+
    \int_0^T
    \lambda e^{-\lambda(T-s)}
    \left\|\phi_s-\phi_T\right\|_{L^2}\,\dd s,
\end{split}
\]
which converges to zero. Applying this observation with
\[
    \phi_t=\frac{t}{T}
\]
and
\[
    \phi_t
    =
    \frac{t}{T}
    \sigma\left(
        aR_t-aq_i\frac{t}{T}
    \right)
\]
proves the claim.

Choose the terminal readout so that
\[
    Y_T^{a,\lambda}
    =
    y_1V_T^{0,\lambda}
    +
    \sum_{i=1}^{N-1}
    (y_{i+1}-y_i)V_T^{i,a,\lambda}.
\]
Then
\[
    Y_T^{a,\lambda}
    \longrightarrow
    H_a(R_T)
    \qquad\text{in }L^2
\]
as \(\lambda\to\infty\). Choose \(\lambda\) sufficiently large that
\[
    \left\|
        Y_T^{a,\lambda}-H_a(R_T)
    \right\|_{L^2}
    <
    \frac{\eta}{3}.
\]

Finally, using the coupling inequality
\[
    W_2\bigl(\mathcal L(X),\mathcal L(X')\bigr)
    \leq
    \|X-X'\|_{L^2},
\]
we obtain
\begin{align*}
    W_2\bigl(
        \mathcal L(Y_T^{a,\lambda}),
        \mu
    \bigr)
    &\leq
    \left\|
        Y_T^{a,\lambda}-H_a(R_T)
    \right\|_{L^2}\\
    &\quad+
    \left\|
        H_a(R_T)-H(R_T)
    \right\|_{L^2}
    +
    W_2(\nu,\mu)
    <\eta.
\end{align*}
Since \(\mu\) and \(\eta\) were arbitrary, the result follows.

\medskip
\noindent
\emph{Step 4: discretised implementation.}
Once \(a\) and \(\lambda\) have been fixed, the configuration of Steps~1--3 is a fixed two-layer model, and its discretisation on a grid \(\pi\) converges. The first layer has constant coefficients and additive noise only, so its Euler--Maruyama iterates are exact,
\[
    \bar Z^{(1)}_k
    =
    \iota_d(z_0)+\beta e_1W^1_{t_k},
\]
and the discrete gate of the second layer therefore reads the exact feature \(R_{t_k}\) at the grid points. The second layer has no diffusion, so, pathwise, it is the linear ordinary differential equation
\[
    \dd U_t^\lambda
    =
    -\lambda U_t^\lambda\,\dd t
    +
    \lambda\phi_t\,\dd t
\]
of Step~3, with a forcing \(\phi\) that is bounded and continuous in \(t\). Its discretisation is the explicit Euler scheme, which converges to \(U^\lambda\) uniformly on \([0,T]\) along every path as \(|\pi|\to0\), and which is bounded by \(\|z_0\|+\lambda T\sup_t|\phi_t|\) once \(|\pi|\le1/\lambda\). By dominated convergence, the discretised terminal output \(Y_T^{a, \lambda,\pi}\) satisfies
\[
    \left\|
        Y_T^{a, \lambda,\pi}-Y_T^{a, \lambda}
    \right\|_{L^2}
    \longrightarrow 0
    \qquad\text{as}\qquad
    |\pi|\longrightarrow0.
\]
When \(z_0\) has finite moments of all orders this is also a special case of Theorem~\ref{thm:disc-error}. Then the same density conclusion holds for the discretised
architecture when the mesh is allowed to vary. Indeed,
\[
    W_2\left(
        \mathcal L(Y_T^{a, \lambda,\pi}),
        \mathcal L(Y_T^{a, \lambda})
    \right)
    \leq
    \left\|
        Y_T^{a, \lambda,\pi}-Y_T^{a, \lambda}
    \right\|_{L^2},
\]
and the discretisation is refined only after the finite parameters
\(a\) and \(\lambda\) have been selected.
\end{proof}

\newpage

\section{Implementation Details}
\label{app:implementation-details}

This appendix expands the implementation choices summarised in the main
text. All reported SLiSDE results use the gated in-flow stack (Section~\ref{ssec:gated-stacking}) with structured block-diagonal base layers, an always-on RMS normalisation of the gate input, optional causal convolutions, and time-feature augmentation; residual stacking is not used. The Girsanov tilt described
in Section~\ref{sec:girsanov} is applied only to the final layer, with the
correlated-driver construction of Eq.~\ref{eq:correlated-driver}.

\subsection{Discretisation of a linear layer}
\label{app:discretisation}

On the interval $[t_k,t_{k+1}]$, the affine linear SDE Eq.~\ref{eq:linear-sde} admits the exponential approximation
\begin{equation}
\begin{gathered}
    Z_{k+1}=F_kZ_k+g_k,
    \qquad
    g_k=b\dt+D\dW_k,\\
    F_k=\exp\left(
        \left(A-\frac12\sum_{j=1}^m(C^j)^2\right)\dt
        +\sum_{j=1}^m C^j\,\dW_k^j
    \right).
\end{gathered}
    \label{eq:exponential-affine-flow}
\end{equation}
The correction is the It\^o correction. If $A,C^1,\ldots,C^m$
commute pairwise, then $F_k$ is the exact flow of the homogeneous
It\^o equation; otherwise it is an exponential approximation. The
affine term $g_k$ remains a first-order approximation.

To avoid computing matrix exponentials, we use the Euler--Maruyama
transition
\begin{equation}
    Z_{k+1}
    =
    \left(I+A\dt+\sum_{j=1}^m C^j\,\dW_k^j\right)Z_k
    +b\dt+D\dW_k,
    \label{eq:linear-approx-flow}
\end{equation}
which preserves the affine recurrence structure and is used throughout
our experiments. With the time-dependent coefficients of Appendix~\ref{app:time-dependent-coefficients}, the static coefficients are replaced by their values at $t_k$ in both transitions; the strong error of the Euler--Maruyama scheme is analysed in Appendix~\ref{app:proof-disc}.

\subsection{Time-dependent coefficient map}
\label{app:time-dependent-coefficients}

Each base layer (Section~\ref{ssec:base-layer}) stores static structured coefficients
$(A,b,C^1,\ldots,C^m,d^1,\ldots,d^m)$ and can be made time-dependent
through a small deterministic decoder. For each grid time $t_k$ we form

\[
    \tau(t_k)
    =
    \bigl(1,\, t_k,\, t_k^2,\, \sin(\omega t_k),\, \cos(\omega t_k)\bigr),
    \qquad
    h(t_k) = \tanh \bigl(U\tau(t_k)+c\bigr),
\]

and add structure-aware corrections to the free coefficient entries:

\[
    A_{t_k} = A + \Delta A \bigl(h(t_k)\bigr),
    \quad
    b_{t_k} = b + \Delta b \bigl(h(t_k)\bigr),
    \quad
    C_{t_k}^j = C^j + \Delta C^j \bigl(h(t_k)\bigr).
\]

For additive diffusion we use a per-Brownian-channel time gate on the
dense diffusion matrix, equivalently producing $d_{t_k}^j$ from
$d^j$ and $h(t_k)$. The decoders respect the chosen matrix structure:
diagonal layers decode only diagonal entries, block-diagonal layers decode
only block entries, and dense layers decode full matrices. Decoders are
zero-initialised, so the model starts from the time-homogeneous structured
SDE and learns time variation only when useful.

\subsection{Feature conditioning in gated layers}
\label{app:feature-conditioning}

The gated layer consumes the previous layer path at the same time
grid as the current layer. The simplest gate input is
$\operatorname{Norm}(Z_k^{(\ell-1)})$, where $Z_k^{(\ell-1)}$ is the model output after $\ell-1$ layers and the normalisation is RMSNorm
or LayerNorm (RMSNorm by default, applied at every gated layer). Note that normalisation controls the \emph{scale} of the gate input but does not by itself bound the paths; the offset gate is nevertheless bounded for every fixed choice of its weights, which is what the analysis uses (Appendix~\ref{app:assumptions}). In the implementation this signal can be enriched in two
ways.

\paragraph{Causal convolution (for gated-in-flow stacking only).}
Before normalisation, the previous-layer path can be passed through a
left-padded depthwise causal convolution along the time axis. The value at
time $t_k$ depends only on previous-layer states up to $t_k$, so the gate is
adapted and does not leak future information. The convolution is wrapped in
a residual update and its kernel is small at initialisation, making the
initial behaviour close to the no-convolution case while still allowing the
gate to learn short temporal summaries.

\paragraph{Time features.}
The deterministic vector
\begin{equation}
    \tau_{\mathrm{gate}}(t)
    =
    (t,t^2,\sin(\omega t),\cos(\omega t))
    \label{eq:time-features}
\end{equation}
can be concatenated to the gate input after convolution and after
normalisation. Appending time features after
normalisation preserves their scale, while keeping the convolution focused
on the stochastic path rather than on deterministic batch-constant signals.
With both refinements, the gate input is
\begin{equation}
    x_k
    =
    \bigl[
        \operatorname{Norm} \left(\operatorname{CausalConv}(H^{(\ell-1)})_k\right),
         t_k, t_k^2, \sin(\omega t_k), \cos(\omega t_k)
    \bigr].
    \label{eq:full-gate-input-appendix}
\end{equation}

\subsection{Gated in-flow implementation}\label{ssec:gated_in_flow_app}
The paper focuses on the gated in-flow stacking. Given the conditioned feature $x_k^{(\ell-1)}$
from Eq.~\ref{eq:full-gate-input-appendix}, the layer constructs a
bounded diagonal scale

\begin{equation}
    \alpha_k^{(\ell)}
    =
    \mathbf{1}
    +
    \varepsilon\,
    \sigma \bigl(W_{gF}^{(\ell)}x_k^{(\ell-1)}\bigr)
    \odot
    \tanh \bigl(W_{iF}^{(\ell)}x_k^{(\ell-1)}\bigr),
    \label{eq:alpha-gate}
\end{equation}

and an affine residual offset

\begin{equation}
    o_k^{(\ell)}
    =
    \sigma \bigl(W_{go}^{(\ell)}x_k^{(\ell-1)}\bigr)
    \odot
    \bigl(W_{io}^{(\ell)}x_k^{(\ell-1)}+c_{io}^{(\ell)}\bigr).
    \label{eq:offset-gate}
\end{equation}

All gate maps carry trainable biases, omitted from the displays except $c_{io}^{(\ell)}$. The $\tanh$ bound gives
$\alpha_k^{(\ell)}\in(\mathbf{1}-\varepsilon,\mathbf{1}+\varepsilon)$
coordinatewise. The scale $\alpha_k^{(\ell)}$ is applied only to the
diagonal entries of the structured drift matrix $F_k^{(\ell)}$, and the
offset $o_k^{(\ell)}$ is added to the affine intercept $g_k^{(\ell)}$.
Gate kernels are initialised at small scale, so initially
$\alpha_k^{(\ell)}\approx\mathbf{1}$ and $o_k^{(\ell)}\approx\mathbf 0$:
the stacked model starts close to an unmodulated structured SDE and learns
cross-layer coupling gradually. We fix $\varepsilon = 0.1$ throughout.

The offset $o_k^{(\ell)}$ is an \emph{unbounded} GLU: the gate $\sigma(W_{go}^{(\ell)}x_k^{(\ell-1)})$ is bounded, while the value branch $W_{io}^{(\ell)}x_k^{(\ell-1)}+c_{io}^{(\ell)}$ is linear in the gate input. Because the gate input is normalised and therefore lies in a ball of radius $\Lambda_x$, $\|o_k^{(\ell)}\|\le\|W^{(\ell)}_{io}\|\Lambda_x+\|c^{(\ell)}_{io}\|$ for every fixed choice of weights, which is the boundedness used in Appendix~\ref{app:proofs}; the magnitude itself is free, which is what the proof of Theorem~\ref{thm:wasserstein-density-terminal-laws} exploits. The drift scale $\alpha_k^{(\ell)}$ is applied only to the diagonal entries of the structured transition matrix $F_k^{(\ell)}$ (a diagonal, row-wise Hadamard action), while $o_k^{(\ell)}$ is added directly to the affine intercept $g_k^{(\ell)}$ with no additional $\dt$ factor. Unless a prior is learned, the shared initial state is $Z_0\sim\mathcal N(0,I_d)$, sampled once and reused by every layer.

\subsection{Last-layer Girsanov tilt: design details}
\label{app:girsanov-details}
This subsection collects the design choices behind Section~\ref{sec:girsanov}.

\paragraph{Why the initial state is tilted.} The innovation of the final layer is not the only source of randomness the rare event depends on. In a gated stack the random initial latent $z_0\sim\mathcal N(0,I_d)$ carries a large share of the terminal variance -- on the \textsc{toy} model of Section~\ref{sec:experiments} the final-layer innovation explains only about a quarter of $\operatorname{Var}(Y_T)$ -- and a controller acting on the innovation alone cannot reach the rare region: every such variant we trained kept the hit rate of the tilted paths at the plain Monte-Carlo level for $p\le10^{-3}$. Under $\QQ$, $z_0=\mu_\theta+\xi$ with $\xi\sim\mathcal N(0,I_d)$ and a learned $\mu_\theta\in\R^d$ (one per controller), which contributes the closed-form factor $\mu_\theta^\top z_0-\tfrac12\|\mu_\theta\|^2$ of Eq.~\ref{eq:rn-discrete} to the log-likelihood ratio and $\tfrac12\|\mu_\theta\|^2$ to $\KL(\QQ\|\Prob)$. The prefix driver $W^{\mathrm{prefix}}$ keeps its law, so the prefix layers are computed by the same map under both measures; only the distribution of their starting point moves.

\paragraph{Controller.} The control $u_{\theta,k}$ is an adapted function of the reference innovation history $\{\dd\epsilon_j^{\Prob}\}_{j<k}$ and of stop-gradient features of the prefix latent path $\{L_j^{(L-1)}\}_{j\le k}$; it is parameterised by the diagonal state-space recurrence of Appendix~\ref{app:girsanov-controller-details}, driven by both inputs, with recurrence rates constrained to $(0,1)$ and zero drift at initialisation.

\paragraph{Objectives and gradient convention.} Backbone and controller are trained on different objectives with separate optimisers. The backbone minimises the calibration loss, in which each far-tail term is estimated on the tilted batch by self-normalised importance sampling (Appendix~\ref{app:snis-ess}), combined with the plain estimate on the reference batch in proportion to the two effective sample sizes. The controller is trained by the cross-entropy method: it minimises a Monte-Carlo estimate of $\KL(\QQ^\star\|\QQ_\theta)$, where $\QQ^\star\propto|G|\,\Prob$ is the zero-variance proposal of the rare-event functional, under a quadratic wall on $\KL(\QQ_\theta\|\Prob)$ that keeps the proposal in the regime where the weights remain usable (Appendix~\ref{app:kl-regulariser}). The gradient convention matters: the cross-entropy objective is a score-function estimator and requires the log-ratio Eq.~\ref{eq:rn-discrete} to be differentiated with the sampled path held fixed, whereas Monte-Carlo functionals of the weights, such as the KL wall, require the reparameterised form in which the tilted noise is held fixed; the two forms have the same value and different fixed points, and using the reparameterised one for the cross-entropy objective drives the tilt away without bound.

\paragraph{Which estimator.} Because the ratio is normalised, the vanilla importance-sampling estimator $n_q^{-1}\sum_iw_iG(X^{(i)})$ is unbiased. For the calibration loss we use the self-normalised estimator (Appendix~\ref{app:snis-ess}): the weights are computed pathwise in log-space, self-normalisation is invariant to a common log-shift, so a max-shift can be applied before exponentiation and no single extreme path can dominate through overflow, the ratio form cancels weight fluctuations shared by numerator and denominator, and its $O(1/n_q)$ bias is negligible at our operating point. For rare-event probabilities the vanilla estimator is the better one: the initial-state shift gives large weights to tilted paths that do not reach the event, and those paths enter only the normalisation of the self-normalised estimator. Both estimators are compared in Appendix~\ref{app:additional-results}.

\subsection{Girsanov controller}
\label{app:girsanov-controller-details}

The controller returns an adapted sequence $u_{\theta,k}\in\R^m$ that enters the Radon--Nikodym formula Eq.~\ref{eq:rn-discrete}. It is a diagonal state-space recurrence with two branches, one driven by the tilted innovation and one by the stop-gradient prefix latent path,
\[
\begin{gathered}
    s_{k+1}^W=\diag(r_W)\,s_k^W+U_W\,\Delta\epsilon^{\QQ}_k,
    \qquad
    s_{k+1}^Z=\diag(r_Z)\,s_k^Z+U_Z\,L_k^{(L-1)},\\
    u_{\theta,k}=W_{\mathrm{mix}}\bigl[s_k^W,s_k^Z\bigr]+c ,
\end{gathered}
\]
so that $u_{\theta,k}$ depends on the increments before step $k$ only. The recurrence is itself affine and can be scanned in parallel.

\paragraph{Implementation details.} Three details proved necessary.
\begin{enumerate}[label=(\roman*),leftmargin=1.8em,itemsep=0.1em,topsep=0.2em]
    \item The recurrence rates of both branches are parameterised as $\exp(-\operatorname{softplus}(-r))\in(0,1)$, initialised at $0.9997$, because an unconstrained rate can drift above one, after which the recurrence grows geometrically over the horizon and overflows, producing non-finite weights.
    \item The mixing weights of the latent-path branch are initialised at zero, so that the untrained controller is the identity tilt with $\KL(\QQ\|\Prob)=0$; with a small random initialisation the branch, which integrates the $O(1)$ prefix latent over the whole horizon, gives the untrained controller a divergence of about $8$ nats per controller, i.e.\ a collapsed proposal before any training.
    \item The initial-state shift $\mu_\theta$ is a free vector per controller initialised at zero.
\end{enumerate}
Controller parameters are trained by Adam with learning rate $2\cdot10^{-3}$ and gradient clipping at norm one, separately from the backbone optimiser: with a single optimiser and a shared global-norm clip, the cross-entropy gradient, three orders of magnitude larger than the calibration gradient, scaled the backbone updates to zero after the controller was activated.

\subsection{Self-normalised importance sampling and ESS}
\label{app:snis-ess}

Given tilted samples $\{X^{(i)}\}_{i=1}^{n_q}$ drawn under $\QQ$ with
unnormalised weights
\[
    w_i = \frac{\dd\Prob}{\dd\QQ}(X^{(i)}),
\]
the self-normalised importance-sampling (SNIS) estimator of
$\E^{\Prob}[G]$ is
\begin{equation}
    \widehat{\E}^{\Prob,\mathrm{SNIS}}[G]
    =
    \frac{\sum_{i=1}^{n_q}w_iG(X^{(i)})}
         {\sum_{i=1}^{n_q}w_i}.
    \label{eq:snis-appendix}
\end{equation}
The estimator is biased at finite $n_q$ but consistent. We monitor the
quality of the tilted sample through the effective sample size
\begin{equation}
    \ESS
    =
    \frac{\bigl(\sum_{i=1}^{n_q}w_i\bigr)^2}
         {\sum_{i=1}^{n_q}w_i^2}
    \in [1,n_q].
    \label{eq:ess-appendix}
\end{equation}

$\ESS=n_q$ iff all weights are equal; $\ESS=1$ iff one weight dominates.
In practice we compute the log-weights $\log w_i = \log(\dd\Prob/\dd\QQ)(X^{(i)})$ directly from the discrete Radon--Nikodym density Eq.~\ref{eq:rn-discrete} and normalise them with a numerically stable log-sum-exp, so that the SNIS estimator Eq.~\ref{eq:snis-appendix} and the $\ESS$ ratio remain well behaved even when a few tilted paths dominate. In training logs we report the ratio $\ESS/n_q\in[1/n_q,1]$ for comparability across batch sizes; a controller that drives $\ESS/n_q$ toward $1/n_q$ signals weight collapse and is held back by the KL wall of Appendix~\ref{app:kl-regulariser}. For rare-event probabilities we report the vanilla estimator $n_q^{-1}\sum_iw_iG(X^{(i)})$ as the primary one: with the initial-state shift, tilted paths that do not reach the event can carry large weights, which affect only the denominator of Eq.~\ref{eq:snis-appendix} and make the self-normalised estimator markedly worse (Table~\ref{tab:girsanov-full}). Its truncated variant caps the weights at $\sqrt{n_q}\,\E^{\QQ}[w]=\sqrt{n_q}$, using the known mean rather than the sample mean, which the very weights to be tamed would inflate.

\subsection{Training objectives of backbone and controller}\label{app:kl-regulariser}

\paragraph{Backbone.} The backbone parameters minimise the calibration loss Eq.~\ref{eq:calib-loss}. Terms that involve a rare-event functional $G$ are estimated on the tilted batch by the self-normalised estimator Eq.~\ref{eq:snis-appendix}, and, in the combined variant, averaged with the plain estimate on the reference batch of size $n_p$ with the weights $n_p/(n_p+\ESS)$ and $\ESS/(n_p+\ESS)$, held fixed for the gradient; all other terms are plain Monte-Carlo estimates on the reference batch. The backbone receives no gradient through the weights.

\paragraph{Controller.} The controller parameters $\theta_c=(u_\theta,\mu_\theta)$ minimise the cross-entropy objective with a KL wall,
\begin{equation}
\begin{gathered}
    \widehat{\mathcal J}(\theta_c)
    =
    -\sum_{i=1}^{n_q}\bar w^\star_i\,\log\frac{\dd\QQ_{\theta_c}}{\dd\Prob}(X^{(i)})
    +\lambda_{\KL}\,\widehat{\KL}(\QQ_{\theta_c}\|\Prob)
    +\bigl(\widehat{\KL}(\QQ_{\theta_c}\|\Prob)-\kappa\bigr)_+^2 ,\\
    \bar w^\star_i\propto|G(X^{(i)})|\,\frac{\dd\Prob}{\dd\QQ_{\theta_c}}(X^{(i)}),
\end{gathered}
    \label{eq:kl-objective}
\end{equation}
where the self-normalised target weights $\bar w^\star_i$ are held fixed for the gradient (their log-ratios clipped at $\pm5$ so that no single path can dominate the update), $\kappa$ is the wall in nats ($6$ on \textsc{dax}; $-\log p+2$ for an event of nominal probability $p$ on \textsc{toy}), and $\widehat{\KL}$ is the energy form $n_q^{-1}\sum_i\bigl(\tfrac12\sum_k\|u_{\theta,k}\|^2\dt+\tfrac12\|\mu_\theta\|^2\bigr)$, whose $\QQ$-expectation equals $\KL(\QQ\|\Prob)$ exactly. The first term is the Monte-Carlo estimate of $\KL(\QQ^\star\|\QQ_{\theta_c})$ up to a constant, with $\QQ^\star\propto|G|\,\Prob$ the zero-variance proposal; its minimiser is the moment match $\E^{\QQ^\star}$ of the sufficient statistics of the tilt (the conditional mean of $z_0$ and of the innovation increments given the rare event). The wall is needed because the weak penalty $\lambda_{\KL}=10^{-4}$ alone does not prevent the collapse of the proposal onto a handful of paths once the target weights degenerate.

\paragraph{Gradient convention.} Eq.~\ref{eq:rn-discrete} can be differentiated in two ways with the same value. In the score form the sampled path is held fixed: $\partial_{u_{\theta,k}}\log(\dd\QQ/\dd\Prob)=\Delta\epsilon^{\QQ}_k$ and $\partial_{\mu_\theta}=z_0-\mu_\theta$, the scores of the tilted law, which is what the cross-entropy term requires. In the reparameterised form the tilted noise is held fixed and the derivative is $\Delta\epsilon^{\Prob}_k$ (respectively $z_0$), which is the correct total derivative of a $\QQ$-expectation of a function of the weights, such as the KL wall or the effective sample size. Using the reparameterised form inside the cross-entropy term gives an update whose fixed point is not the moment match and which pushes the tilt outward without bound.

\newpage
\section{Experimental Setup and Hyperparameter Grids}
\label{app:expset}
\label{app:hyperparams}
\subsection*{\textsc{toy} Dataset}

\textsc{toy} is a controlled synthetic benchmark designed to separate path
fitting from path-functional calibration. The ground truth is a
three-component nonlinear It\^o SDE on $[0,T]$ with $T = 1$, state
$X_t = (X^1_t, X^2_t, X^3_t)$ and diagonal state-dependent diffusion:
\begin{align*}
    \dd X^1_t &= -X^1_t\bigl((X^1_t)^2 - 1\bigr)\dd t
         +  \sqrt{|X^1_t| + 1} \sqrt[3]{X^1_t + 5} \dd W^1_t,
        & X^1_0 &= 0.5,\\
    \dd X^2_t &= \sin(X^2_t)\bigl(2 - (X^2_t)^2\bigr)\dd t
         +  \sqrt[3]{1 + (X^2_t)^2} \dd W^2_t,
        & X^2_0 &= 0,\\
    \dd X^3_t &= \bigl(X^3_t - (X^3_t)^3/3\bigr)\dd t
         +  \exp \bigl(-(X^3_t)^2/4\bigr)\,X^3_t\,\dd W^3_t,
        & X^3_0 &= 1,
\end{align*}

with mutually independent driving Brownians. The observable is the linear
combination
$$Y_t = 0.40\,X^1_t + 0.35\,X^2_t + 0.25\,X^3_t.$$

We simulate $2{\times}10^{5}$ Monte-Carlo paths via Euler--Maruyama on a
uniform grid of $N = 2048$ steps. From the simulated paths we compute four classes of supervised path-functional targets at the evaluation times
$T_i\in\mathcal T = \{0.1,\,0.25,\,0.5,\,0.75,\,1.0\}$.

These targets are:
\begin{itemize}[leftmargin=1.2em,topsep=0.2em]
    \item \textbf{Thresholded positive-part functionals} at levels
    $\mathcal K = \{-0.5,-0.25,0,0.25,0.5,0.75,1.0\}$:
    $$\E[\max(Y_{T_i} - K, 0)].$$
    \item \textbf{Running maximum}: $$\E[\max_{0\le t\le T_i}Y_t].$$
    \item \textbf{Squared path average}:
    $$\E \bigl[\tfrac{1}{T_i}\int_0^{T_i} Y_t^2\,\dd t\bigr],$$ computed on the
    discrete grid.
    \item \textbf{Threshold-crossing probabilities}
    $$\mathbb P(\max_{0\le t\le T_i} Y_t \ge H_j)$$ at three threshold levels $H_j$ chosen as the $80\%$, $90\%$, and $95\%$ empirical quantiles of the terminal running maximum; in the objective the indicator is replaced by a smooth approximation (a sigmoid of the excess over the threshold with a fixed temperature), so that Eq.~\ref{eq:calib-loss} is differentiable.
\end{itemize}

The threshold-crossing and high-level positive-part targets are deliberately
tail-sensitive path functionals. Errors on typical-path statistics and on tail functionals need
not agree: a model may fit typical trajectories well while still missing
the rare excursions or extrema that determine these calibration targets.
Thus \textsc{toy} provides a controlled setting for evaluating whether a
model can match both ordinary path behaviour and rare path-dependent
statistics.

\subsection*{\textsc{DAX} Dataset}

\textsc{dax} is a real-data functional calibration benchmark built from
historical European option quotes on the DAX index. The purpose of the
benchmark is to test whether a stochastic path generator can match a large
collection of expectations of nonlinear functions of its terminal value.
Each option quote is treated as a supervised target of the form
\[
    \mathbb E[\varphi(Y_T;K)],
\]
where \(Y\) is the generated one-dimensional path, \(T\) is the evaluation
time, \(K\) is a threshold level, and \(\varphi\) is a payoff function. All
thresholds and target values are normalised by the spot level
\(S_0=25{,}280\), so that the generated process starts from \(Y_0=1\) and
\(K_{\mathrm{norm}}=K/S_0\).

\paragraph{Option payoffs (for the non-specialist).}
A European option on the index with maturity $T$ and strike $K$ pays a fixed function of the terminal value $Y_T$ alone. The two basic instruments are the \emph{call} and the \emph{put}, with payoffs
\begin{equation}
    \varphi_{\mathrm{call}}(Y_T;K)=(Y_T-K)_+,
    \qquad
    \varphi_{\mathrm{put}}(Y_T;K)=(K-Y_T)_+,
    \qquad
    (x)_+:=\max(x,0).
    \label{eq:option-payoffs}
\end{equation}
Under the normalised, discounting-free pricing convention used here, the quoted price of an option is the expectation
\begin{equation}
    c_{T,K}=\mathbb E\bigl[\varphi(Y_T;K)\bigr],
    \label{eq:option-price}
\end{equation}
which is exactly the supervised target above. A strike is \emph{at-the-money} when $K\approx S_0$ (i.e.\ $K_{\mathrm{norm}}\approx1$), \emph{in-the-money} when the option would pay out if exercised at the current level, and \emph{out-of-the-money} (OTM) otherwise, that is, a call with $K>S_0$ or a put with $K<S_0$. OTM prices are small and depend only on the tails of the terminal law, which is why the far-tail groups are the hardest part of the benchmark. The threshold-indicator targets approximate digital payoffs through finite differences of nearby put prices in the strike,
\begin{equation}
    \mathbb P(Y_T\le K)
    =
    \partial_K\,\mathbb E\bigl[(K-Y_T)_+\bigr]
    \approx
    \frac{\mathbb E\bigl[(K+h-Y_T)_+\bigr]-\mathbb E\bigl[(K-h-Y_T)_+\bigr]}{2h},
    \label{eq:digital-payoff}
\end{equation}
and analogously $\mathbb P(Y_T>K)=-\partial_K\,\mathbb E[(Y_T-K)_+]$ from call prices.

The dataset is intentionally dense: at each available maturity we include
several groups of terminal functionals covering both typical and tail
regions of the path distribution:
\begin{itemize}[leftmargin=1.2em,itemsep=0.1em,topsep=0.2em]
    \item \(25\) positive-part upper-tail targets,
    \(K_{\mathrm{norm}}\in[0.88,1.12]\), corresponding to
    \(\mathbb E[(Y_T-K)_+]\);
    \item \(25\) positive-part lower-tail targets,
    \(K_{\mathrm{norm}}\in[0.88,1.12]\), corresponding to
    \(\mathbb E[(K-Y_T)_+]\);
    \item \(20\) far lower-tail targets,
    \(K_{\mathrm{norm}}\in[0.60,0.88]\);
    \item \(15\) far upper-tail targets,
    \(K_{\mathrm{norm}}\in[1.08,1.35]\);
    \item \(12\) threshold-indicator targets, synthesised from differences
    of nearby lower-tail positive-part targets,
    \(K_{\mathrm{norm}}\in[0.65,0.98]\).
\end{itemize}
The target threshold levels are evenly spaced on each interval and are
matched to the nearest observed threshold in the raw quote table. Maturity
indices are recomputed on the model time grid with \(T=1\) and
\(N=1024\) steps.

The far-tail groups are the most challenging part of the benchmark because
they depend on rare terminal events under the generated path distribution.
When the Girsanov overlay is used for \textsc{dax}, \textbf{we use two separate
controllers}: one targets the far lower-tail functionals and the other
targets the far upper-tail functionals. This lets each controller focus on
a different rare region of path space, rather than requiring a single tilt
to cover both tails simultaneously.

\textbf{Train/eval split.} Evaluation uses a structured hold-out over strike levels. For each group of targets, one quarter of the strikes are held out for evaluation and the remaining entries are used for training. The held-out strikes are selected non-contiguously, so that the training and evaluation sets cover overlapping ranges. This tests interpolation across the surface rather than extrapolation outside the observed range.

\subsection*{\textsc{SPX} Dataset}
\textsc{spx} is a second real-data option-surface benchmark, constructed identically to \textsc{dax} but from historical European option quotes on the S\&P~500 index. As for \textsc{dax}, each quote is a supervised target $\mathbb E[\varphi(Y_T;K)]$ with the call and put payoffs defined above; all strikes and target values are normalised by the spot level so that $Y_0=1$ and $K_{\mathrm{norm}}=K/S_0$, and maturities are recomputed on the model grid with $T=1$ and $N=1024$ steps. The surface is organised into the same five target groups (near-tail calls and puts, far lower- and upper-tail targets, and threshold-indicator targets) spanning comparable normalised-strike ranges, and evaluation uses the same structured non-contiguous $25\%$ strike hold-out, testing interpolation across the surface. When the Girsanov overlay is used, we again employ two controllers, one per far tail. \textsc{spx} serves as an out-of-sample check that the architecture and calibration procedure transfer across underlyings; results are reported in Table~\ref{tab:spx-full} of Appendix~\ref{app:additional-results}.

\subsection*{Training protocol}

All models use the same optimiser and learning-rate schedule unless otherwise stated: AdamW with $\beta=(0.9,0.999)$, peak learning rate $10^{-3}$, $100$ warm-up steps followed by cosine decay, gradient clipping at $1$, batch size $1024$, evaluation batch size $2048$, $1\,000$ epochs and seven seeds per configuration; every reported number is a mean $\pm$ standard error over the seven seeds of the cell in question. The models are simulated on $N=512$ time steps. SLiSDE uses the offset gate of Eq.~\ref{eq:offset-gate} with the RMS-normalised gate input; the search grids are listed in Appendix~\ref{app:grids}, and the best cell per family and depth is reported.

For Girsanov runs the controller is kept inactive for the first $200$ epochs. After activation each training step simulates $1\,024$ reference paths and two tilted batches of $512$ paths, one per controller (far puts, far calls), with $\rho=0$. Each far-tail term is estimated by self-normalised importance sampling on its tilted batch, combined with the plain estimate on the reference batch in proportion to the effective sample sizes (the `tilted batch only' variant omits the reference-batch estimate); all other terms are evaluated under $\Prob$ on the reference batch. The controllers minimise Eq.~\ref{eq:kl-objective} with $\lambda_{\KL}=10^{-4}$, a $6$-nat wall, target log-weights clipped at $\pm5$ and Adam with learning rate $2\cdot10^{-3}$ on their own parameters; the backbone keeps the optimiser and schedule of the vanilla runs. The tilt studies use seven seeds and $1\,000$ epochs; the tilted arms cost $1.3\times$ the epoch time of the budget-matched untilted arm ($303$ against $229$ ms).

\subsection*{Search grids}
\label{app:grids}

\paragraph{SLiSDE backbone.}
\begin{itemize}[leftmargin=1.2em,itemsep=0.1em,topsep=0.2em]
    \item matrix structure: \texttt{blockdiag} with block size
    $b\in\{1,4,8,16\}$
    \item latent dimension $d\in\{32,64\}$
    \item layers $L\in\{1,2,3\}$
    \item noise dimension $m\in\{4, 8,16\}$
    \item time-feature variant
    $\in\{\texttt{True},\texttt{False}\}$
    \item normalisation $\in\{\texttt{rmsnorm},\texttt{layernorm}\}$
    \item gate amplitude $\varepsilon = 0.1$
    \item LoRA rank of $D$ $\in\{0,8\}$
    \item time-embedding dimension $8$
\end{itemize}

\paragraph{Girsanov controllers.}
\begin{itemize}[leftmargin=1.2em,itemsep=0.1em,topsep=0.2em]
    \item tilt type
    \texttt{state\_space} with the latent-path branch, initial-state shift on
    \item controller objective $\in\{$cross-entropy, SNIS-variance$\}$; far-tail estimate $\in\{$combined, tilted batch only$\}$
    \item KL wall $\kappa\in\{2,6\}$ nats, $\lambda_{\KL}=10^{-4}$
    \item last-layer BM correlation $\rho\in\{0.0,0.9\}$
\end{itemize}

\paragraph{Neural-SDE baseline.}
\begin{itemize}[leftmargin=1.2em,itemsep=0.1em,topsep=0.2em]
    \item latent dimension $\in\{32,64\}$
    \item hidden width $\in\{64,128\}$
    \item hidden layers $\in\{2,3\}$
    \item noise dimension $m\in\{4, 8,16\}$
    \item activation $\in\{\texttt{tanh}, \texttt{GeLU}\}$
    \item diffusion type $\in\{\texttt{full},\texttt{diagonal}\}$
\end{itemize}

\paragraph{SLiCE baseline.}
\begin{itemize}[leftmargin=1.2em,itemsep=0.1em,topsep=0.2em]
    \item layers $\in\{2,3,4\}$
    \item hidden dimension $\in\{32,64\}$
    \item noise dimension $\in\{4, 8,16\}$
    \item block size $b\in\{1,4,8,16\}$
\end{itemize}

\paragraph{\textsc{dax} ablations.} Two one-factor sweeps around a two-layer \textsc{dax} configuration of SLiSDE ($d=64$, block size $8$, general noise): the noise type $\in\{$multiplicative, additive, general$\}$ at block size $8$, and the block size $\in\{1,4,8,16\}$ at general noise.

\newpage
\section{Additional Experimental Results}
\label{app:additional-results}
\label{app:ablations}

This appendix collects the results referred to from Section~\ref{sec:experiments}: the \textsc{spx} surface at two and three layers, the structural ablations (noise type and block size), the timing study, the distributional-recovery study on \textsc{toy}, and the Girsanov tilt studies on \textsc{dax} and \textsc{toy}. Unless stated otherwise, each study varies a single factor of the tuned gated in-flow cell of the corresponding main-text table, with the calibration targets held fixed.

\subsection*{\textsc{spx} option-surface dataset}
Table~\ref{tab:spx-full} reports the \textsc{spx} option-surface benchmark for two- and three-layer configurations of every family under the protocol of Table~\ref{tab:dax3-main} (held-out $25\%$ strike mask; tuned configuration per model; seven seeds). SLiSDE matches the held-out loss of SLiCE with two layers ($6.44$ against $6.43$) and attains the lowest held-out loss and the lowest far-put error with three layers, at $6$--$8\times$ fewer parameters and $3.5$--$4\times$ faster epochs. The Neural SDE is clearly worse on both metrics at both depths.

\begin{table}[ht]
\centering
\footnotesize
\captionsetup{font=footnotesize}
\setlength{\tabcolsep}{4pt}
\begin{tabular*}{\textwidth}{@{\extracolsep{\fill}} l c c c c @{}}
\toprule
Model & Params & Loss ($10^{-4}$) & Far put ($10^{-6}$) & ms / epoch \\
\midrule
Neural SDE -- $L=2$ & 79\textsc{K} & $11.90\pm0.11$ & $135.4\pm2.9$ & 94.1 \\
SLiCE -- $L=2$ & 142\textsc{K} & $\mathbf{6.43\pm0.12}$ & $\mathbf{81.4\pm4.6}$ & 170.1 \\
\textbf{SLiSDE gated in-flow -- $L=2$} & 18\textsc{K} & $6.44\pm0.40$ & $83.7\pm4.8$ & \textbf{44.3} \\
\midrule
Neural SDE -- $L=3$ & 157\textsc{K} & $10.29\pm0.13$ & $136.0\pm2.8$ & 109.7 \\
SLiCE -- $L=3$ & 213\textsc{K} & $6.53\pm0.26$ & $82.9\pm3.6$ & 248.4 \\
\textbf{SLiSDE gated in-flow -- $L=3$} & 34\textsc{K} & $\mathbf{6.29\pm0.24}$ & $\mathbf{70.1\pm6.6}$ & \textbf{70.5} \\
\bottomrule
\end{tabular*}
\caption{\textsc{spx} option-surface results without Girsanov tilt, two and three layers. Same masks and protocol as Table~\ref{tab:dax3-main}; best cell of the search grid per family and depth, seven seeds, mean $\pm$ standard error. The narrower three-layer SLiCE cell diverged on every seed and is excluded.}
\label{tab:spx-full}
\end{table}

\subsection*{Noise type and block size}
Each layer can carry multiplicative noise only ($d^j_t=0$), additive noise only ($C^j_t=0$) or both (the noise-type remark of Section~\ref{ssec:base-layer}), and its structured coefficients are block-diagonal with block size $b_s$. Table~\ref{tab:abl-structure} varies each factor separately around the two-layer gated in-flow configuration on \textsc{dax} (general noise, block size $8$) and reports parameter count, held-out loss, far-put loss and time per epoch relative to that reference configuration. The three noise types are within one standard error of each other on the total loss; additive noise is cheaper, while the variants with multiplicative noise fit the far puts better. The total loss is flat in the block size; blocks of $4$ or more improve the far puts over the diagonal model, at a cost in time per epoch that grows with $b_s$.
\begin{table}[ht]
\centering
\footnotesize
\captionsetup{font=footnotesize}
\renewcommand{\arraystretch}{1.1}
\begin{tabular}{lcccc}
\toprule
Variant & Params & Loss & Far put & ms / epoch \\
\midrule
\multicolumn{5}{l}{\emph{Noise type (block size $8$)}} \\
multiplicative & $0.99\times$ & $1.05\times$ & $\mathbf{0.97\times}$ & $0.93\times$ \\
additive & $0.46\times$ & $\mathbf{0.95\times}$ & $1.69\times$ & $\mathbf{0.78\times}$ \\
general (both; reference) & $1.00\times$ & $1.00\times$ & $1.00\times$ & $1.00\times$ \\
\midrule
\multicolumn{5}{l}{\emph{Block size (general noise)}} \\
$b_s=1$ (diagonal) & $0.46\times$ & $0.95\times$ & $1.45\times$ & $\mathbf{0.60\times}$ \\
$b_s=4$ & $0.69\times$ & $\mathbf{0.94\times}$ & $\mathbf{0.95\times}$ & $0.77\times$ \\
$b_s=8$ (reference) & $1.00\times$ & $1.00\times$ & $1.00\times$ & $1.00\times$ \\
$b_s=16$ & $1.61\times$ & $0.97\times$ & $\mathbf{0.95\times}$ & $1.41\times$ \\
\bottomrule
\end{tabular}
\caption{Structural ablations on \textsc{dax}, each varying one factor of the two-layer gated in-flow configuration (general noise, block size $8$). Every entry is the ratio of the seven-seed mean of the variant to that of the reference configuration; the total-loss differences within each block are within one standard error. Top: noise type of the structured layers. Bottom: block size of the block-diagonal structure.}
\label{tab:abl-structure}
\end{table}

\subsection*{Timing study}
\paragraph{Setup.} Tables~\ref{tab:timing-main} and~\ref{tab:timing-full} report wall-clock times after compilation on one NVIDIA RTX 6000 Ada ($42.6$\,GiB usable), as the median of five repeats after two warm-ups, across batch sizes $B$ and horizons $T$. The models are a two-layer SLiSDE with $d=64$, block size $8$ and general noise ($136$K parameters), evaluated \emph{sequentially} so that the comparison isolates the structured layers from the scan, a Neural SDE of the same parameter count ($137$K) and SLiCE at the same width and depth ($142$K).

\paragraph{Training step.} SLiSDE is $6$--$7\times$ faster than the Neural SDE at $B=64$, $3.2\times$ at $B=256$ and $1.3$--$1.4\times$ at $B=1024$. At $B=4096$, where batch parallelism saturates the GPU, the Neural SDE step is $1.4\times$ faster, and SLiCE runs out of memory beyond $B=1024$ at $T=512$.
\begin{table}[ht]
\centering
\small
\captionsetup{font=footnotesize}
\renewcommand{\arraystretch}{1.1}
\begin{tabular}{cccccc}
\toprule
$B$ & $T$ & SLiSDE (seq.) & Neural SDE & SLiCE & speed-up \\
\midrule
64 & 512 & 9.5 & 54.0 & 10.1 & 5.7$\times$ \\
64 & 2048 & 35.8 & 212.0 & 42.8 & 5.9$\times$ \\
64 & 8192 & 117.1 & 838.9 & 156.9 & 7.2$\times$ \\
256 & 2048 & 79.7 & 257.1 & 155.2 & 3.2$\times$ \\
1024 & 512 & 69.0 & 89.4 & 155.0 & 1.3$\times$ \\
1024 & 2048 & 258.9 & 359.5 & OOM & 1.4$\times$ \\
4096 & 512 & 242.1 & 178.6 & OOM & 0.7$\times$ \\
\bottomrule
\end{tabular}
\caption{Wall-clock time per training step (ms; forward plus backward, post-compilation, median of five repeats) as a function of batch size $B$ and horizon $T$ for SLiSDE ($136$K parameters, sequential evaluation), the Neural SDE of the same parameter count ($137$K) and SLiCE ($142$K). Speed-up is Neural SDE over SLiSDE; OOM: exceeds the $42.6$\,GiB of the GPU.}
\label{tab:timing-main}
\end{table}
\paragraph{Forward pass and scan.} Table~\ref{tab:timing-full} separates forward and backward passes and adds the chunked associative scan (best chunk size of $16$, $64$, $256$). The scan shortens the forward pass by $1.4\times$ at $B=64$, $T=512$ and is otherwise on par with or slower than the sequential evaluation: each sequential structured update is already cheap and the chunked scan roughly doubles the arithmetic, so its advantage lies in parallel depth rather than in wall-clock time at these sizes.
\begin{table}[ht]
\centering
\small
\captionsetup{font=footnotesize}
\scriptsize
\setlength{\tabcolsep}{3pt}
\begin{tabular}{ccccccc}
\toprule
$B$ & $T$ & SLiSDE fwd (seq.) & SLiSDE fwd (scan) & NSDE fwd & SLiSDE fwd+bwd & NSDE fwd+bwd \\
\midrule
64 & 512 & 5.7 & 4.1 & 15.3 & 9.5 & 54.0 \\
64 & 2048 & 14.9 & 15.8 & 50.6 & 35.8 & 212.0 \\
64 & 8192 & 48.4 & 50.4 & 202.6 & 117.1 & 838.9 \\
256 & 2048 & 32.5 & 44.6 & 61.0 & 79.7 & 257.1 \\
1024 & 512 & 27.6 & 39.9 & 23.8 & 69.0 & 89.4 \\
1024 & 2048 & 89.9 & OOM & 84.2 & 258.9 & 359.5 \\
4096 & 512 & 80.5 & OOM & 37.5 & 242.1 & 178.6 \\
\bottomrule
\end{tabular}
\caption{Wall-clock time (ms, post-compilation, NVIDIA RTX 6000 Ada) of forward passes and full training steps as a function of batch size and horizon, including the chunked associative scan (best of chunk sizes $16$, $64$, $256$); the Neural SDE has the same parameter count as SLiSDE ($137$K against $136$K).}
\label{tab:timing-full}
\end{table}

\subsection*{Distributional recovery on \textsc{toy}}
Because the \textsc{toy} generator is known, the learned \emph{distribution} can be assessed directly, well beyond the calibrated targets. Table~\ref{tab:abl-dist} reports Kolmogorov--Smirnov and Wasserstein-1 distances to the ground truth, computed from $2\times10^4$ model paths for the best two-layer cell of each family (seven seeds), for the terminal value and for the running maximum. SLiSDE and the Neural SDE recover the terminal marginal comparably well (SLiSDE better in KS, the Neural SDE in $W_1$); on the law of the running maximum, the genuinely path-dependent statistic, SLiSDE is markedly better than both baselines.
\begin{table}[ht]
\centering
\footnotesize
\captionsetup{font=footnotesize}
\setlength{\tabcolsep}{3.5pt}
\begin{tabular}{lcccc}
\toprule
 & KS (terminal) & $W_1$ (terminal) & KS (running max) & $W_1$ (running max) \\
\midrule
Neural SDE ($L=2$) & $0.019\pm0.002$ & $\mathbf{0.020\pm0.002}$ & $0.059\pm0.002$ & $0.047\pm0.002$ \\
SLiCE ($L=2$) & $0.022\pm0.002$ & $0.031\pm0.003$ & $0.081\pm0.003$ & $0.067\pm0.003$ \\
SLiSDE ($L=2$) & $\mathbf{0.015\pm0.001}$ & $0.023\pm0.002$ & $\mathbf{0.036\pm0.002}$ & $\mathbf{0.037\pm0.002}$ \\
\bottomrule
\end{tabular}
\caption{Distances between the learned and the ground-truth laws of the terminal value and of the running maximum on \textsc{toy} ($2\times10^4$ model paths; best two-layer cell of each family; seven seeds, mean $\pm$ s.e.).}
\label{tab:abl-dist}
\end{table}

\subsection*{Girsanov tilt: calibration on \textsc{dax} and rare-event estimation on \textsc{toy}}
\paragraph{Calibration with the tilt on \textsc{dax}: protocol.} Table~\ref{tab:dax3-girsanov} compares five arms over seven seeds under one evaluation protocol, plain Monte Carlo under the reference law with $32\,768$ paths on the held-out strikes: the vanilla model; the untilted reference-law model with $1\,024$ and with $2\,048$ paths per step, the latter being the path budget of the tilted arms; and two tilted arms, which simulate $1\,024$ reference paths and two tilted batches of $512$ paths, one controller per tail. The tilted arms differ in the far-tail estimate: one combines the tilted batch with the reference batch, the other uses the tilted batch alone.

\paragraph{Calibration with the tilt on \textsc{dax}: results.} All ratios in this paragraph are ratios of means over seeds. Relative to the vanilla model, the combined variant lowers the far-call error by about a third ($0.64\times$; lower on six seeds of seven, paired $t$-test on the log-errors $p=0.03$) and the total held-out loss by $15\%$; the tilted-batch-only variant halves the far-put error on average, with a four times smaller spread across seeds, and lowers the total loss by $24\%$. Against the budget-matched untilted arm the tilted arms are never worse and are better on average on the far tails ($0.78\times$ on far calls and $0.60\times$ on far puts for the respective variants), differences that are not significant at seven seeds. The tilt is a genuine change of measure here, $\KL(\QQ\|\Prob)\approx3$ nats and $\ESS/n_q\approx0.05$ on the put side, and costs $1.3\times$ the epoch time of the budget-matched untilted arm.

\begin{table}[ht]
\centering
\scriptsize
\captionsetup{font=footnotesize}
\renewcommand{\arraystretch}{1.1}
\setlength{\tabcolsep}{3pt}
\begin{tabular*}{\textwidth}{@{\extracolsep{\fill}} l c c c c c c @{}}
\toprule
Arm & paths / step & Loss ($10^{-4}$) & Far put ($10^{-6}$) & Far call ($10^{-6}$) & $\ESS/n_q$ put / call & $\KL$ put / call \\
\midrule
Vanilla & $1024$ & $1.15\pm0.12$ & $0.93\pm0.46$ & $3.15\pm0.33$ & -- & -- \\
Untilted, $\rho=0$ & $1024$ & $1.18\pm0.19$ & $0.87\pm0.42$ & $3.62\pm0.78$ & -- & -- \\
Untilted, $\rho=0$ & $2048$ & $0.95\pm0.13$ & $0.81\pm0.30$ & $2.59\pm0.49$ & -- & -- \\
Tilt, combined far-tail estimate & $1024+2{\times}512$ & $0.98\pm0.10$ & $0.82\pm0.27$ & $\mathbf{2.03\pm0.78}$ & $0.05$ / $0.09$ & $2.9$ / $1.5$ \\
Tilt, tilted batch only & $1024+2{\times}512$ & $\mathbf{0.88\pm0.10}$ & $\mathbf{0.48\pm0.08}$ & $3.25\pm1.04$ & $0.05$ / $0.12$ & $3.0$ / $1.4$ \\
\bottomrule
\end{tabular*}
\caption{\textsc{dax} with the Girsanov tilt (seven seeds, mean $\pm$ standard error). The tilt study is trained and evaluated on the strikes that lie within the quoted range at each expiry, so its absolute values are not comparable with Table~\ref{tab:dax3-main}, whose protocol it otherwise follows. Far-tail losses are held-out MSEs on the far strikes; every arm is evaluated by plain Monte Carlo under its reference law with $32\,768$ paths. The tilted arms use $\rho=0$, the cross-entropy controller with a $6$-nat KL wall and the initial-state shift; $\ESS/n_q$ and $\KL(\QQ\|\Prob)$ (nats) are end-of-training values of the put / call controllers; time per epoch $130$ / $130$ / $229$ / $303$ / $295$ ms.}
\label{tab:dax3-girsanov}
\end{table}

\paragraph{The tilt as a learned importance sampler: protocol.} To isolate the benefit of the tilt from the calibration task, we evaluate it on \textsc{toy} in a controlled estimator study against ordinary Monte Carlo at the same compute budget. The reference law is the three-layer \textsc{toy} cell of Table~\ref{tab:toy3-vanilla}, simulated at $N=512$ steps with $\rho=0$ and calibrated seven times with different seeds. The targets are the exceedance probabilities $\Prob(Y_T>H)$ with $H$ at the model's own quantiles of nominal probability $10^{-2}$, $10^{-3}$ and $10^{-4}$; reference values come from $8\times10^5$ plain Monte-Carlo paths. For each backbone one controller per threshold is trained post hoc, in three stages of $1\,000$ cross-entropy steps on batches of $1\,024$ tilted paths, the stage for the rarer event starting from the controller of the previous one, under the wall $\kappa=-\log p+2$ nats. Every estimator receives the same per-estimate path budget $n$; in addition, plain Monte Carlo receives a cost-matched budget inflated by the measured per-path overhead of the tilted simulation (a factor $1.18$). The relative root-mean-square error of every estimator is measured over $100$ independent replications, and the tables report the mean and standard error over the seven backbones.

\paragraph{The tilt as a learned importance sampler: results.} Table~\ref{tab:girsanov-estimator} reports the budget $n=1024$ for plain Monte Carlo, for the tilt with the vanilla importance-sampling estimator and with the self-normalised one, together with the fraction of replications in which plain Monte Carlo returns exactly zero because no path reaches the threshold. The picture is the one predicted by the variance argument of Section~\ref{sec:girsanov}: the gain of the tilt grows with the rarity of the event, from a factor $2$ in relative error at $p=10^{-2}$ to a factor $4$ at $p=10^{-4}$, where plain Monte Carlo returns zero in $90\%$ of the replications. The self-normalised estimator is markedly weaker, because the initial-state shift assigns large weights to tilted paths that miss the event and those weights enter only its normalisation.
\begin{table}[ht]
\centering
\small
\captionsetup{font=footnotesize}
\renewcommand{\arraystretch}{1.1}
\begin{tabular}{lccc}
\toprule
Estimator ($n=1024$), rel.\ RMSE & $p\approx10^{-2}$ & $p\approx10^{-3}$ & $p\approx10^{-4}$ \\
\midrule
Plain MC & $0.320\pm0.005$ & $1.051\pm0.019$ & $3.39\pm0.39$ \\
Plain MC, cost-matched ($\times1.18$) & $0.278\pm0.009$ & $0.934\pm0.047$ & $3.15\pm0.37$ \\
Girsanov tilt, IS & $\mathbf{0.156\pm0.018}$ & $\mathbf{0.394\pm0.104}$ & $\mathbf{0.785\pm0.184}$ \\
Girsanov tilt, SNIS & $0.274\pm0.021$ & $0.789\pm0.106$ & $1.88\pm0.42$ \\
\midrule
Share of plain-MC runs returning exactly $0$ & 0\% & 34\% & 90\% \\
\bottomrule
\end{tabular}
\caption{Rare-event estimation on \textsc{toy}: relative RMSE of each estimator at a budget of $n=1024$ paths per estimate (cost-matched plain Monte Carlo receives $1.18\,n$ paths), over $100$ replications for each of seven independently calibrated models; mean $\pm$ standard error over the models.}
\label{tab:girsanov-estimator}
\end{table}
\paragraph{Budgets and diagnostics.} Table~\ref{tab:girsanov-full} extends Table~\ref{tab:girsanov-estimator} to the budgets $n\in\{256,1024,4096\}$ and adds the truncated vanilla estimator (weights capped at $\sqrt n$). At the end of training the controllers use $2.1$, $3.5$ and $5.1$ nats of divergence, and the tilted paths hit the event at rates of $26\%$, $20\%$ and $15\%$ for $p=10^{-2}$, $10^{-3}$ and $10^{-4}$ (plain Monte Carlo: $1\%$, $0.1\%$, $0.01\%$), at $\ESS/n_q$ of $0.045$, $0.024$ and $0.011$. Two features of the table deserve comment. The self-normalised estimator is uniformly worse than the vanilla one, for the reason given in Appendix~\ref{app:snis-ess}, and the truncated estimator coincides with the untruncated one except at $p=10^{-2}$. The weights are heavy-tailed ($\E^{\QQ}[w]$ estimated on the tilted batch is $0.75$--$0.97$ instead of $1$), so the standard errors across backbones are large and individual cells can lose to plain Monte Carlo; the gain of the tilt nevertheless grows with the rarity of the event at every budget, and at $p=10^{-4}$ it is the only estimator that returns a non-zero value in most replications at $n\le1024$.
\begin{table}[ht]
\centering
\small
\captionsetup{font=footnotesize}
\begin{tabular}{llccc}
\toprule
 & rel.\ RMSE & $n=256$ & $n=1024$ & $n=4096$ \\
\midrule
\multirow{6}{*}{$p\approx10^{-2}$}
 & Plain MC & $0.645\pm0.017$ & $0.320\pm0.005$ & $0.163\pm0.005$ \\
 & Plain MC, cost-matched & $0.560\pm0.012$ & $0.278\pm0.009$ & $0.151\pm0.003$ \\
 & Tilt, IS & $\mathbf{0.411\pm0.084}$ & $\mathbf{0.156\pm0.018}$ & $0.140\pm0.033$ \\
 & Tilt, IS truncated & $0.391\pm0.070$ & $0.156\pm0.018$ & $\mathbf{0.119\pm0.017}$ \\
 & Tilt, SNIS & $0.645\pm0.131$ & $0.274\pm0.021$ & $0.197\pm0.034$ \\
 & \emph{plain MC returning $0$} & $8\%$ & $0\%$ & $0\%$ \\
\midrule
\multirow{6}{*}{$p\approx10^{-3}$}
 & Plain MC & $1.953\pm0.083$ & $1.051\pm0.019$ & $0.501\pm0.014$ \\
 & Plain MC, cost-matched & $1.780\pm0.039$ & $0.934\pm0.047$ & $0.460\pm0.015$ \\
 & Tilt, IS & $\mathbf{0.649\pm0.147}$ & $\mathbf{0.394\pm0.104}$ & $\mathbf{0.330\pm0.075}$ \\
 & Tilt, IS truncated & $0.649\pm0.147$ & $0.394\pm0.104$ & $0.330\pm0.075$ \\
 & Tilt, SNIS & $1.610\pm0.287$ & $0.789\pm0.106$ & $0.482\pm0.066$ \\
 & \emph{plain MC returning $0$} & $75\%$ & $34\%$ & $1\%$ \\
\midrule
\multirow{6}{*}{$p\approx10^{-4}$}
 & Plain MC & $6.64\pm0.95$ & $3.39\pm0.39$ & $1.65\pm0.10$ \\
 & Plain MC, cost-matched & $6.96\pm0.40$ & $3.15\pm0.37$ & $1.65\pm0.13$ \\
 & Tilt, IS & $\mathbf{1.77\pm0.59}$ & $\mathbf{0.785\pm0.184}$ & $\mathbf{0.816\pm0.235}$ \\
 & Tilt, IS truncated & $1.77\pm0.59$ & $0.785\pm0.184$ & $0.816\pm0.235$ \\
 & Tilt, SNIS & $4.06\pm1.08$ & $1.88\pm0.42$ & $1.81\pm0.72$ \\
 & \emph{plain MC returning $0$} & $97\%$ & $90\%$ & $68\%$ \\
\bottomrule
\end{tabular}
\caption{Rare-event estimation on \textsc{toy}: relative RMSE for three budgets and three nominal probabilities, over $100$ replications for each of seven independently calibrated models (mean $\pm$ standard error over the models); cost-matched plain Monte Carlo receives $1.18\,n$ paths; the last row of each block is the share of plain-Monte-Carlo replications with no path above the threshold.}
\label{tab:girsanov-full}
\end{table}

\subsection*{Effect of the controller objective, the KL wall and the far-tail estimate}
The wall $\kappa$ in Eq.~\ref{eq:kl-objective} controls how far the Girsanov proposal may move from the reference measure, and the choice of controller objective decides whether it moves at all. Table~\ref{tab:abl-kl} varies the wall, the objective and the way the far-tail terms are estimated on \textsc{dax}. The variance-type objective, which minimises the self-normalised estimate of the far-tail loss directly, converges to the identity tilt ($\ESS/n_q=0.998$) and its arm is then an untilted model with $1\,024$ extra paths on the far-tail terms; the cross-entropy objective produces a proposal with $2$--$3$ nats of divergence at either wall. Combining the tilted-batch estimate with the reference-batch estimate helps the far calls, whose controller is the weaker of the two, and hurts the far puts, where the tilted batch alone is the better estimator.

\begin{table}[ht]
\centering
\small
\captionsetup{font=footnotesize}
\setlength{\tabcolsep}{4pt}
\begin{tabular}{lcccc}
\toprule
Controller & Loss ($10^{-4}$) & Far put ($10^{-6}$) & Far call ($10^{-6}$) & $\ESS/n_q$ put / call \\
\midrule
CE, $\kappa=6$, combined (7 seeds) & $0.98\pm0.10$ & $0.82\pm0.27$ & $2.03\pm0.78$ & $0.05$ / $0.09$ \\
CE, $\kappa=6$, tilted batch only (7 seeds) & $0.88\pm0.10$ & $0.48\pm0.08$ & $3.25\pm1.04$ & $0.05$ / $0.12$ \\
CE, $\kappa=2$, combined (3 seeds) & $1.13\pm0.20$ & $0.59\pm0.01$ & $2.46\pm1.03$ & $0.06$ / $0.08$ \\
SNIS-variance, combined (3 seeds) & $1.18\pm0.35$ & $1.18\pm0.33$ & $2.71\pm1.10$ & $1.00$ / $0.99$ \\
\bottomrule
\end{tabular}
\caption{Controller ablations on \textsc{dax} far-tail calibration ($\rho=0$, protocol of Table~\ref{tab:dax3-girsanov}; mean $\pm$ standard error over seeds).}
\label{tab:abl-kl}
\end{table}

\end{document}